\documentclass{article}

\usepackage{algorithm}
\usepackage{algorithmic}

\usepackage[preprint]{neurips_2026} %

\PassOptionsToPackage{capitalize,noabbrev}{cleveref}
\PassOptionsToPackage{inline}{enumitem}
\usepackage{ushio}
\DeclareMathOperator{\probability}{\mathbb{P}}
\renewcommand{\P}[2][]{\probability_{#1} \left( {#2} \right)}
\DeclareMathOperator{\expectation}{\mathbb{E}}
\newcommand{\E}[2][]{\expectation_{#1} \left[ {#2} \right]}
\DeclareMathOperator{\variance}{\mathrm{Var}}
\newcommand{\Var}[2][]{\variance_{#1} \left[ {#2} \right]}
\DeclareMathOperator{\covariance}{\mathrm{Cov}}
\newcommand{\Cov}[2][]{\covariance_{#1} \left[ {#2} \right]}
\DeclareMathOperator{\indicator}{\mathds{1}}
\newcommand{\I}[2][]{\indicator_{#1} \left[ {#2} \right]}

\newcommand{\fpr}{\mathrm{FPR}}
\newcommand{\fnr}{\mathrm{FNR}}
\newcommand{\err}{\mathrm{Err}}

\newcommand{\ber}{\mathrm{BER}}
\newcommand{\auc}{\mathrm{AUC}}
\newcommand{\besterr}{\err^*}
\newcommand{\bestber}{\ber^*}
\newcommand{\bestauc}{\auc^*}

\newcommand{\bestberhat}{\widehat{\bestber}}
\newcommand{\bestauchat}{\widehat{\bestauc}}
\newcommand{\phimin}{\phi^{(1)}}
\newcommand{\phimax}{\phi^{(2)}}
\DeclareMathOperator{\clip}{\mathrm{clip}}
\newcommand{\thetahat}[1][\tau_{n}]{\widehat{\theta}_{#1}}
\newcommand{\thetatilde}[1][\tau_{n}]{\tilde{\theta}_{#1}}

\crefname{equation}{}{}
\Crefname{equation}{}{}

\crefname{enumi}{}{}
\Crefname{enumi}{}{}

\crefname{appsec}{Appendix}{Appendices}
\Crefname{appsec}{Appendix}{Appendices}

\mathtoolsset{showonlyrefs}
\usepackage{makecell}

\usepackage{pifont}

\let\oldFootnote\footnote
\newcommand\nextToken\relax
\renewcommand\footnote[1]{%
    \oldFootnote{#1}\futurelet\nextToken\isFootnote}
\newcommand\isFootnote{%
    \ifx\footnote\nextToken\textsuperscript{,}\fi}

\usepackage{import}

\usepackage[utf8]{inputenc} %
\usepackage[T1]{fontenc}    %
\usepackage{hyperref}       %
\usepackage{url}            %
\usepackage{booktabs}       %
\usepackage{amsfonts}       %
\usepackage{nicefrac}       %
\usepackage{microtype}      %
\usepackage{xcolor}         %
\usepackage{wrapfig}        %
\usepackage{caption}        %

\title{Bayes-Optimal BER and AUC: \\Estimation and Evaluation of Estimators}

\author{%
  Ryota Ushio$^{1,2}$ \\
  \And
  Takashi Ishida$^{2,1}$ \\
  \And
  Masashi Sugiyama$^{2,1}$ \\
  \AND
  \vspace{-2.4em}
  ~\\
  $^1$ Graduate School of Frontier Sciences, The University of Tokyo, Tokyo, Japan \\
  $^2$ RIKEN AIP, Tokyo, Japan \\
  \texttt{\{ushio@ms.,ishida@,sugi@\}k.u-tokyo.ac.jp}
}

\begin{document}

\maketitle

\begin{abstract}
  A fundamental quantity in machine learning is
  the \emph{optimal performance}
  achievable by any model on a given task.
  Estimating this quantity 
  allows us to 
  distinguish 
  the irreducible part of the error
  from a deficiency of the model,
  telling us how much room for improvement remains.
  Recent work has shown that the Bayes error, or equivalently
  the optimal accuracy, can be estimated from soft labels in
  binary classification.
  However, accuracy is often a poor summary of performance
  in settings with severe class imbalance or noisy annotations,
  where metrics such as the balanced error rate (BER)
  and the area under the ROC curve (AUC) are more appropriate.
  We address this gap with two complementary contributions.
  \textbf{(i) Estimation.}
  We propose soft-label-based estimators for the optimal BER and AUC.
  We first consider the clean setting in which true soft labels
  and the class prior are known,
  and then extend the estimators to 
  a more realistic setting in which 
  the class prior is unknown 
  and 
  the observed soft labels are \emph{corrupted} 
  by an unknown order-preserving transformation,
  possibly followed by additive noise.
  In the latter setting, 
  we approximately recover the clean soft labels
  via isotonic regression
  with auxiliary hard labels,
  estimate the class prior with a clipped mean of the hard labels,
  and derive finite-sample error bounds for the resulting plug-in estimators.
  \textbf{(ii) Evaluation.}
  Since the optimum is unobservable on real datasets,
  evaluating any such estimator is itself nontrivial.
  We extend the FeeBee framework,
  originally proposed for evaluating Bayes-error estimators,
  to the optimal BER and AUC.
  The resulting procedure provides practical evaluation scores
  without requiring knowledge of the optimum,
  and applies to any estimator of the optimal BER or AUC,
  not only our proposed ones.
  Experiments on synthetic and real-world datasets validate
  both the estimators and the evaluation procedure.
\end{abstract}

\section{Introduction}
\label{sec:introduction}

An integral part of machine learning workflows is
evaluating models using performance metrics.
Of natural interest, then, is
the \emph{Bayes-optimal performance}, i.e., the best value achievable
by any model on a given task.
This quantity represents the irreducible component of the error
and helps us
separate a deficiency of the current model
from the inherent difficulty of the task.
In practice,
when the gap between current and optimal performance is small,
it tells us that not much room is left for improvement
and that
further training of the current model or
training of larger models
is unlikely to help.
This is particularly important
in the era of large-scale machine learning models,
which often require a huge amount of computational resources to train
and have a significant environmental impact
\citep{strubell2020energy,luccioni2023estimating}.
Moreover, recent work suggests that estimation 
of the Bayes-optimal performance can be useful for
detecting data contamination \citep{ishida2026capbencher}.

We focus on binary classification in this paper.
Given a binary classifier
$h: \mathcal{X} \to \{0, 1\}$,
the most commonly used metric
to measure its performance
is the \emph{error rate}
$\err(h) \coloneqq \P{h(x) \neq y}$
or the \emph{accuracy} $1 - \err(h)$,
where $(x, y)$ is drawn from an unknown joint distribution
$\mathbb{P}$
over 
the space $\mathcal{X}$
of input features
and the set $\set{0, 1}$ of class labels.
When the class posterior,
described by
$\eta(x) \coloneqq \P{y = 1 \mid x}$,
is genuinely uncertain,
even a perfect predictor cannot
avoid misclassification
due to the stochasticity of the labels.
The lowest achievable error rate
$\besterr \coloneqq \inf_{h: \mathcal{X} \to \{0, 1\}} \err(h)$,
also known as the \emph{Bayes error},
therefore quantifies the inherent label uncertainty
and is determined by 
the underlying data distribution.

Estimation of the Bayes error 
has a long history in the literature.
While most existing methods use input-label pairs
$\set{(x_i, y_i)}_{i=1}^n$
\citep{fukunaga1975k,devijver1985multiclass,berisha2014empirically,moon2018ensemble,noshad2019learning,theisen2021evaluating},
a recent line of work suggests another approach
that estimates the Bayes error from \emph{soft labels}
$\eta_i = \eta(x_i)$, $i = 1, \dots, n$
\citep{ishida2023is,jeong2023demystifying,ushio2025practical}.
A soft label $\eta_{i}$ represents class membership
as a probability value in $[0, 1]$
unlike a binary \emph{hard label} $y_i \in \{0, 1\}$.
This soft-label-based approach
does not need access to the input features $x_i$,
which makes it 
suitable for privacy-sensitive applications
and also allows it to bypass the curse of dimensionality.
Also, it does not make the distributional assumptions
often imposed by the methods based on input-label pairs,
such as 
lower/upper bounds on the probability density function
or Lipschitzness of the density.
Early work typically assumes that the true soft labels $\eta_i$
or their high-quality estimates are available,
e.g., by averaging many independent annotations per instance
\citep{ishida2023is,jeong2024data},
which might be unrealistic in practice.
More recent work \citep{ushio2025practical}
relaxed this assumption
by allowing the soft labels to be corrupted
by unknown order-preserving transformations.

However, the error rate is
a poor summary of performance
under class imbalance or label noise, 
which are both common in practice,
and other metrics such as
the \emph{balanced error rate (BER)}
and the \emph{area under the ROC curve (AUC)}
are often preferred in such settings
\citep{ling1998data,gu2009evaluation,menon2013statistical,menon2015learning,van2015average,menon2018learning,charoenphakdee2019symmetric}.
Concretely,
the BER of a classifier
$h: \mathcal{X} \to \set{0, 1}$
is the average of the in-class error rates,
i.e.,
$
  \ber(h) \coloneqq (\fpr(h) + \fnr(h)) / 2
$,
where
$\fpr(h) \coloneqq \P{h(x) = 1 \mid y = 0}$
and
$\fnr(h) \coloneqq \P{h(x) = 0 \mid y = 1}$
are the false positive and false negative rates
\citep{menon2015learning}.
The AUC of a score-based classifier
$f: \mathcal{X} \to \R$
that predicts $h(x) = \I{f(x) \geq 0}$
is the frequency that 
a random positive instance
$x_+ \sim \P{x \mid y = 1}$
is ranked
above a random negative instance
$x_- \sim \P{x \mid y = 0}$,
with ties counted at $1/2$:
$
  \auc(f)
  \coloneqq 
  \E{
    \I{f(x_+) > f(x_-)} 
    + 
    \frac{1}{2} \I{f(x_+) = f(x_-)}
  }
$,
where $\I{\cdot}$ denotes the indicator function
\citep{clemenccon2008ranking,menon2015learning}.
Under severe class imbalance,
a trivial classifier that always predicts the majority class
can achieve a near-zero error rate
while learning nothing from data,
and hence the error rate is not helpful.
BER remains informative even under class imbalance
by averaging the per-class error rates \citep{menon2013statistical}.
AUC is also a preferred metric under class imbalance, as
it summarizes the trade-off between true- and false-positive rates over all decision thresholds
\citep{ling1998data,fawcett2006introduction,gu2009evaluation}.
Moreover, BER and AUC are popular metrics
under label noise
due to their robustness;
e.g., their optimizers remain unchanged under noisy labels
\citep{menon2015learning,van2015average}.

\paragraph{Our contribution}

Despite the practical importance of BER and AUC,
the estimation of their Bayes-optimal values,
i.e.,
$\bestber \coloneqq \inf_{h: \mathcal{X} \to \set{0, 1}} \ber(h)$
and
$\bestauc \coloneqq \sup_{f: \mathcal{X} \to \R} \auc(f)$,
remains underexplored.
The first part of our contribution 
fills this gap by proposing soft-label-based estimators
of $\bestber$ and $\bestauc$.
We start with the clean setting
similar to the work by \citet{ishida2023is} for the Bayes error,
in which the true soft labels $\eta_i$
are available (\cref{sec:clean}).
Then, inspired by \citet{ushio2025practical},
we extend the estimators to a more realistic setting
where the soft labels are only observed through
some unknown order-preserving corruption,
possibly followed by additive noise (\cref{sec:corrupted}).
Unlike the Bayes error, the optimal BER and AUC depend on
the class prior
$\theta \coloneqq \P{y = 1}$ ($0 < \theta < 1$).
We first assume that $\theta$ is known for simplicity and then handle the unknown-prior case.

A separate question is how to evaluate
an estimator of $\bestber$ or $\bestauc$
on \emph{real datasets}.
This is challenging because,
for real-world datasets,
the underlying data distribution is unknown and
does not reveal the ground-truth optimal values.
For the Bayes error,
\citet{renggli2021evaluating} proposed \emph{FeeBee},
which sidesteps this issue
with a simple idea of injecting label noise.
In the second part of our contribution,
we extend this idea to the optimal BER and AUC
(\cref{sec:evaluation}).
Specifically, we first derive 
a simple closed-form relationship
between the original optima
and their noise-injected counterparts.
Then, we propose scores for evaluating estimators of $\bestber$ and $\bestauc$
based on the derived relationship.
The procedure can be applied to any
estimator of $\bestber$ or $\bestauc$,
not only the ones we propose.
Our proposed method allows us to 
evaluate Bayes-optimal BER/AUC estimators
on real-world datasets
without knowledge of the true optima.

\section{Estimation from clean soft labels}
\label{sec:clean}

In this section,
we first show that 
the optimal BER and AUC
can be expressed as expectations of 
certain functions of the class posterior $\eta(x)$
(\cref{lem:bestber,lem:bestauc}).
This reveals that, for each of BER and AUC,
there are two different natural choices of
unbiased estimators of the optimum.
We discuss data-dependent criteria (or \emph{discriminants})
 for choosing between them
in \cref{sec:min-vs-max}.
Next, in \cref{sec:alg-bestauc},
we present an efficient $O(n \log n)$ time algorithm 
for computing the proposed optimal AUC estimator and
AUC discriminant, which would otherwise take $\Theta(n^{2})$ time.
Finally in \cref{sec:unknown-class-prior},
we discuss the case where the class prior $\theta$ is unknown.

\begin{lemma}[store=lem:bestber]
  \label{lem:bestber}
  The optimal BER can be expressed as
  $\bestber = \E{ \phimin_{\ber}(\eta(x)) } = \E{ \phimax_{\ber}(\eta(x)) }$, where
  $\phimin_{\ber}(z) = \frac{1}{2} \min \set{\frac{z}{\theta}, \frac{1 - z}{1 - \theta}}$
  and
  $\phimax_{\ber}(z) = 1 - \frac{1}{2} \max \set{\frac{z}{\theta}, \frac{1 - z}{1 - \theta}}$.
\end{lemma}

\begin{lemma}[store=lem:bestauc]
  \label{lem:bestauc}
  Let $x'$ be an i.i.d.\ copy of $x$.
  Then, 
  the optimal
  AUC is given by
  $\bestauc = \E{\phimin_{\auc}(\eta(x), \eta(x'))} = \E{\phimax_{\auc}(\eta(x), \eta(x'))}$,
  where
  $
    \phimin_{\auc}(z_1, z_2) = 
    1 - \frac{1}{2 \theta (1 - \theta)} \min \set{
      z_1 (1 - z_2), z_2 (1 - z_1)
    }
  $ and $
    \phimax_{\auc}(z_1, z_2) = 
    \frac{1}{2 \theta (1 - \theta)} 
    \max \set{z_1 (1 - z_2), z_2 (1 - z_1)}
  $.
\end{lemma}

\begin{remark}
  \citet[Example 1]{clemenccon2008ranking}
  mentioned 
  the former expression of
  $\bestauc = \expectation [\phimin_\auc(\eta(x), \eta(x'))]$
  but without a proof.
  We state the full proof for the sake of completeness
  (\cref{sec:proof-clean:bestber-bestauc}). 
\end{remark}

\subsection{Discriminants for BER and AUC}
\label{sec:proof-clean:discriminants}

First, we recall the definitions of the discriminants for BER and AUC.
Let $\zeta \coloneqq \eta(x) - \theta$ 
be the centered version of the class posterior.
The discriminant for BER is defined as
$\Delta_{\ber} \coloneqq (1 - 2 \theta) \E{ \zeta \abs{\zeta}}$.
For AUC, we define
$\Delta_{\auc} \coloneqq (1 - 2 \theta) \E{ \zeta \abs{\zeta - \zeta'}}$,
where
$\zeta' \coloneqq \eta(x') - \theta$
is an i.i.d.\ copy of $\zeta$.

Now we prove \cref{thm:min-vs-max:necessary-and-sufficient},
which we restate below:

\getkeytheorem{thm:min-vs-max:necessary-and-sufficient}

We first prove the BER part:

\begin{proposition}
  \label{prop:ber-min-vs-max:necessary-and-sufficient}
  For all $n \in \N$,
  we have
  $\Var{\widehat{\bestber_1}} \leq \Var{\widehat{\bestber_2}}$
  if and only if $\Delta_\ber \geq 0$.
\end{proposition}

\cref{prop:ber-min-vs-max:necessary-and-sufficient}
follows immediately from the next lemma.

\begin{lemma}
  \label{lem:ber-min-vs-max:variance-difference}
  The variance gap between
  the two minimum BER estimators
  $\widehat{\bestber_1}$ and $\widehat{\bestber_2}$
  is given by
  \begin{equation}
    \Var{ \widehat{\bestber_2} }
    -
    \Var{ \widehat{\bestber_1} }
    =
    \frac{1 - 2 \theta}{4 \theta^2 (1 - \theta)^2 n}
    \E{\zeta |\zeta|}
    \text{.}
  \end{equation}
\end{lemma}

\begin{proof}
  Throughout the proof, we abuse the notation
  and write $\eta = \eta(x)$ for brevity.
  Let 
  \begin{equation}
    M \coloneqq \max \set{\eta (1 - \theta), \theta (1 - \eta)}
    \text{,}
    \quad
    m \coloneqq \min \set{\eta (1 - \theta), \theta (1 - \eta)}
    \text{.}
  \end{equation}
  Noting that
  \begin{equation}
    \begin{aligned}
      M
      &=
      \frac{1}{2}
      \Big[
        \set{\eta (1 - \theta) + \theta (1 - \eta)}
        +
        |\eta - \theta|
      \Big]
      \text{,}
      \\ 
      m
      &=
      \frac{1}{2}
      \Big[
        \set{\eta (1 - \theta) + \theta (1 - \eta)}
        -
        |\eta - \theta|
      \Big]
      \text{,}
    \end{aligned}
  \end{equation}
  it holds that
  \begin{equation}
    \begin{aligned}
      \Var{M} - \Var{m}
      &=
      \Cov{
        \eta (1 - \theta) + \theta (1 - \eta),
        |\eta - \theta|
      }
      \\ &=
      \Cov{
        (1 - 2 \theta) \eta + \theta,
        |\eta - \theta|
      }
      \\ &=
      (1 - 2 \theta) 
      \Cov{
        \eta - \theta,
        |\eta - \theta|
      }
      \\ &=
      (1 - 2 \theta) 
      \Cov{
        \zeta,
        |\zeta|
      }
      \text{.}
    \end{aligned}
  \end{equation}
  Since $\zeta$ is a zero-mean random variable, we have
  $\Cov{\zeta, |\zeta|} = \E{\zeta |\zeta|}$.
  Also observe that 
  $\phimin_{\ber}(\eta) = \frac{1}{2 \theta (1 - \theta)} m$ 
  and
  $\phimax_{\ber}(\eta) = 1 - \frac{1}{2 \theta (1 - \theta)} M$.
  Hence, we obtain
  \begin{equation}
    \begin{aligned}
      \Var{ \widehat{\bestber_2} }
      -
      \Var{ \widehat{\bestber_1} }
      &=
      \frac{1}{n} \paren{
        \Var{\phimax_{\ber}(\eta)}
        -
        \Var{\phimin_{\ber}(\eta)}
      }
      \\ &=
      \frac{1}{4 \theta^2 (1 - \theta)^2 n}
      \paren{
        \Var{M} - \Var{m}
      }
      \\ &=
      \frac{1 - 2 \theta}{4 \theta^2 (1 - \theta)^2 n}
      \E{\zeta |\zeta|}
      \text{.}
    \end{aligned}
  \end{equation}
\end{proof}

Next, we prove the AUC part of \cref{thm:min-vs-max:necessary-and-sufficient}
(restated as \cref{prop:auc-min-vs-max:necessary-and-sufficient}).
Recall that
$\zeta \coloneqq \eta(x) - \theta$,
$\zeta' \coloneqq \eta(x') - \theta$.
$\zeta$ and $\zeta'$ are i.i.d.\ random variables with mean $0$.

\begin{lemma}
  \label{lem:auc-min-vs-max:variance-difference}
  The difference between the variances of
  the two maximum AUC estimators
  $\widehat{\bestauc_1}$ and $\widehat{\bestauc_2}$
  is given by
  \begin{equation}
  \Var{\widehat{\bestauc_2}} - \Var{\widehat{\bestauc_1}}
  =
  \frac{1 - 2 \theta}{\theta^2 (1 - \theta)^2 n} 
  \E{\zeta |\zeta - \zeta'|}
  -
  \frac{1}{\theta^2 (1 - \theta)^2 n (n - 1)} 
  \E{\zeta \zeta' |\zeta - \zeta'|}
  \text{.}
  \end{equation}
\end{lemma}

\begin{proof}
  Throughout the proof, we abuse the notation
  and write $\eta = \eta(x)$ and $\eta' = \eta(x')$ for brevity.
  Let 
  \begin{equation}
    M(z, w) \coloneqq \max \set{z (1 - w), w (1 - z)}
    \text{,}
    \quad
    m(z, w) \coloneqq \min \set{z (1 - w), w (1 - z)}
    \text{,}
  \end{equation}
  and define
  \begin{equation}
    M_1 \coloneqq \E{M(\eta, \eta') \mid \eta}
    \text{,}
    \quad
    m_1 \coloneqq \E{m(\eta, \eta') \mid \eta}
    \text{.}
  \end{equation}
  By the general theory of U-statistics \citep{hoeffding1948class}, we have
  \begin{equation}
    \begin{gathered}
      \Var{
        \frac{2}{n (n - 1)} \sum_{i < j} M(\eta_{i}, \eta_{j})
      } 
      =
      \frac{2}{n (n - 1)} 
      \paren{
        2 (n - 2) \Var{M_1} 
        + 
        \Var{M(\eta, \eta')}
      }
      \text{,}
      \\
      \Var{
        \frac{2}{n (n - 1)} \sum_{i < j} m(\eta_{i}, \eta_{j})
      } 
      =
      \frac{2}{n (n - 1)} 
      \paren{
        2 (n - 2) \Var{m_1} 
        + 
        \Var{m(\eta, \eta')}
      }
      \text{.}
    \end{gathered}
  \end{equation}

  Noting that
  \begin{equation}
    \begin{aligned}
      M(\eta, \eta')
      &=
      \frac{1}{2}
      \Big[
        \set{\eta (1 - \eta') + \eta' (1 - \eta)}
        +
        |\eta - \eta'|
      \Big]
      \text{,}
      \\ 
      m(\eta, \eta')
      &=
      \frac{1}{2}
      \Big[
        \set{\eta (1 - \eta') + \eta' (1 - \eta)}
        -
        |\eta - \eta'|
      \Big]
      \text{,}
    \end{aligned}
  \end{equation}
  we have 
  \begin{equation}
    \begin{aligned}
      M_1
      &=
      \frac{1}{2}
      \Big[
        \set{\eta (1 - \theta) + \theta (1 - \eta)}
        +
        \E{|\eta - \eta'| \mid \eta}
      \Big]
      \text{,}
      \\ 
      m_1
      &=
      \frac{1}{2}
      \Big[
        \set{\eta (1 - \theta) + \theta (1 - \eta)}
        -
        \E{|\eta - \eta'| \mid \eta}
      \Big]
      \text{.}
    \end{aligned}
  \end{equation}
  Therefore, it holds that
  \begin{equation}
    \begin{aligned}
      \Var{M_1} - \Var{m_1}
      &=
      \Cov{
        \eta (1 - \theta) + \theta (1 - \eta),
        \E{|\eta - \eta'| \mid \eta}
      }
      \\ &=
      \Cov{
        (1 - 2 \theta) \eta + \theta,
        \E{|\eta - \eta'| \mid \eta}
      }
      \\ &=
      (1 - 2 \theta) 
      \Cov{
        \eta,
        \E{|\eta - \eta'| \mid \eta}
      }
      \\ &=
      (1 - 2 \theta) 
      \Cov{
        \eta,
        |\eta - \eta'|
      }
      \\ &=
      (1 - 2 \theta) 
      \Cov{
        \zeta,
        |\zeta - \zeta'|
      }
    \end{aligned}
  \end{equation}
  and
  \begin{equation}
    \begin{aligned}
      \Var{M(\eta, \eta')} - \Var{m(\eta, \eta')}
      &=
      \Cov{
        \eta (1 - \eta') + \eta' (1 - \eta),
        |\eta - \eta'|
      }
      \\ &=
      \Cov{
        (1 - 2\theta) (\zeta + \zeta')
        - 2 \zeta \zeta'
        + 2 \theta (1 - \theta),
        |\zeta - \zeta'|
      }
      \\ &=
      (1 - 2\theta) 
      \Cov{
        (\zeta + \zeta'),
        |\zeta - \zeta'|
      }
      - 2
      \Cov{
        \zeta \zeta',
        |\zeta - \zeta'|
      }
      \\ &=
      2 (1 - 2\theta) 
      \Cov{
        \zeta,
        |\zeta - \zeta'|
      }
      - 2
      \Cov{
        \zeta \zeta',
        |\zeta - \zeta'|
      }
      \text{.}
    \end{aligned}  
  \end{equation}

  Combining these results, we obtain
  \begin{equation}
    \begin{aligned}
      &
      \Var{\widehat{\bestauc_2}} - \Var{\widehat{\bestauc_1}}
      \\ ={} &
      \frac{1}{4 \theta^2 (1 - \theta)^2}
      \paren{
        \Var{
          \frac{2}{n (n - 1)} \sum_{i < j} M(\eta_{i}, \eta_{j})
        } 
        - 
        \Var{
          \frac{2}{n (n - 1)} \sum_{i < j} m(\eta_{i}, \eta_{j})
        }
      }
      \\ ={} &
      \frac{1}{4 \theta^2 (1 - \theta)^2}
      \frac{2}{n (n - 1)} 
      \Big [
        2 (n - 2) 
        \cdot
        (1 - 2 \theta) 
        \Cov{
          \zeta,
          |\zeta - \zeta'|
        }
        + 
        2 (1 - 2\theta) 
        \Cov{
          \zeta,
          |\zeta - \zeta'|
        }
        - 2
        \Cov{
          \zeta \zeta',
          |\zeta - \zeta'|
        }
      \Big ]
      \\ ={} &
      \frac{1}{\theta^2 (1 - \theta)^2}
      \frac{1}{n (n - 1)} 
      \Big [
        (n - 1) 
        \cdot
        (1 - 2 \theta) 
        \Cov{
          \zeta,
          |\zeta - \zeta'|
        }
        -
        \Cov{
          \zeta \zeta',
          |\zeta - \zeta'|
        }
      \Big ]
      \text{.}
    \end{aligned}
  \end{equation}
  Finally, 
  since both $\zeta$ and $\zeta \zeta'$ have zero means,
  we have
  $\Cov{\zeta, |\zeta - \zeta'|} = \E{\zeta |\zeta - \zeta'|}$
  and
  $\Cov{\zeta \zeta', |\zeta - \zeta'|} = \E{\zeta \zeta' |\zeta - \zeta'|}$.
  Now the lemma has been proved.
\end{proof}

\begin{remark}[Asymptotic variance difference]
  Let
  $\sigma_{i, n}^2 \coloneqq \Var{\sqrt{n} \paren{\widehat{\bestauc_i} - \bestauc}}$.
  \cref{lem:auc-min-vs-max:variance-difference} implies that\footnote{
    Also, it can be shown that 
    $\sqrt{n} \paren{\widehat{\bestauc_i} - \bestauc}
    \xrightarrow{d} \mathcal{N}(0, \lim_{n \to \infty} \sigma_{i, n}^2)$
    for each $i = 1, 2$
    \citep{hoeffding1948class}.
  }
  \begin{equation}
    \sigma_{2, n}^2 - \sigma_{1, n}^2
    \xrightarrow{n \to \infty}
    \frac{1 - 2 \theta}{\theta^2 (1 - \theta)^2} 
    \E{\zeta |\zeta - \zeta'|}
    \text{.}
  \end{equation}
\end{remark}

\begin{lemma}
  \label{lem:auc-min-vs-max:covariance2} 
  We have
  \begin{equation}
  \E{\zeta \zeta' |\zeta - \zeta'|} < 0
  \end{equation}
  unless $\zeta$ is a.s.\ constant.
\end{lemma}

\begin{proof}
  If $\zeta \zeta' > 0$, we have
  $|\zeta - \zeta'| < |\zeta| + |\zeta'|$
  and hence
  $\zeta \zeta' |\zeta - \zeta'| < \zeta \zeta' (|\zeta| + |\zeta'|)$.
  On the other hand, if $\zeta \zeta' \leq 0$,
  it holds that
  $|\zeta - \zeta'| = |\zeta| + |\zeta'|$
  and thus
  $\zeta \zeta' |\zeta - \zeta'| = \zeta \zeta' (|\zeta| + |\zeta'|)$.
  Therefore, we always have
  $\zeta \zeta' |\zeta - \zeta'| \leq \zeta \zeta' (|\zeta| + |\zeta'|)$
  and the equality holds only when $\zeta \zeta' \leq 0$.

  Taking the expectation on both sides, we obtain
  \begin{equation}
    \begin{aligned}
    \E{\zeta \zeta' |\zeta - \zeta'|} 
    &\leq 
    \E{\zeta \zeta' (|\zeta| + |\zeta'|)}
    \\ &=
    \E{\zeta'} \E{\zeta |\zeta|} + \E{\zeta} \E{\zeta' |\zeta'|}
    \\ &=
    0
    \text{.}
    \end{aligned}
  \end{equation}
  Here, we used the independence of $\zeta$ and $\zeta'$ as well as their zero means.

  To see that the strict inequality holds, note that
  \begin{equation}
    \begin{aligned}
      \P{
        \zeta \zeta' |\zeta - \zeta'|
        \neq
        \zeta \zeta' (|\zeta| + |\zeta'|)
      }
      &=
      \P{\zeta \zeta' > 0}
      \\ &=
      \P{
        (\zeta > 0 \text{ and } \zeta' > 0)
        \text{ or }
        (\zeta < 0 \text{ and } \zeta' < 0)
      }
      \\ &=
      \P{\zeta > 0}^2
      +
      \P{\zeta < 0}^2
      \text{.}
    \end{aligned}
  \end{equation}
  We must have 
  $\P{\zeta > 0}^2 + \P{\zeta < 0}^2 > 0$,
  since otherwise $\zeta$ would be a.s.\ constant.
  Therefore, we have a non-zero probability that
  $\zeta \zeta' |\zeta - \zeta'|$ disagrees with
  $\zeta \zeta' \paren{|\zeta| + |\zeta'|}$,
  which implies the strict inequality.
\end{proof}

We can now prove the desired result,
i.e., the AUC part of 
\cref{thm:min-vs-max:necessary-and-sufficient},
by combining
\cref{lem:auc-min-vs-max:variance-difference}
and
\cref{lem:auc-min-vs-max:covariance2}.

\begin{proposition}
  \label{prop:auc-min-vs-max:necessary-and-sufficient}
  Assume that the random variable $\eta(x)$ is not a.s.\ constant,
  and define 
  $\Delta_\auc \coloneqq (1 - 2 \theta) \E{\zeta |\zeta - \zeta'|}$.
  Then,
  $\Var{\widehat{\bestauc_1}} < \Var{\widehat{\bestauc_2}}$
  holds for all $n \geq 2$
  if and only if $\Delta_\auc \geq 0$.
\end{proposition}

\begin{proof}
  If $\Delta_\auc \geq 0$, then
  \cref{lem:auc-min-vs-max:variance-difference,lem:auc-min-vs-max:covariance2}
  gives $\Var{\widehat{\bestauc_2}} - \Var{\widehat{\bestauc_1}} > 0$.
  Conversely, if $\Delta_\auc < 0$, we have
  \begin{equation}
  \lim_{n \to \infty} 
  n \paren{
    \Var{\widehat{\bestauc_2}} - \Var{\widehat{\bestauc_1}}
  } 
  = 
  \frac{\Delta_\auc}{\theta^{2} (1 - \theta)^{2}} 
  < 0
  \text{,}
  \end{equation}
  and therefore we will end up with
  $\Var{\widehat{\bestauc_1}} > \Var{\widehat{\bestauc_2}}$
  for sufficiently large $n$.
\end{proof}

\subsubsection{
  How to choose the better estimator based on data
}
\label{sec:proof-clean:discriminants:test}

Define 
$h(x, y) 
\coloneqq
\frac{1 - 2 \theta}{2}
(x + y) |x - y|$,
so that we have
$\Delta_\auc = \E{h(\eta - \theta, \eta' - \theta)}$.
Given a sample $\set{\eta_{i}}_{i=1}^n$,
let
\begin{equation}
  \label{eq:auc-min-vs-max:hbar}
  \bar{h}_{i}
  \coloneqq
  \frac{1}{n - 1}
  \sum_{j\text{: } j \neq i} 
  h(\eta_{i} - \theta, \eta_{j} - \theta)
\end{equation}
for each $i \in \set{1, \dots, n}$.
Define $\widehat{\Delta}_{\auc}$ and $\widehat{v}_\auc$ as
\begin{equation}
  \label{eq:auc-min-vs-max:deltahat}
  \widehat{\Delta}_{\auc}
  \coloneqq
  \frac{1}{n}
  \sum_{i=1}^{n} \bar{h}_{i}
  \text{.}
\end{equation}
Then, $\widehat{\Delta}_{\auc}$ is clearly 
an unbiased estimator of $\Delta_\auc$.
Therefore, we can use the sign of $\widehat{\Delta}_{\auc}$
as a data-dependent criterion to choose between
$\widehat{\bestauc_1}$ and $\widehat{\bestauc_2}$:
we choose $\widehat{\bestauc_1}$
if $\widehat{\Delta}_{\auc} \geq  0$,
and choose $\widehat{\bestauc_2}$
if $\widehat{\Delta}_{\auc} < 0$.

Moreover, we can test the sign of $\Delta_\auc$
based on the asymptotic normality of $\widehat{\Delta}_\auc$.
by the general theory of U-statistics of order two
\citep{callaert1981order}, we have
\begin{equation}
  \label{eq:auc-min-vs-max:asymptotic-normality}
  \frac{
    \widehat{\Delta}_{\auc} - \Delta_\auc
  }{
    \sqrt{\widehat{v}_\auc / n}
  }
  \xrightarrow{d}
  \mathcal{N}(0, 1)
  \text{,}
\end{equation}
where
\begin{equation}
  \label{eq:auc-min-vs-max:vhat}
  \widehat{v}_\auc
  \coloneqq
  \frac{4 (n - 1)}{(n - 2)^{2}}
  \sum_{i=1}^{n}
  \paren{
    \bar{h}_{i} - \widehat{\Delta}_\auc
  }^{2}
  \text{.}
\end{equation}
Based on this asymptotic normality,
we can construct a Wald-type hypothesis test for
the sign of $\Delta_\auc$:
\begin{equation}
  H_0 \text{: } \Delta_\auc < 0 
  \quad \text{vs.} \quad
  H_1 \text{: } \Delta_\auc \geq 0
  \text{.}
\end{equation}
We define the test statistic as
\begin{equation}
  T_\auc \coloneqq
  \frac{
    \widehat{\Delta}_{\auc}
  }{
    \sqrt{\widehat{v}_\auc / n}
  }
\end{equation}
and reject $H_0$ if $T_\auc \geq z_{\alpha}$,
where $z_{\alpha}$ is the upper $\alpha$-quantile of
the standard normal distribution.
The p-value of this test is given by $1 - \Phi(T_\auc)$,
where $\Phi$ is the cumulative distribution function
of the standard normal distribution.
We can also consider the converse test
\begin{equation}
  H_0 \text{: } \Delta_\auc \geq 0
  \quad \text{vs.} \quad
  H_1 \text{: } \Delta_\auc < 0
  \text{,}
\end{equation}
by rejecting $H_0$ if $T_\auc \leq z_{1 - \alpha}$.
The p-value of this test is given by $\Phi(T_\auc)$.

\subsubsection{
  Efficient computation of
  $\widehat{\Delta}_\auc$ 
  and 
  $\widehat{v}_\auc$
}

Naively, computing $\widehat{\Delta}_\auc$ and $\widehat{v}_\auc$
would take $O(n^{2})$ time due to the double summation.
However, we can construct an $O(n \log n)$-time algorithm,
which is based on the following theorem.

\begin{theorem}
  \label{thm:auc-min-vs-max:hbar-constant-time}
  For each $j \in \set{1, \dots, n}$, 
  let $(j)$ denote the index of the $j$-th smallest element in
  $\set{\eta_{1}, \dots, \eta_{n}}$,
  i.e., $\eta_{(1)} \leq \dots \leq \eta_{(n)}$,
  and let $S_{i}$
  be the $i$-th prefix sum of 
  $\set{(\eta_{(j)} - \theta)^{2}}_{j=1}^n$,
  i.e.,
  $S_{i}
  \coloneqq 
  \sum_{j=1}^{i} (\eta_{(j)} - \theta)^{2}$.
  Then, $\bar{h}_{(i)}$ 
  can be expressed as
  \begin{equation}
    \label{eq:auc-min-vs-max:hbar-constant-time}
    \bar{h}_{(i)}
    =
    \frac{1 - 2 \theta}{2 (n - 1)}
    \left[
      (2i - n) (\eta_{(i)} - \theta)^{2} + S_{n} - 2 S_{i}
    \right]
    \text{.}
  \end{equation}
\end{theorem}

\begin{proof}
  For ease of notation, let 
  $a_{i} \coloneqq \frac{2 (n - 1)}{1 - 2 \theta} \bar{h}_{(i)}$
  and 
  $\zeta_{i} \coloneqq \eta_{(i)} - \theta$.
  Then
  \begin{equation}
    \begin{aligned}
      a_{i}
      &=
      \sum_{j: j \neq i} 
      (\zeta_{i} + \zeta_{j}) |\zeta_{i} - \zeta_{j}|
      \\ &=
      \sum_{j < i} 
      (\zeta_{i} + \zeta_{j}) (\zeta_{i} - \zeta_{j})
      +
      \sum_{j > i}
      (\zeta_{i} + \zeta_{j}) (\zeta_{j} - \zeta_{i})
      \\ &=
      \sum_{j < i}
      \left(
        \zeta_{i}^{2} - \zeta_{j}^{2}
      \right)
      +
      \sum_{j > i}
      \left(
        \zeta_{j}^{2} - \zeta_{i}^{2}
      \right)
      \\ &=
      (i - 1) \zeta_{i}^{2} 
      - 
      \sum_{j < i} \zeta_{j}^{2}
      +
      \sum_{j > i} \zeta_{j}^{2} 
      - 
      (n - i) \zeta_{i}^{2}
      \text{.}
    \end{aligned}
  \end{equation}
  Observe that
  \begin{equation}
    \sum_{j < i} \zeta_{j}^{2}
    =
    S_{i - 1}
    = 
    S_{i} - \zeta_{i}^{2}
    \text{, \quad}
    \sum_{j > i} \zeta_{j}^{2}
    =
    S_{n} - S_{i}
    \text{.}
  \end{equation}
  Hence, we have
  \begin{equation}
    a_{i}
    =
    (2i - n) \zeta_{i}^{2} 
    +
    S_{n} - 2 S_{i}
    \text{,}
  \end{equation}
  which completes the proof.
\end{proof}

\cref{thm:auc-min-vs-max:hbar-constant-time}
gives an algorithm that computes
$\widehat{\Delta}_\auc$ and $\widehat{v}_\auc$
in $O(n \log n)$ time,
as shown in 
\cref{alg:auc-delta-vhat}.
The steps are as follows:
\begin{enumerate}
  \item
    Sort $\eta_{1}, \dots, \eta_{n}$ in ascending order
    to obtain $\eta_{(1)}, \dots, \eta_{(n)}$:
    $O(n \log n)$ time, assuming standard sorting algorithms
    such as Mergesort.
  \item 
    Compute the prefix sums
    $S_1, \dots, S_n$:
    $O(n)$ time in total.
  \item
    For each $i = 1, \dots, n$,
    compute $\bar{h}_{(i)}$ using
    \eqref{eq:auc-min-vs-max:hbar-constant-time}:
    constant time for each $i$, hence $O(n)$ time in total.
  \item
    Compute $\widehat{\Delta}_\auc$ and $\widehat{v}_\auc$
    by \eqref{eq:auc-min-vs-max:deltahat} and
    \eqref{eq:auc-min-vs-max:vhat}:
    $O(n)$ time.
\end{enumerate}

\begin{algorithm}[t]
  \caption{An $O(n \log n)$-time algorithm for computing $\widehat{\Delta}_\auc$ and $\widehat{v}_\auc$}
  \label{alg:auc-delta-vhat}
  \begin{algorithmic}
    \INPUT Soft labels $\eta_{1}, \dots, \eta_{n}$, class prior $\theta$
    \STATE $\eta_{(1)}, \dots, \eta_{(n)} \gets$ sort $\eta_{1}, \dots, \eta_{n}$ in ascending order
    \STATE $S_{0} \gets 0$
    \FOR{$i = 1$ \textbf{to} $n$}
      \STATE $S_{i} \gets S_{i-1} + (\eta_{(i)} - \theta)^{2}$
    \ENDFOR
    \FOR{$i = 1$ \textbf{to} $n$}
      \STATE $\bar{h}_{(i)} \gets
      \frac{1 - 2 \theta}{2 (n - 1)}
      \left[
        (2i - n) (\eta_{(i)} - \theta)^{2} + S_{n} - 2 S_{i}
      \right]$
    \ENDFOR
    \STATE $\widehat{\Delta}_{\auc} \gets \frac{1}{n} \sum_{i=1}^{n} \bar{h}_{i}$
    \STATE $\widehat{v}_\auc \gets
    \frac{4 (n - 1)}{(n - 2)^{2}}
    \sum_{i=1}^{n}
    \paren{
      \bar{h}_{i} - \widehat{\Delta}_\auc
    }^{2}$
    \OUTPUT $\widehat{\Delta}_{\auc}$, $\widehat{v}_\auc$
  \end{algorithmic}
\end{algorithm}

\subsubsection{Sufficient conditions for non-negative discriminants}

Here, we discuss some sufficient conditions
for $\Delta_\ber, \Delta_{\auc} \geq 0$.

\begin{proposition}
  \label{prop:min-vs-max:formal}
  Assume that the random variable $\eta(x)$ 
  is not a.s.\ constant.
  Then, we have
  $\Delta_{\ber} \geq 0$ and $\Delta_{\auc} \geq 0$
  if one of the following conditions holds:
  \begin{enumerate}[label=(\roman*)]
    \item 
      \label{item:auc-min-vs-max:formal:balanced-classes}
      The two classes are balanced, i.e., $\theta = \frac{1}{2}$.
    \item 
      \label{item:auc-min-vs-max:formal:symmetric}
      The law of $\eta(x)$ is symmetric around its mean $\theta$
      (which may or may not be equal to $\frac{1}{2}$).
    \item 
      \label{item:min-vs-max:formal:small-bayes-error}
      The Bayes error satisfies $\besterr \leq \delta_{\theta}$,
      where $\delta_{\theta} > 0$ is a certain constant 
      depending only on $\theta$.
  \end{enumerate}
\end{proposition}

\ref{item:auc-min-vs-max:formal:balanced-classes}
is trivial by the definitions of 
$\Delta_{\ber}$ and $\Delta_{\auc}$.
In what follows, we prove the rest of \cref{prop:min-vs-max:formal}.
Specifically, 
\ref{item:auc-min-vs-max:formal:symmetric}
follows from
\cref{lem:min-vs-max:symmetric};
we prove 
\ref{item:min-vs-max:formal:small-bayes-error}
as \cref{lem:min-vs-max:small-bayes-error}.

\begin{lemma}
  \label{lem:min-vs-max:symmetric}
  Assume that the distribution of $\zeta$ is symmetric around $0$,
  i.e., $\zeta$ and $-\zeta$ are identically distributed.
  Then, we have the following:
  \begin{enumerate*}[label=(\roman*)]
    \item 
      $\E{\zeta |\zeta|} = 0$ and
    \item
      $\E{\zeta |\zeta - \zeta'|} = 0$.
  \end{enumerate*}
\end{lemma}

\begin{proof}
  \begin{enumerate}[label=(\roman*)]
    \item 
      By symmetry,
      we have
      \begin{equation}
        \E{\zeta |\zeta|}
        =
        \E{-\zeta |-\zeta|}
        =
        - \E{\zeta |\zeta|}
        \text{,}
      \end{equation}
      which implies $\E{\zeta |\zeta|} = 0$.
    \item
      By symmetry, $(\zeta, \zeta')$ and $(-\zeta, -\zeta')$ are identically distributed,
      and hence
      \begin{equation}
        \E{\zeta |\zeta - \zeta'|}
        =
        \E{-\zeta |-\zeta + \zeta'|}
        =
        - \E{\zeta |\zeta - \zeta'|}
        \text{.}
      \end{equation}
      Therefore, we have $\E{\zeta |\zeta - \zeta'|} = 0$.
  \end{enumerate}
\end{proof}

\begin{example}
  \label{ex:min-vs-max:two-point}
  Consider the following two-point distribution:
  \begin{equation}
    \begin{cases}
      \eta(x) = a & \text{with probability } p \text{,} \\
      \eta(x) = b & \text{with probability } 1 - p \text{,}
    \end{cases}
  \end{equation}
  where $0 \leq a < b \leq 1$ and $0 < p < 1$.
  Then, we have
  $\theta = \E{\eta(x)} = a p + b (1 - p)$,
  and
  it can be shown that
  \begin{equation}
    \Delta_{\ber}
    =
    \Delta_{\auc}
    = 
    (b - a)^2 p (1 - p) \cdot (1 - 2 \theta) (2p - 1)
    \text{.}
  \end{equation}
  Since the two discriminants coincide in this case,
  let us write $\Delta \coloneqq \Delta_{\ber} = \Delta_{\auc}$
  for ease of notation.
  Two immediate observations are:

  \begin{enumerate}[label=(\roman*)]
    \item If $\theta = \frac{1}{2}$, we have $\Delta = 0$.
    This is consistent with
    \cref{prop:min-vs-max:formal}~\cref{item:auc-min-vs-max:formal:balanced-classes}.
    \item If $p = \frac{1}{2}$, we have $\Delta = 0$
    no matter how we choose $a$ and $b$.
    This is the symmetric case in 
    \cref{prop:min-vs-max:formal}~\cref{item:auc-min-vs-max:formal:symmetric}. 
  \end{enumerate}

  In general, the sign of $\Delta$
  coincides with that of
  \begin{equation}
    \frac{\Delta}{4 (b - a) p (1 - p)}
    =
    \paren{\theta - \frac{1}{2}} \paren{\theta - \frac{a+b}{2}}
  \end{equation}
  since $4 (b - a) p (1 - p) > 0$.\footnote{
    Note that $\frac{a + b}{2}$ is 
    the class prior $\theta = ap + b (1 - p)$
    for $p = \frac{1}{2}$.
  }
  Therefore, we have $\Delta < 0$ if
  \begin{equation}
    \label{ex:min-vs-max:two-point:negative-delta-condition}
    \min \set{
      \frac{1}{2}, \frac{a + b}{2}
    }
    <
    \theta
    <
    \max \set{
      \frac{1}{2}, \frac{a + b}{2}
    }
  \end{equation}
  and $\Delta \geq 0$ otherwise.
  Some concrete examples are as follows:
  \begin{enumerate}[label=(\roman*),resume]
    \item 
      \label{item:min-vs-max:bernoulli}
      For distributions with $a + b = 1$,
      including the Bernoulli case ($a = 0$, $b = 1$),
      we have 
      $\Delta = 4 (b - a) p (1 - p) \cdot \paren{\theta - \frac{1}{2}}^{2} \geq 0$ 
      regardless of $p$.
      In particular, $\Delta > 0$ if $\theta \neq \frac{1}{2}$.
    \item 
      To see that $\Delta$ can indeed be negative,
      let $a = 0.1$ and $b = 0.7$.
      Then, any $p \in (\frac{1}{3}, \frac{1}{2})$
      satisfies
      \eqref{ex:min-vs-max:two-point:negative-delta-condition}
      and hence gives $\Delta < 0$.
  \end{enumerate}

\end{example}

While \cref{ex:min-vs-max:two-point}~\ref{item:min-vs-max:bernoulli}
is very simple, it has an interesting implication for 
a wider class of distributions;
that is, it suggests that we have 
$\Delta_{\ber}, \Delta_{\auc} \geq 0$
for any distribution with a sufficiently low 
label uncertainty (e.g., small Bayes error).
The rationale is as follows.
We have seen that $\Delta_{\ber} = \Delta_{\auc} > 0$ holds
for any Bernoulli distribution with mean $\theta \neq \frac{1}{2}$.
Since $\Delta_{\ber}$ and $\Delta_{\auc}$ are 
continuous as a functional of $\eta(x)$'s law 
(with respect to, e.g., the topology of weak convergence
or Wasserstein distance),
we can expect that they remain positive
for any distribution sufficiently close to
a Bernoulli distribution with mean $\theta$.
As Bernoulli distributions are the most ``certain'' distributions,
being close to a Bernoulli distribution
means that the distribution contains only a small amount of 
label uncertainty.
\cref{lem:min-vs-max:small-bayes-error} formalizes this intuition.

\begin{lemma}
  \label{lem:min-vs-max:small-bayes-error}
  We have 
  $\Delta_\ber \geq 0$
  and
  $\Delta_\auc \geq 0$
  for any distribution whose Bayes error $\besterr$
  and class prior $\theta$ satisfies
  \begin{equation}
    \label{eq:auc-min-vs-max:small-bayes-error:condition}
    \besterr
    \leq
    \frac{1}{4} \min \set{\theta, 1 - \theta} |1 - 2 \theta|
    \text{.}
  \end{equation}
\end{lemma}

\begin{proof}
  First, we prove the claim for $\Delta_\ber$.
  Let $h(x) = x |x|$ for
  $x \in [-\theta, 1 - \theta]$.
  Then, we have 
  $\abs{h'(x)} = 2 |x| \leq 2 \max \{ \theta, 1 - \theta \}$,
  and hence $h$ is
  $2 \max \{ \theta, 1 - \theta \}$-Lipschitz continuous
  with respect to $1$-norm.
  Therefore, for any $[0, 1]$-valued random variable $z$,
  we have
  \begin{equation}
    \label{eq:min-vs-max:small-bayes-error:ber}
    \begin{aligned}
      \abs{
        \E{h(\zeta)}
        -
        \E{h(z - \theta)}
      }
      &\leq
      2 \max \set{\theta, 1 - \theta}
      \E{
        |\zeta - (z - \theta)| 
      }
      \\ &=
      2 \max \set{\theta, 1 - \theta}
      \E{ |\eta(x) - z| }
      \text{.}
    \end{aligned}
  \end{equation}

  Now, take $z$ to be a random variable whose 
  conditional distribution given $\eta(x)$ is
  a Bernoulli distribution with mean $\eta(x)$.\footnote{
    Let $u \sim \mathrm{Unif}([0,1])$ be independent of $\eta(x)$,
    and define $z \coloneqq \I{u \leq \eta(x)}$. 
    Then,
    $\E{z \mid \eta(x)} = \E{\I{u \leq \eta(x)} \mid \eta(x)} = \eta(x)$,
    hence $z \mid \eta(x) \sim \mathrm{Bernoulli}(\eta(x))$.
  }
  Then, $z$ is a Bernoulli random variable 
  with mean $\theta$.
  Noting that
  $\E{ |\eta(x) - z| \mid \eta(x) }
  =
  2 \eta(x) (1 - \eta(x))$,
  we have
  \begin{equation}
    \label{eq:ber-min-vs-max:small-bayes-error:diff}
    \begin{aligned}
      \abs{
        \E{h(\zeta)}
        -
        \E{h(z - \theta)}
      }
      &\leq
      4 \max \set{\theta, 1 - \theta}
      \E{ \eta(x) (1 - \eta(x)) }
      \\ &\leq
      4 \max \set{\theta, 1 - \theta}
      \E{ \min \{ \eta(x), 1 - \eta(x) \} }
      \\ &=
      4 \max \set{\theta, 1 - \theta} \besterr
      \text{.}
    \end{aligned}
  \end{equation}

  Observe that $(1 - 2 \theta) \E{h(\zeta)} = \Delta_\ber$ 
  and that $(1 - 2 \theta) \E{h(z - \theta)} = \theta (1 - \theta) |1 - 2 \theta|^{2}$ 
  (by \cref{ex:min-vs-max:two-point}~\ref{item:min-vs-max:bernoulli}).
  Hence, \eqref{eq:ber-min-vs-max:small-bayes-error:diff}
  implies
  \begin{equation}
    \begin{aligned}
      \Delta_\ber
      &\geq
      \theta (1 - \theta) |1 - 2 \theta|^{2}
      -
      4 \max \set{\theta, 1 - \theta} |1 - 2 \theta| \besterr
      \\ &=
      4 \max \set{\theta, 1 - \theta} |1 - 2 \theta| 
      \paren{
        \frac{1}{4} \min \set{\theta, 1 - \theta} |1 - 2 \theta|
        -
        \besterr
      }
      \text{,}
    \end{aligned}
  \end{equation}
  which is non-negative if
  \begin{equation}
    \besterr
    \leq
    \frac{1}{4} \min \set{\theta, 1 - \theta} |1 - 2 \theta|
    \text{.}
  \end{equation}

  The proof for $\Delta_\auc$ is similar.
  Let $h(x, y) = \frac{1}{2} (x + y) |x - y|$ for
  $x, y \in [-\theta, 1 - \theta]$.
  Then, we have 
  $\norm[\infty]{\nabla h(x, y)} = \max \{ |x|, |y| \} \leq \max \{ \theta, 1 - \theta \}$,
  and hence $h$ is
  $\max \{ \theta, 1 - \theta \}$-Lipschitz continuous
  with respect to $1$-norm.
  Therefore, for any $[0, 1]$-valued random variable $z$
  and its i.i.d.\ copy $z'$,
  we have
  \begin{equation}
    \label{eq:min-vs-max:small-bayes-error:auc}
    \begin{aligned}
      \abs{
        \E{h(\zeta, \zeta')}
        -
        \E{h(z - \theta, z' - \theta)}
      }
      &\leq
      \max \set{\theta, 1 - \theta}
      \E{
        |\zeta - (z - \theta)| + |\zeta' - (z' - \theta)|
      }
      \\ &=
      2 \max \set{\theta, 1 - \theta}
      \E{ |\eta(x) - z| }
      \text{.}
    \end{aligned}
  \end{equation}
  Notice that
  the right-hand side of
  \eqref{eq:min-vs-max:small-bayes-error:auc}
  exactly matches that of
  \eqref{eq:min-vs-max:small-bayes-error:ber}.
  Hence, the rest of the proof is almost identical.
\end{proof}

\begin{remark}
  \begin{itemize}
    \item 
      It is possible to derive similar sufficient conditions
      in terms of the optimal BER or AUC instead of the Bayes error.
    \item
      This lemma implies that we can choose
      the constant $\delta_{\theta}$ 
      in
      \cref{prop:min-vs-max:formal}
      \ref{item:min-vs-max:formal:small-bayes-error}
      to be
      \begin{equation}
        \delta_{\theta}
        =
        \begin{cases}
          \frac{1}{4} \theta (1 - 2 \theta)
          & \text{if } \theta < \frac{1}{2} \text{,} \\
          \infty
          & \text{if } \theta = \frac{1}{2} \text{,} \\
          \frac{1}{4} (1 - \theta) (2 \theta - 1)
          & \text{if } \theta > \frac{1}{2} \text{,}
        \end{cases}
      \end{equation}
      which might not be optimal.
  \end{itemize}
\end{remark}

\subsection{
  Near-Linear Time Algorithm for Computing the Max Formula $\widehat{\bestauc_2}$ 
}
\label{sec:alg-bestauc-max}

In \cref{sec:alg-bestauc}, we presented an $O(n \log n)$-time algorithm
for computing the min formula $\widehat{\bestauc_1}$
or the optimal AUC.
Here we provide the proof of the correctness
and time complexity of \cref{alg:bestauc}.

\begin{theorem}[store=thm:alg-bestauc]
  \label{thm:alg-bestauc}
  \cref{alg:bestauc} computes $\widehat{\bestauc_1}$ in $O(n \log n)$ time
  in the worst case,
  assuming any standard sorting algorithm with
  $O(n \log n)$ worst-case time complexity (e.g., Mergesort).
\end{theorem}

\begin{proof}
  First, we prove the correctness of \cref{alg:bestauc}.
  Observe that 
  $\min \set{ a (1 - b), b (1 - a)} = \min \set{a, b} (1 - \max \set{a, b})$.
  Since $\eta_{(1)}, \dots, \eta_{(n)}$ are sorted in ascending order,
  we have
  \begin{equation}
    \begin{aligned}
      \sum_{i < j} 
      \min \set{ \eta_{(i)} (1 - \eta_{(j)}), \eta_{(j)} (1 - \eta_{(i)}) }
      &=
      \sum_{i < j} 
      \eta_{(i)} (1 - \eta_{(j)})
      \\ &=
      \sum_{j = 2}^{n} \left[
        (1 - \eta_{(j)})
        \sum_{i = 1}^{j - 1} \eta_{(i)} 
      \right]
      \text{.}
    \end{aligned}
  \end{equation}
  At the $j$-th iteration of the for loop in \cref{alg:bestauc},
  the variable $\mathrm{PREFIX}$ stores
  the prefix sum
  $\sum_{i = 1}^{j - 1} \eta_{(i)}$.
  Therefore, \cref{alg:bestauc} correctly computes 
  $\widehat{\bestauc_1}$.

  For the time complexity analysis,
  \cref{alg:bestauc} first sorts $\eta_{1}, \dots, \eta_{n}$, 
  which takes $O(n \log n)$ time by assumption.
  Then, the for loop iterates only $\Theta(n)$ times,
  and each iteration takes a constant time.
  Therefore, the total worse-case time complexity of \cref{alg:bestauc}
  is $O(n \log n)$.
\end{proof}

The max-based estimator $\widehat{\bestauc_2}$
can also be computed in $O(n \log n)$ time
by a similar algorithm, as shown in
\cref{alg:bestauc-max}.

\begin{algorithm}[t]
  \caption{An $O(n \log n)$-time algorithm for estimating $\bestauc$ with the max formula}
  \label{alg:bestauc-max}
  \begin{algorithmic}
    \INPUT Soft labels $\eta_{1}, \dots, \eta_{n}$, class prior $\theta$
    \STATE $\eta_{(1)}, \dots, \eta_{(n)} \gets$ sort $\eta_{1}, \dots, \eta_{n}$ in ascending order
    \STATE $\mathrm{SUM} \gets 0$
    \STATE $\mathrm{PREFIX} \gets 0$
    \FOR{$j = 2$ \textbf{to} $n$}
      \STATE $\mathrm{PREFIX} \gets \mathrm{PREFIX} + 1 - \eta_{(j-1)}$
      \STATE $\mathrm{SUM} \gets \mathrm{SUM} 
      + \eta_{(j)} \cdot \mathrm{PREFIX}$
    \ENDFOR
    \OUTPUT $\frac{\mathrm{SUM}}{n (n-1) \theta (1-\theta)}$
  \end{algorithmic}
\end{algorithm}

\begin{theorem}[store=thm:alg-bestauc-max]
  \label{thm:alg-bestauc-max}
  \cref{alg:bestauc-max} computes 
  $\widehat{\bestauc_2}$ in $O(n \log n)$ time
  in the worst case,
  assuming any standard comparison-based sorting algorithm with 
  $O(n \log n)$ worst-case time complexity (e.g., Mergesort).
\end{theorem}

\begin{proof}
  We prove the correctness of \cref{alg:bestauc-max}.
  Observe that 
  $\max \set{ a (1 - b), b (1 - a)} = \max \set{a, b} (1 - \min \set{a, b})$.
  Since $\eta_{(1)}, \dots, \eta_{(n)}$ are sorted in ascending order,
  we have
  \begin{equation}
    \begin{aligned}
      n (n - 1) \theta (1 - \theta) \widehat{\bestauc_2}
      &=
      \sum_{i<j} 
      \max \set{
        \eta_{i} (1 - \eta_{j}), 
        \eta_{j} (1 - \eta_{i})
      }
      \\ &=
      \sum_{i<j} 
      \max \set{
        \eta_{(i)} (1 - \eta_{(j)}), 
        \eta_{(j)} (1 - \eta_{(i)})
      }
      \\ &=
      \sum_{i<j} 
      \eta_{(j)} (1 - \eta_{(i)})
      \\ &=
      \sum_{j=2}^{n} 
      \sum_{i=1}^{j-1}
      \eta_{(j)} (1 - \eta_{(i)})
      \\ &=
      \sum_{j=2}^{n} 
      \left[
        \eta_{(j)} 
        \sum_{i=1}^{j-1}
        (1 - \eta_{(i)})
      \right]
    \end{aligned}
  \end{equation}
  At the $j$-th iteration of the for loop in \cref{alg:bestauc-max},
  the variable $\mathrm{PREFIX}$ stores
  the prefix sum
  $\sum_{i = 1}^{j - 1} \eta_{(i)}$.
  Therefore, \cref{alg:bestauc-max} correctly computes 
  $\widehat{\bestauc}$.

  For the time complexity analysis,
  \cref{alg:bestauc-max} first sorts $\eta_{1}, \dots, \eta_{n}$, 
  which takes $O(n \log n)$ time by assumption.
  Then, the for loop iterates $\Theta(n)$ times,
  and each iteration takes a constant time.
  Therefore, the total worse-case time complexity of \cref{alg:bestauc-max}
  is $O(n \log n)$.
\end{proof}

\begin{remark}
  The algorithm for computing the discriminant estimator
  $\widehat{\Delta}_\auc$
  in $O(n \log n)$ time
  is shown in \cref{alg:auc-delta-vhat}.
\end{remark}

\subsection{Clipping the AUC estimator}
\label{sec:clip-auc}

Another minor note
about the estimation of $\bestauc$:
although the true value
$\bestauc$ lies between $0.5$ and $1$
for any data distribution,
it is possible that 
the estimates 
$\widehat{\bestauc_{1}}$
and 
$\widehat{\bestauc_{2}}$
can fall outside the valid range $[0.5, 1]$
in rare events.\footnote{
  For example, imagine that
  we draw a tiny sample of size $n = 2$,
  which happens to be $\set{\eta_{1}, \eta_{2}} = \set{\frac{1}{2}, \frac{1}{2}}$.
  Then, we have
  $\widehat{\bestauc_{1}} = 1 - \frac{1}{8 \theta (1 - \theta)}$,
  which is less than $0.5$
  unless $\theta = 0.5$.
}
To address this issue,
we can simply clip the raw estimates
to the valid range $[0.5, 1]$.
We denote the clipped estimators by
$
  \widetilde{\bestauc_k} 
  \coloneqq 
  \clip_{0.5}^{1}(\widehat{\bestauc_k})
  = \min \{
    1,
    \max \{
      0.5, 
      \widehat{\bestauc_k}
    \}
  \}
$,
$k = 1, 2$.
Here, $\clip_{a}^{b}(z)
=
\min \set{
  b,
  \max \set{ a, z }
}$ 
represents 
the operation that clips its argument to the interval
$[a, b]$.
It is straightforward to verify that
each clipped estimator is always 
more accurate than
the unclipped counterpart,
i.e.,
$|\widetilde{\bestauc_k} - \bestauc| \leq |\widehat{\bestauc_k} - \bestauc|$.
Again,
we present our theoretical analysis mainly for
the min formula estimator
$\widetilde{\bestauc_1}$
and denote it
just by $\widetilde{\bestauc}$
for brevity,
but the arguments for $\widetilde{\bestauc_2}$
are similar.
Note that the BER estimators do not require clipping
as both the true value $\bestber$
and the estimates 
$\widehat{\bestber_1}$,
$\widehat{\bestber_2}$
lie in $[0, 0.5]$.

\subsection{
  Estimation Error Bounds
  for
  Estimators Based on Clean Soft Labels
}
\label{sec:proof-clean:bound}

\subsubsection{
  Preliminaries:
  Hoeffding's 
  and
  Bernstein's
  inequalities
  for U-statistics
}

\emph{U-statistics} are a broad class of statistics 
introduced by \citet{hoeffding1948class}. 
Let $\phi: \R^d \to \R$ be a function 
that is symmetric, 
i.e., invariant to permutations of its $d$ arguments.
Given $n \geq d$ i.i.d.\ observations
$Z_1, \dots, Z_n$, 
consider the estimator
\begin{equation}
  \label{eq:u-statistic}
  U_n 
  \coloneqq 
  \binom{n}{d}^{-1} \sum_{1 \leq i_1 < \ldots < i_d \leq n} 
  \phi(Z_{i_1}, \ldots, Z_{i_d})
  \text{.}
\end{equation}
$U_n$ is called a U-statistic with a kernel $\phi$ of degree $d$. 

For $d = 1$, $U_n$ is simply an average of i.i.d.\ random variables,
so the standard concentration inequalities such as
the famous Hoeffding's inequality can be used
to analyze its deviation from the mean.
However, for $d \geq 2$, the summands in \eqref{eq:u-statistic}
are no longer independent.
Nonetheless, 
in the same paper as the one that introduced 
his famous concentration inequality 
for sums of independent random variables,
\citet{hoeffding1963probability}
established a similar concentration bound
for U-statistics:
if the kernel $\phi$ is bounded
in an interval $[a, b]$,
then
for any $\epsilon > 0$,
\begin{equation}
  \label{eq:hoeffding-u-statistic}
  \P{
    \abs{
      U_{n} - \E{U_{n}}
    }
    > \epsilon
  }
  \leq
  2 \exp\paren{
    - \frac{2 (n/d) \epsilon^2}{(b - a)^2}
  }
  \text{.}
\end{equation}
Or equivalently,
for any $\delta \in (0, 1)$,
with probability at least $1 - \delta$,
it holds that
\begin{equation}
  \label{eq:hoeffding-u-statistic:tail}
  \abs{
    U_{n} - \E{U_{n}}
  }
  \leq
  (b - a)
  \sqrt{
    \frac{\log (2/\delta)}{2 n/d}
  }
  \text{.}
\end{equation}
One thing to note is that,
although we can also apply McDiarmid's inequality
(also known as the bounded difference inequality)
to obtain a similar concentration bound,
Hoeffding's bound \eqref{eq:hoeffding-u-statistic:tail}
is tighter by a multiplicative factor of $\sqrt{d}$.

\citet{hoeffding1963probability}
also provided a Bernstein-type concentration bound
for U-statistics,
which takes into account the variance 
of the kernel:
if the kernel $\phi$ is bounded as
$|\phi(Z_{1}, \dots, Z_{d}) - \E{\phi(Z_{1}, \dots, Z_{d})}| \leq c$,
for any $\epsilon > 0$,
we have
\begin{equation}
  \label{eq:bernstein-u-statistic}
  \P{
    \abs{
      U_{n} - \E{U_{n}}
    }
    > \epsilon
  }
  \leq
  2 \exp\paren{
    - \frac{(n/d) \epsilon^2}{2 \paren{\sigma^{2} + c \epsilon / 3}}
  }
  \text{,}
\end{equation}
where
$\sigma^{2} \coloneqq \Var{\phi(Z_{1}, \dots, Z_{d})}$ 
is the variance of the kernel.
It also can be rewritten as a tail bound:
for any $\delta \in (0, 1)$,
with probability at least $1 - \delta$,
it holds that
\begin{equation}
  \label{eq:bernstein-u-statistic:tail}
  \abs{
    U_{n} - \E{U_{n}}
  }
  \leq
  \sqrt{
    \frac{2 \sigma^{2} \log (2/\delta)}{n/d}
  }
  +
  \frac{2 c \log (2/\delta)}{3 (n/d)}
  \text{.}
\end{equation}
Bernstein-type bounds 
\eqref{eq:bernstein-u-statistic},
\eqref{eq:bernstein-u-statistic:tail}
are often tighter than
Hoeffding-type bounds
\eqref{eq:hoeffding-u-statistic},
\eqref{eq:hoeffding-u-statistic:tail}
when the standard deviation $\sigma$
is much smaller than 
the width $b - a$ of the kernel's range.

\subsubsection{
  Proofs of the estimation error bounds
}

In the simplest setting of this paper
where 
clean soft labels $\set{\eta_{i}}_{i=1}^n$ 
are available
and the class prior $\theta$ is known,
our estimators 
$\widehat{\bestber}$
and
$\widehat{\bestauc}$
can be seen as U-statistics
of degrees $d = 1, 2$, respectively,
with bounded kernels.
Thus,
the above concentration inequalities
can be applied
to obtain the bounds
(\cref{thm:ber-auc:bound,thm:ber-auc:bound:auc:bernstein})
on the estimation errors
of these estimators.

The bound for 
$\widehat{\bestber}$ 
is a straightforward
application of
the Hoeffding bound
\eqref{eq:hoeffding-u-statistic:tail},
so we focus on 
the $\widehat{\bestauc}$ case 
in what follows.
The kernel
\begin{equation}
  \phimin_{\auc}(z_1, z_2) 
  = 
  1 - 
  \frac{1}{2 \theta (1 - \theta)} 
  \min \set{
    z_1 (1 - z_2), z_2 (1 - z_1)
  }
\end{equation}
takes values between
$a = 1 - \frac{1}{8 \theta (1 - \theta)}$
and
$b = 1$,
so
we can use
the Hoeffding bound
\eqref{eq:hoeffding-u-statistic:tail}
to obtain
the following result:
\begin{equation}
  \abs{
    \widehat{\bestauc}
    -
    \bestauc
  }
  \leq
  \frac{1}{8 \theta (1 - \theta)}
  \sqrt{
    \frac{\log (2/\delta)}{n}
  }
  \quad
  \text{with probability at least }
  1 - \delta 
  \text{.}
\end{equation}
The $\theta$-dependence of the right-hand side
comes from the fact that, 
unlike in the case of
$\widehat{\bestber}$,
the range $b - a = \frac{1}{8 \theta (1 - \theta)}$
of the kernel $\phimin$ depends on the class prior $\theta$.
Under severe class imbalance
where $\theta \to 0$ (or $1$),
the factor $\frac{1}{8 \theta (1 - \theta)}$
blows up at the rate of $O(1/\theta)$
(or $O(1/(1 - \theta))$).

We can mitigate this issue and slow down the blow-up rate
to $O(1/\sqrt{\theta})$
by using
the Bernstein bound
\eqref{eq:bernstein-u-statistic:tail}
instead of the Hoeffding bound.
The resulting bound is stated in
\cref{thm:ber-auc:bound:auc:bernstein}.

\begin{lemma}
  \label{lem:ber-auc:bound:auc:variance}
  The variance of the kernel $\phimin_{\auc}$
  is bounded as
  \begin{equation}
    \Var{\phimin_{\auc}(\eta(x), \eta(x'))}
    \leq
    \frac{1}{16 \theta (1 - \theta)}
    \text{.}
  \end{equation}
\end{lemma}

\begin{proof}
  Let $m \coloneqq \min \set{\eta(x) (1 - \eta(x')), \eta(x') (1 - \eta(x))}$.
  Then, we have $0 \leq m \leq \frac{1}{4}$
  and
  $\E{m} = 2 \theta (1 - \theta) (1 - \bestauc) \leq \theta (1 - \theta)$.
  So it follows that
  \begin{equation}
    \Var{m}
    \leq
    \E{m^2}
    \leq
    \frac{1}{4} \E{m}
    \leq
    \frac{1}{4} \theta (1 - \theta)
    \text{.}
  \end{equation}
  Therefore,
  we have
  \begin{equation}
    \Var{\phimin_{\auc}(\eta(x), \eta(x'))}
    =
    \frac{1}{4 \theta^2 (1 - \theta)^2}
    \Var{m}
    \leq
    \frac{1}{16 \theta (1 - \theta)}
    \text{.}
  \end{equation}
\end{proof}

\begin{lemma}
  \label{lem:ber-auc:bound:auc:range}
  The kernel $\phimin_{\auc}$
  is bounded as
  \begin{equation}
    \abs{
      \phimin_{\auc}(\eta(x), \eta(x'))
      -
      \E{\phimin_{\auc}(\eta(x), \eta(x'))}
    }
    \leq
    \frac{1}{8 \theta (1 - \theta)}
    \text{.}
  \end{equation}
\end{lemma}
\begin{proof}
  Continuing from the proof of \cref{lem:ber-auc:bound:auc:variance},
  $m = \min \set{\eta(x) (1 - \eta(x')), \eta(x') (1 - \eta(x))}$
  is bounded as
  $0 \leq m \leq \frac{1}{4}$,
  and its mean satisfies
  $\E{m} = 2 \theta (1 - \theta) (1 - \bestauc) \in [0, \theta (1 - \theta)]$.
  Thus,
  we have
  $m - \E{m} \leq \frac{1}{4}$ 
  and
  $\E{m} - m \leq \theta (1 - \theta) \leq \frac{1}{4}$,
  which implies
  \begin{equation}
    \abs{
      \phimin_{\auc}(\eta(x), \eta(x'))
      -
      \E{\phimin_{\auc}(\eta(x), \eta(x'))}
    }
    =
    \frac{1}{2 \theta (1 - \theta)}
    \abs{m - \E{m}}
    \leq
    \frac{1}{8 \theta (1 - \theta)}
    \text{.}
  \end{equation}
\end{proof}

\begin{lemma}
  \label{lem:ber-auc:bound:auc:bernstein}
  For any
  $\delta \in (0, 1)$
  and
  $n \geq \frac{\log (2/\delta)}{9 \theta (1 - \theta)}$,
  with probability at least $1 - \delta$,
  it holds that
  \begin{equation}
    \abs{
      \widetilde{\bestauc}
      -
      \bestauc
    }
    \leq
    \abs{
      \widehat{\bestauc}
      -
      \bestauc
    }
    \leq
    \sqrt{
      \frac{\log (2/\delta)}{\theta (1 - \theta) n}
    }
    \text{.}
  \end{equation}
\end{lemma}

\begin{proof}
  With \cref{lem:ber-auc:bound:auc:variance,lem:ber-auc:bound:auc:range}
  proven above,
  we are now ready to apply
  the Bernstein bound
  \eqref{eq:bernstein-u-statistic:tail}
  to the estimator $\widehat{\bestauc}$.
  As a result,
  for any $\delta \in (0, 1)$,
  with probability at least $1 - \delta$,
  we have
  \begin{equation}
    \abs{
      \widehat{\bestauc}
      -
      \bestauc
    }
    \leq
    \sqrt{
      \frac{\log (2/\delta)}{4 \theta (1 - \theta) n}
    }
    +
    \frac{\log (2/\delta)}{6 \theta (1 - \theta) n}
    \text{.}
  \end{equation}
  Furthermore,
  the first term dominates the second term
  since
  $n \geq \frac{\log (2/\delta)}{9 \theta (1 - \theta)}$,
  so
  the above bound can be simplified as
  \begin{equation}
    \abs{
      \widehat{\bestauc}
      -
      \bestauc
    }
    \leq
    \sqrt{
      \frac{\log (2/\delta)}{\theta (1 - \theta) n}
    }
    \text{.}
  \end{equation}

  Finally, the claim follows as
  we have $\abs{\widetilde{\bestauc} - \bestauc} \leq \abs{\widehat{\bestauc} - \bestauc}$
  by the definition of the clipped estimator $\widetilde{\bestauc}$.
\end{proof}

\getkeytheorem{thm:ber-auc:bound:auc:bernstein}

\begin{proof}
  If 
  $n < \frac{4 \log (2 / \delta)}{\theta (1 - \theta)}$,
  then the claim trivially holds
  because
  $\text{(RHS)} = \sqrt{\frac{\log (2 / \delta)}{\theta (1 - \theta) n}} > \frac{1}{2}$
  and the left-hand side is at most $\frac{1}{2}$.
  If
  $n \geq \frac{4 \log (2 / \delta)}{\theta (1 - \theta)}$,
  then the claim follows from
  \cref{lem:ber-auc:bound:auc:bernstein}.
\end{proof}

\subsection{Unknown Class Prior}
\label{sec:proof-clean:unknown-class-prior}

\subsubsection{Failure of the estimator by \citet{jeong2024data}}

\citet{jeong2024data} discussed the estimation of
the FPR and FNR of the Bayes optimal classifier
from soft labels, 
and they proposed to estimate the class prior by 
$\frac{1}{n} \sum_{i=1}^{n} \I{\eta_{i} \geq 0.5}$.
However, this estimator relies on the naive assumption that
$\P{y = 1} = \P{\eta(x) \geq 0.5}$.
It is easy to construct examples 
where this assumption does not hold.

\begin{example}
  \label{ex:jeong-class-prior-counterexample}
  Consider a simple distribution over 
  $\mathcal{X} \times \mathcal{Y} = \set{x_{0}} \times \set{0, 1}$ 
  where $\P{y = 0} = \P{y = 1} = 0.5$.
  Then, $\P{y = 1} = 0.5$ and $\P{\eta(x) \geq 0.5} = \P{0.5 \geq 0.5} = 1$
  does not match.
  In this case, the estimator by \citet{jeong2024data} always returns $1$
  however large the sample size is,
  while the true class prior is $0.5$.
  Therefore, their estimator is biased and statistically inconsistent.

  Also, it is trivial to extend this example to allow 
  $\mathcal{X}$ to have multiple (or even infinitely many) elements
  with various $\eta(x)$ values.
\end{example}

\subsubsection{Clipped estimator of the class prior and plug-in estimators}

Here, we first show that 
using Hoeffding's inequality
to bound $|\thetahat - \theta|$
(\cref{lem:unknown-class-prior})
results in a estimation error bound 
for the BER estimator
of order $O_p(\frac{1}{\theta (1 - \theta) \sqrt{n}})$
(\cref{thm:unknown-class-prior:ber-hoeffding}).
Then, 
by employing Bernstein's inequality
instead of Hoeffding's inequality
(\cref{lem:unknown-class-prior:bernstein}),
we improve the $\theta$-dependence
to $O_p(\frac{1}{\sqrt{\theta (1 - \theta) n}})$
(\cref{thm:unknown-class-prior:ber}).
After that, 
a similar result for the AUC estimator
(\cref{thm:unknown-class-prior:auc})
is presented.
Finally, we prove the lemmas used in the proofs
of these theorems.

\begin{lemma}[store=lem:unknown-class-prior]
  \label{lem:unknown-class-prior}
  Under \cref{ass:unknown-class-prior:epsilon},
  for any
  $\delta \in (0, 1)$
  and
  $n \geq 1$,
  with probability at least $1 - \delta$,
  we have
  $
  \lvert 
    \thetahat - \theta 
  \rvert
  \leq
  \sqrt{\frac{2 \log (2/\delta)}{n}}
  $.
\end{lemma}

\begin{proof}
  We decompose the estimation error as
  \begin{equation}
    \lvert 
      \thetahat - \theta 
    \rvert
    \leq
    \lvert 
      \clip_{\tau_{n}}^{1 - {\tau_{n}}}(\bar{\eta}) - \bar{\eta}
    \rvert
    +
    \lvert 
      \bar{\eta} - \theta
    \rvert
    \text{.}
  \end{equation}
  By Hoeffding's inequality, the second term is bounded as
  \begin{equation}
    \lvert 
      \bar{\eta} - \theta
    \rvert
    \leq \sqrt{\frac{\log (2 / \delta)}{2n}}
  \end{equation}
  with probability at least $1 - \delta$.
  The first term is the error due to clipping,
  which is at most $\tau_{n} = \frac{c}{n}$.
  For all
  $n \geq
  n_0 \coloneqq \lceil \frac{2c^{2}}{\log (\frac{2}{\delta})} \rceil$,
  we have
  $\frac{c}{n} \leq \sqrt{\frac{\log(2 / \delta)}{2n}}$ 
  and hence
  \begin{equation}
    \lvert \thetahat - \theta \rvert
    \leq 2 \sqrt{\frac{\log (2 / \delta)}{2n}}
    \text{.}
  \end{equation}
  Finally, the proof is completed by noting that
  the above inequality holds for all $n \geq 1$ 
  because
  $n_0 = 1$ 
  for any 
  $\delta \in (0, 1)$ 
  and
  $c \in (0, \frac{1}{2})$.
\end{proof}

\begin{theorem}
  \label{thm:unknown-class-prior:ber-hoeffding}
  For any $\epsilon > 0$,
  it holds that
  \begin{equation}
    \P{
      \abs{
        \widehat{\bestber}(\thetahat)
        -
        \bestber
      }
      > \epsilon
    }
    \leq
    4 \exp \paren{
      - \frac{
        n \theta^{2} (1 - \theta)^{2} \epsilon^{2}
      }{8}
    }
    \text{.}
  \end{equation}
  Or equivalently,
  for any $\delta \in (0, 1)$,
  with probability at least $1 - \delta$,
  \begin{equation}
    \abs{
      \widehat{\bestber}(\thetahat)
      -
      \bestber
    }
    \leq
    \frac{1}{\theta (1 - \theta)}
    \sqrt{
      \frac{8 \log (4 / \delta)}{n}
    }
    \text{.}
  \end{equation}
\end{theorem}

\begin{proof}
  We have
  \begin{equation}
    \P{
      \abs{
        \widehat{\bestber}(\thetahat)
        -
        \bestber
      }
      > \epsilon
    }
    \leq
    \P{
      \abs{
        \widehat{\bestber}(\theta)
        -
        \bestber
      }
      >
      \frac{\epsilon}{2}
    }
    +
    \P{
      \abs{
        \widehat{\bestber}(\thetahat)
        -
        \widehat{\bestber}(\theta)
      }
      >
      \frac{\epsilon}{2}
    }
    \text{.}
  \end{equation}
  The first term is bounded by Hoeffding's inequality
  \eqref{eq:hoeffding-u-statistic} as
  \begin{equation}
    \P{
      \abs{
        \widehat{\bestber}(\theta)
        -
        \bestber
      }
      >
      \frac{\epsilon}{2}
    }
    \leq
    2 \exp \paren{
      - 2n\epsilon^2
    }
    \text{.}
  \end{equation}
  As for the second term,
  we use \cref{lem:unknown-class-prior:ber}
  to obtain
  \begin{equation}
    \P{
      \abs{
        \widehat{\bestber}(\thetahat)
        -
        \widehat{\bestber}(\theta)
      }
      >
      \frac{\epsilon}{2}
    }
    \leq 
    \P{
      \frac{1}{\theta (1 - \theta)}
      \abs{\thetahat - \theta}
      >
      \frac{\epsilon}{2}
    }
    \text{,}
  \end{equation}
  which can be further bounded by
  \cref{lem:unknown-class-prior} as
  \begin{equation}
    \begin{aligned}
      \P{
        \abs{\thetahat - \theta}
        >
        \frac{\theta (1 - \theta)}{2}
        \epsilon
      }
      &\leq
      2 \exp \paren{
        - \frac{
          n
          \paren{
            \frac{\theta (1 - \theta)}{2}
            \epsilon
          }^2
        }{2}
      }
      \\ &=
      2 \exp \paren{
        - \frac{
          n \theta^{2} (1 - \theta)^{2} \epsilon^{2}
        }{8}
      }
      \text{.}
    \end{aligned}
  \end{equation}
  Since $0 < \theta (1 - \theta) \leq \frac{1}{4}$,
  we have
  $\exp(-2n\epsilon^{2}) \leq \exp \paren{- \frac{n \theta^{2} (1 - \theta)^{2} \epsilon^{2}}{8}}$, 
  and thus
  \begin{equation}
    \P{
      \abs{
        \widehat{\bestber}(\thetahat)
        -
        \bestber
      }
      > \epsilon
    }
    \leq
    4 \exp \paren{
      - \frac{
        n \theta^{2} (1 - \theta)^{2} \epsilon^{2}
      }{8}
    }
    \text{.}
  \end{equation}
\end{proof}

The following lemma is similar to \cref{lem:unknown-class-prior},
but it gives a tighter bound
especially under severe class imbalance,
i.e., when $\theta$ is close to $0$ or $1$.

\getkeytheorem{lem:unknown-class-prior:bernstein}

\begin{proof}
  The proof is similar to that of \cref{lem:unknown-class-prior},
  but we use Bernstein's inequality
  instead of Hoeffding's inequality.
  Since $\eta$ is a $[0, 1]$-valued random variable with mean $\theta$,
  its variance is at most that of a Bernoulli random variable with the same mean,
  that is, $\theta (1 - \theta)$.
  Also, $|\eta - \theta| \leq \max \set{\theta, 1 - \theta} \leq 1$.
  Hence by Bernstein's inequality,
  with probability at least $1 - \delta$,
  we have
  \begin{equation}
    \lvert 
      \bar{\eta} - \theta
    \rvert
    \leq
    \sqrt{\frac{2 \theta (1 - \theta) \log (2 / \delta)}{n}}
    +
    \frac{2 \log (2 / \delta)}{3n}
    \text{.}
  \end{equation}
  and thus
  \begin{equation}
    \begin{aligned}
      \lvert 
        \thetahat - \theta 
      \rvert
      &\leq
      \lvert 
        \clip_{\tau_{n}}^{1 - {\tau_{n}}}(\bar{\eta}) - \bar{\eta}
      \rvert
      +
      \lvert 
        \bar{\eta} - \theta
      \rvert
      \\ &\leq
      \frac{c}{n}
      +
      \lvert 
        \bar{\eta} - \theta
      \rvert
      \\ &\leq
      \sqrt{\frac{2 \theta (1 - \theta) \log (2 / \delta)}{n}}
      +
      \frac{
        2 \log (2 / \delta)
        +
        3c
      }{3n}
      \text{.}
    \end{aligned}
  \end{equation}
  If 
  $
    n 
    \geq 
    \frac
      {\paren{2 \log (2 / \delta) + 3c}^2}
      {18 \theta (1 - \theta) \log (2 / \delta)}
  $,
  the first term dominates the second term,
  and hence
  the above inequality is simplified as
  \begin{equation}
    \lvert 
      \thetahat - \theta 
    \rvert
    \leq
    2 \sqrt{\frac{2 \theta (1 - \theta) \log (2 / \delta)}{n}}
    \text{.}
  \end{equation}
\end{proof}

\begin{theorem}[store=thm:unknown-class-prior:ber]
  \label{thm:unknown-class-prior:ber}
  Under \cref{ass:unknown-class-prior:epsilon},
  for any 
  $\delta > 0$,
  the following holds
  with probability at least $1 - \delta$:
  \begin{equation}
    \abs{
      \widehat{\bestber}(\thetahat) - \bestber
    }
    \leq 
    \frac{17}{8}
    \sqrt{
      \frac{2 \log (4/\delta)}{\theta (1 - \theta) n}
    }
    \text{.}
  \end{equation}
\end{theorem}

\begin{proof}
  Fix an arbitrary $\delta \in (0, 1)$.
  If $n < \frac{4 \log (4 / \delta)}{\theta (1 - \theta)}$,
  then the claim trivially holds
  because
  $\text{(RHS)} > \sqrt{\frac{\log (4 / \delta)}{\theta (1 - \theta) n}} > \frac{1}{2}$
  and the left-hand side is at most $\frac{1}{2}$.
  Therefore, we assume
  $n \geq \frac{4 \log (4 / \delta)}{\theta (1 - \theta)}$
  from now on.

  By \cref{thm:ber-auc:bound},
  we have the following with probability 
  at least $1 - \frac{\delta}{2}$:
  \begin{equation}
    \abs{
      \widehat{\bestber}(\theta) - \bestber
    }
    \leq
    \sqrt{
      \frac{\log (4/\delta)}{8n}
    }
    \text{.}
  \end{equation}
  Next, observe
  \begin{equation}
    \begin{aligned}
      \frac
        {\paren{2 \log (4 / \delta) + 3c}^2}
        {18 \theta (1 - \theta) \log (4 / \delta)}
      &\leq
      \frac
        {\paren{
          \paren{
            2
            +
            \frac{3c}{2 \log 2}
          }
          \log (4 / \delta)
        }^2}
        {18 \theta (1 - \theta) \log (4 / \delta)}
      \\ &=
      \frac{\paren{2 + \frac{3c}{2 \log 2}}^{2}}{18} 
      \cdot
      \frac{
        \log (4 / \delta)
      }{
        \theta (1 - \theta)
      }
      \\ &<
      0.53 
      \frac{
        \log (4 / \delta)
      }{
        \theta (1 - \theta)
      }
      \\ &\leq
      n
      \text{.}
    \end{aligned}
  \end{equation}
  It follows that
  with probability at least $1 - \frac{\delta}{2}$,
  the gap between
  the plug-in estimator and
  the estimator with the true class prior
  can be bounded by
  \cref{lem:unknown-class-prior:ber,lem:unknown-class-prior:bernstein}
  as
  \begin{equation}
    \begin{aligned}
      \abs{
        \widehat{\bestber}(\thetahat) 
        - 
        \widehat{\bestber}(\theta)
      }
      &\leq
      \frac{1}{\theta (1 - \theta)}
      \abs{
        \thetahat - \theta
      }
      \\ &\leq
      \frac{1}{\theta (1 - \theta)}
      \sqrt{
        \frac{8 \theta (1 - \theta) \log (4/\delta)}{n}
      }
      \\ &=
      \sqrt{
        \frac{8 \log (4/\delta)}{\theta (1 - \theta) n}
      }
    \text{.}
    \end{aligned}
  \end{equation}
  Combining these two inequalities
  by union bound, with probability at least $1 - \delta$,
  we have
  \begin{equation}
    \abs{
      \widehat{\bestber}(\thetahat) 
      - 
      \bestber
    }
    \leq
    \sqrt{
      \frac{\log (4/\delta)}{8n}
    }
    +
    \sqrt{
      \frac{8 \log (4/\delta)}{\theta (1 - \theta) n}
    }
    \text{.}
  \end{equation}
  We can further use 
  $0 < \theta (1 - \theta) \leq \frac{1}{4}$
  to simplify the right-hand side as
  \begin{equation}
    \begin{aligned}
      \sqrt{
        \frac{\log (4/\delta)}{8n}
      }
      +
      \sqrt{
        \frac{8 \log (4/\delta)}{\theta (1 - \theta) n}
      }
      &\leq
      \frac{1}{8}
      \sqrt{
        \frac{2 \log (4/\delta)}{\theta (1 - \theta) n}
      }
      +
      2
      \sqrt{
        \frac{2 \log (4/\delta)}{\theta (1 - \theta) n}
      }
      \\ &=
      \frac{17}{8}
      \sqrt{
        \frac{2 \log (4/\delta)}{\theta (1 - \theta) n}
      }
      \text{,}
    \end{aligned}
  \end{equation}
  which completes the proof.
\end{proof}

\begin{theorem}[store=thm:unknown-class-prior:auc]
  \label{thm:unknown-class-prior:auc}
  Under \cref{ass:unknown-class-prior:epsilon},
  for any 
  $\delta > 0$,
  the following holds
  with probability at least $1 - \delta$:
  \begin{equation}
    \abs{
      \widetilde{\bestauc}(\thetahat) - \bestauc
    }
    \leq 
    10
    \sqrt{
      \frac{\log (4/\delta)}{\theta (1 - \theta) n}
    }
    \text{.}
  \end{equation}
\end{theorem}

\begin{proof}
  First, if $n < \frac{18 \log (4 / \delta)}{\theta (1 - \theta)}$,
  then the claim follows trivially
  because
  $\text{(RHS)} = 10 \sqrt{\frac{\log (4 / \delta)}{\theta (1 - \theta) n}} > \frac{5 \sqrt{2}}{3} > \frac{1}{2}$
  and the left-hand side is at most $\frac{1}{2}$.
  Therefore, we assume
  $n \geq \frac{18 \log (4 / \delta)}{\theta (1 - \theta)}$
  in the following.

  By \cref{thm:ber-auc:bound:auc:bernstein},
  with probability at least $1 - \frac{\delta}{2}$,
  we have
  \begin{equation}
    \abs{
      \widehat{\bestauc}(\theta)
      -
      \bestauc
    }
    \leq
    \sqrt{
      \frac{\log (4/\delta)}{\theta (1 - \theta) n}
    }
    \leq
    \frac{1}{2}
    \text{.}
  \end{equation}
  In this event,
  we have
  \begin{equation}
    \abs{
      \widehat{\bestauc}(\thetahat)
      -
      \widehat{\bestauc}(\theta)
    }
    \leq
    \frac{1 / 2 + 1 / 2}{\thetahat (1 - \thetahat)}
    |\thetahat - \theta|
    =
    \frac{1}{\thetahat (1 - \thetahat)}
    |\thetahat - \theta|
    \text{.}
  \end{equation}
  by \cref{lem:unknown-class-prior:auc}.
  Also, 
  by
  \cref{lem:unknown-class-prior:auc2},
  with probability at least $1 - \frac{\delta}{2}$,
  it holds that
  \begin{equation}
    \frac{1}{\thetahat (1 - \thetahat)}
    \lvert 
      \thetahat - \theta 
    \rvert
    \leq
    6
    \sqrt{\frac{2 \log (4 / \delta)}{\theta (1 - \theta) n}}
    \text{.}
  \end{equation}

  Therefore, by union bound,
  with probability at least $1 - \delta$,
  we have
  \begin{equation}
    \begin{aligned}
      \abs{
        \widetilde{\bestauc}(\thetahat)
        -
        \bestauc
      }
      &\leq
      \abs{
        \widehat{\bestauc}(\thetahat)
        -
        \bestauc
      }
      \\ &\leq
      \abs{
        \widehat{\bestauc}(\theta)
        -
        \bestauc
      }
      +
      \abs{
        \widehat{\bestauc}(\thetahat)
        -
        \widehat{\bestauc}(\theta)
      }
      \\ &\leq
      \sqrt{
        \frac{\log (4/\delta)}{\theta (1 - \theta) n}
      }
      +
      6
      \sqrt{\frac{2 \log (4 / \delta)}{\theta (1 - \theta) n}}
      \\ &<
      10
      \sqrt{
        \frac{\log (4/\delta)}{\theta (1 - \theta) n}
      }
      \text{.}
    \end{aligned}
  \end{equation}
\end{proof}

\begin{proposition}
  \label{prop:unknown-class-prior:ber-mse}
  The mean squared error of
  $\widehat{\bestber}(\thetahat)$
  is
  \begin{equation}
    \label{eq:unknown-class-prior:ber-mse}
    \E{
      \paren{
        \widehat{\bestber}(\thetahat) 
        -
        \bestber
      }^{2}
    }
    \leq
    \frac{C}{\theta (1 - \theta) n}
    \text{,}
  \end{equation}
  where $C > 0$ is a constant that does not depend on $\theta$ or $n$.
  Therefore
  $\widehat{\bestber}(\thetahat)$
  is asymptotically unbiased.
\end{proposition}

\begin{proof}
  Take any
  $\epsilon > 0$
  and let
  $\delta \coloneqq 4 \exp \paren{- K \theta (1 - \theta) n \epsilon^{2}}$,
  where $K = \frac{8^{2}}{2 \cdot 17^{2}}$
  is a constant.
  By
  applying
  \cref{thm:unknown-class-prior:ber}
  with
  the above $\delta$,
  we have
  \begin{equation}
    \P{
      \abs{
        \widehat{\bestber}(\thetahat) 
        -
        \bestber
      }
      >
      \epsilon
    }
    \leq
    4 \exp \paren{
      - K \theta (1 - \theta) n \epsilon^{2}
    }
    \text{.}
  \end{equation}
  Integrating 
  the above tail bound
  over all $\epsilon > 0$ gives
  \begin{equation}
    \begin{aligned}
      \E{
        \paren{
          \widehat{\bestber}(\thetahat) 
          -
          \bestber
        }^{2}
      }
      &=
      \int_{0}^{\infty}
      \P{
        \paren{
          \widehat{\bestber}(\thetahat)
          -
          \bestber
        }^2
        >
        \epsilon
      }
      \, d\epsilon
      \\ &=
      \int_{0}^{\infty}
      \P{
        \abs{
          \widehat{\bestber}(\thetahat)
          -
          \bestber
        }
        >
        \sqrt{\epsilon}
      }
      \, d\epsilon
      \\ &\leq
      \int_{0}^{\infty}
      4 \exp \paren{
        - K \theta (1 - \theta) n \epsilon
      }
      \, d\epsilon
      \\ &=
      \frac{4/K}{\theta (1 - \theta) n}
      \text{,}
    \end{aligned}
  \end{equation}
  which proves \eqref{eq:unknown-class-prior:ber-mse}
  with $C = \frac{4}{K}$.

  Asymptotic unbiasedness follows from
  \begin{equation}
    \begin{aligned}
      \abs{
        \E{
          \widehat{\bestber}(\thetahat)
        }
        -
        \bestber
      }
      &\leq
      \E{
        \abs{
          \widehat{\bestber}(\thetahat)
          -
          \bestber
        }
      }
      \\ &\leq
      \sqrt{
        \E{
          \paren{
            \widehat{\bestber}(\thetahat)
            -
            \bestber
          }^{2}
        }
      }
      \\ &\leq
      \sqrt{
        \frac{C}{\theta (1 - \theta) n}
      }
      \\ &\xrightarrow{n \to \infty} 0
      \text{.}
    \end{aligned}
  \end{equation}
\end{proof}

\begin{proposition}
  \label{prop:unknown-class-prior:auc-mse}
  The mean squared error of
  $\widehat{\bestauc}(\thetahat)$
  is
  \begin{equation}
    \E{
      \paren{
        \widehat{\bestauc}(\thetahat) 
        -
        \bestauc
      }^{2}
    }
    \leq
    \frac{C}{\theta (1 - \theta) n}
    \text{,}
  \end{equation}
  where $C > 0$ is a constant that does not depend on $\theta$ or $n$.
  Therefore
  $\widehat{\bestauc}(\thetahat)$
  is asymptotically unbiased.
\end{proposition}

\begin{proof}
  We omit the proof as it is almost the same as the one of 
  \cref{prop:unknown-class-prior:ber-mse}.
\end{proof}

In the following, we prove the lemmas used 
in the proof of 
\cref{thm:unknown-class-prior:ber,thm:unknown-class-prior:auc}.

\begin{lemma}
  \label{lem:unknown-class-prior:ber}
  For any $\theta' \in (0, 1)$,
  we have
  \begin{equation}
    \abs{
      \widehat{\bestber}(\theta')
      -
      \widehat{\bestber}(\theta)
    }
    \leq
    \frac{1}{\theta (1 - \theta)}
    \abs{\theta' - \theta}
    \text{.}
  \end{equation}
\end{lemma}

\begin{proof}
  Fix an arbitrary $\eta \in [0, 1]$ and let
  \begin{equation}
    N(\theta)
    \coloneqq
    \min \set{
      \eta (1 - \theta),
      \theta (1 - \eta)
    }
    \text{,}
    \quad
    D(\theta)
    \coloneqq
    \theta (1 - \theta)
    \text{.}
  \end{equation}
  Then, we have 
  \begin{equation}
    \min \set{
      \frac{\eta}{\theta},
      \frac{1 - \eta}{1 - \theta}
    }
    =
    \frac{N(\theta)}{D(\theta)}
    \text{.}
  \end{equation}
  Observe that any $\theta' \in (0, 1)$ satisfies
  \begin{equation}
    \abs{
      \frac{N(\theta')}{D(\theta')}
      -
      \frac{N(\theta)}{D(\theta)}
    }
    \leq
    \frac{1}{D(\theta) D(\theta')}
    \paren{
      D(\theta') \abs{N(\theta') - N(\theta)}
      +
      N(\theta') \abs{D(\theta') - D(\theta)}
    }
  \end{equation}
  and that
  \begin{equation}
    \abs{N(\theta') - N(\theta)}
    \leq
    \abs{\theta' - \theta}
    \text{,}
    \quad
    \abs{D(\theta') - D(\theta)}
    \leq
    \abs{\theta' - \theta}
    \text{.}
  \end{equation}
  Therefore, we have
  \begin{equation}
    \abs{
      \frac{N(\theta')}{D(\theta')}
      -
      \frac{N(\theta)}{D(\theta)}
    }
    \leq
    \frac{D(\theta') + N(\theta')}{D(\theta) D(\theta')}
    \abs{\theta' - \theta}
    \leq
    \frac{2}{D(\theta)}
    \abs{\theta' - \theta}
    \text{,}
  \end{equation}
  where we used the fact that
  $N(\theta') \leq D(\theta')$
  for the last inequality.

  Using the above inequality, we obtain
  \begin{equation}
    \label{eq:unknown-class-prior:ber:intermediate}
    \begin{aligned}
      \abs{
        \widehat{\bestber}(\theta')
        -
        \widehat{\bestber}(\theta)
      }
      &\leq
      \frac{1}{2 n}
      \sum_{i=1}^{n}
      \abs{
        \min \set{
          \frac{\eta_{i}}{\theta'},
          \frac{1 - \eta_{i}}{1 - \theta'}
        }
        -
        \min \set{
          \frac{\eta_{i}}{\theta},
          \frac{1 - \eta_{i}}{1 - \theta}
        }
      }
      \\ &\leq
      \frac{1}{2 n}
      \sum_{i=1}^{n}
      \frac{2}{\theta (1 - \theta)}
      \abs{\theta' - \theta}
      \\ &=
      \frac{1}{\theta (1 - \theta)}
      \abs{\theta' - \theta}
      \text{.}
    \end{aligned}
  \end{equation}
\end{proof}

\begin{remark}
  In \cref{lem:unknown-class-prior:ber},
  we could have simply used the fact that
  $\theta \mapsto \min \set{\frac{\eta}{\theta}, \frac{1 - \eta}{1 - \theta}}$ 
  is $\frac{1}{\min \set{\theta, 1 - \theta}^{2}}$-Lipschitz
  to obtain 
  $
    \abs{\widehat{\bestber}(\theta') - \widehat{\bestber}(\theta)}
    \leq
    \frac{1}{\min \set{\theta, 1 - \theta}^{2}}
    \abs{\theta' - \theta}
  $,
  but it would have given a worse $\theta$-dependence.
\end{remark}

\begin{lemma}
  \label{lem:unknown-class-prior:auc}
  For any $\theta' \in (0, 1)$,
  it holds that
  \begin{equation}
    \abs{
      \widehat{\bestauc}(\theta')
      -
      \widehat{\bestauc}(\theta)
    }
    \leq
    \frac{
      1/2
      +
      \abs{
        \widehat{\bestauc}(\theta) - \bestauc
      }
    }{\theta' (1 - \theta')}
    \abs{\theta' - \theta}
    \text{.}
  \end{equation}
\end{lemma}

\begin{proof}
  We first show that
  \begin{equation}
    \label{eq:unknown-class-prior:auc:intermediate}
    \abs{
      \widehat{\bestauc}(\theta')
      -
      \widehat{\bestauc}(\theta)
    }
    \leq
    \frac{
      \abs{1 - \widehat{\bestauc}(\theta)}
    }{
      \theta' (1 - \theta')
    }
    \abs{\theta' - \theta}
    \text{.}
  \end{equation}
  If $\widehat{\bestauc}(\theta) = 1$,
  \eqref{eq:unknown-class-prior:auc:intermediate}
  trivially holds because
  we also have
  $\widehat{\bestauc}(\theta') = 1$.
  If $\widehat{\bestauc}(\theta) < 1$,
  we have
  \begin{equation}
    \begin{aligned}
      \abs{
        \frac{
          1 - \widehat{\bestauc}(\theta')
        }{
          1 - \widehat{\bestauc}(\theta)
        }
        - 1
      }
      &=
      \abs{
        \frac{\theta(1 - \theta)}{\theta' (1 - \theta')} - 1
      }
      \\ &=
      \abs{
        \frac{\theta' (1 - \theta') - \theta (1 - \theta)}{\theta' (1 - \theta')}
      }
      \\ &\leq
      \frac{\abs{\theta' - \theta}}{\theta' (1 - \theta')}
      \text{,}
    \end{aligned}
  \end{equation}
  where the last inequality follows from the fact that
  $x \in (0, 1) \mapsto x (1 - x)$ is $1$-Lipschitz.
  Now,
  \eqref{eq:unknown-class-prior:auc:intermediate}
  follows by multiplying both sides
  by $|1 - \widehat{\bestauc}(\theta)|$.

  Next,
  observe that
  the AUC of the trivial scoring function $f(x) \equiv 0$
  is
  \begin{equation}
      \auc(f)
      =
      \E{
        \I{0 > 0} 
        + 
        \frac{1}{2} \I{0 = 0}
      }
      =
      \frac{1}{2}
      \text{,}
  \end{equation}
  and hence 
  the maximum AUC over all scoring functions
  satisfies $\bestauc \in [\frac{1}{2}, 1]$.
  Using this fact,
  we have
  \begin{equation}
    \abs{1 - \widehat{\bestauc}(\theta)}
    \leq
    \abs{1 - \bestauc}
    +
    \abs{
      \widehat{\bestauc}(\theta) - \bestauc
    }
    \leq
    \frac{1}{2}
    +
    \abs{
      \widehat{\bestauc}(\theta) - \bestauc
    }
    \text{.}
  \end{equation}
  Combining this with
  \eqref{eq:unknown-class-prior:auc:intermediate},
  we obtain the desired result.
\end{proof}

\begin{remark}
  In \cref{lem:unknown-class-prior:auc},
  we could have argued more simply by using
  $
    \abs{\frac{1}{\theta' (1 - \theta')} - \frac{1}{\theta (1 - \theta)}}
    \leq
    \frac{1}{\theta' (1 - \theta') \theta (1 - \theta)}
    \abs{\theta' - \theta}
  $,
  but it would have given a worse $\theta$-dependence in the final result.
\end{remark}

\begin{lemma}
  \label{lem:unknown-class-prior:auc2}
  Suppose \cref{ass:unknown-class-prior:epsilon}
  holds.
  Then, for any $\delta \in (0, 1)$
  and
  $n \geq \frac{18 \log (2 / \delta)}{\theta (1 - \theta)}$,
  with probability at least $1 - \delta$,
  it holds that
  \begin{equation}
    \frac{1}{\thetahat (1 - \thetahat)}
    \lvert 
      \thetahat - \theta 
    \rvert
    \leq
    6
    \sqrt{\frac{2 \log (2 / \delta)}{\theta (1 - \theta) n}}
    \text{.}
  \end{equation}
\end{lemma}

\begin{proof}
  Some calculation shows that 
  any $c \in (0, \frac{1}{2})$,
  $\delta \in (0, 1)$
  and $\theta \in (0, 1)$
  satisfy
  $
    \frac{18 \log (2 / \delta)}{\theta (1 - \theta)}
    \geq
    \frac
      {\paren{2 \log (2 / \delta) + 3c}^2}
      {18 \theta (1 - \theta) \log (2 / \delta)}
  $.
  Therefore,
  by \cref{lem:unknown-class-prior:bernstein},
  with probability at least $1 - \delta$,
  we have
  \begin{equation}
    \label{eq:unknown-class-prior:auc2:intermediate}
    \lvert 
      \thetahat - \theta 
    \rvert
    \leq
    \sqrt{\frac{8 \theta (1 - \theta) \log (2 / \delta)}{n}}
    \text{.}
  \end{equation}
  Since $n \geq \frac{18 \log (2 / \delta)}{\theta (1 - \theta)}$,
  it follows that
  \begin{equation}
    \lvert 
      \thetahat - \theta 
    \rvert
    \leq
    \sqrt{8 \theta (1 - \theta) \log (2 / \delta)}
    \cdot
    \sqrt{
      \frac{\theta (1 - \theta)}{18 \log (2 / \delta)}
    }
    \leq
    \frac{2 \theta (1 - \theta)}{3}
    \text{,}
  \end{equation}
  which implies
  \begin{equation}
    \thetahat (1 - \thetahat)   
    \geq
    \theta (1 - \theta)
    -
    \lvert 
      \thetahat - \theta 
    \rvert
    \geq
    \frac{1}{3}
    \theta (1 - \theta)
    \text{.}
  \end{equation}
  Therefore, again by \eqref{eq:unknown-class-prior:auc2:intermediate},
  we have
  \begin{equation}
    \frac{1}{\thetahat (1 - \thetahat)}
    \lvert 
      \thetahat - \theta 
    \rvert
    \leq
    \frac{3}{\theta (1 - \theta)}
    \sqrt{\frac{8 \theta (1 - \theta) \log (2 / \delta)}{n}}
    =
    6
    \sqrt{\frac{2 \log (2 / \delta)}{\theta (1 - \theta) n}}
    \text{.}
  \end{equation}
\end{proof}

\section{Estimation from corrupted soft labels}
\label{sec:corrupted}

In practice, 
clean soft labels
$\eta_{i} = \P{y = 1 \mid x = x_{i}}$
may be unavailable.
If the data generation process allows 
repeated sampling of class labels for 
the same input $x_{i}$,
it is still possible to approximate $\eta_{i} = \E{y \mid x = x_{i}}$
by the average of the sampled labels,
in which case
the resulting plug-in Bayes error estimator
has been shown to be asymptotically unbiased
\citep{ishida2023is,ushio2025practical}.

However, it is often difficult to obtain
more than one class label per input
due to various reasons.
For example,
the construction of image classification datasets,
e.g., CIFAR-10 \citep{cifar10},
typically proceeds as follows:
first, images are retrieved from the web
by searching with keywords
associated with each class name;
then, human annotators filter out
images that are irrelevant to the class;
finally, the images are downsampled to a fixed resolution.
As such, each class label is associated 
with the original high-resolution image before downsampling.
Although it is possible to
collect multiple class labels for each
low-resolution image after downsampling
as in the CIFAR-10H dataset \citep{cifar10h},
the downsampling process makes the class posterior distribution
much more uncertain 
than that of the original high-resolution image,
distorting the soft label estimates.

A natural way to model such situations 
is to assume that
we have access to
a single class label $y_{i} \in \set{0, 1}$
plus 
a \emph{corrupted soft label}
$\xi_{i} = f(\eta_{i}) \in [0, 1]$,
instead of the clean soft label $\eta_{i}$,
for each input $x_{i}$ \citep{ushio2025practical}.
Here, the corrupted soft label $\xi_{i}$
is the clean soft label $\eta_{i}$
skewed by some unknown increasing transformation $f: [0, 1] \to [0, 1]$.
In the CIFAR-10/10H example above,
$\xi_{i}$ is the class posterior of the low-resolution image,
which can be obtained by averaging the labels from CIFAR-10H,
and
$y_{i}$ is the class label from CIFAR-10.

Here, we consider the problem of
estimating the optimal 
BER and AUC
from a dataset of $n$ pairs
$(\xi_{1}, y_{1}), \dots, (\xi_{n}, y_{n})$.
We first describe an algorithm based on isotonic regression
\citep{ayer1955empirical}
to recover the clean soft labels
from the corrupted ones.
Next, we discuss
the estimation of the class prior.
Then,
we present finite-sample
error bounds
for the resulting plug-in estimators.
Finally, we extend these bounds to
a noisy setting where
the observed soft labels may not preserve
the ordering of the clean soft labels.

\paragraph{Recovery of soft labels}

\begin{wrapfigure}[22]{r}{0.4\textwidth}
  \vspace{-\intextsep}
  \centering
  \includegraphics[width=\linewidth]{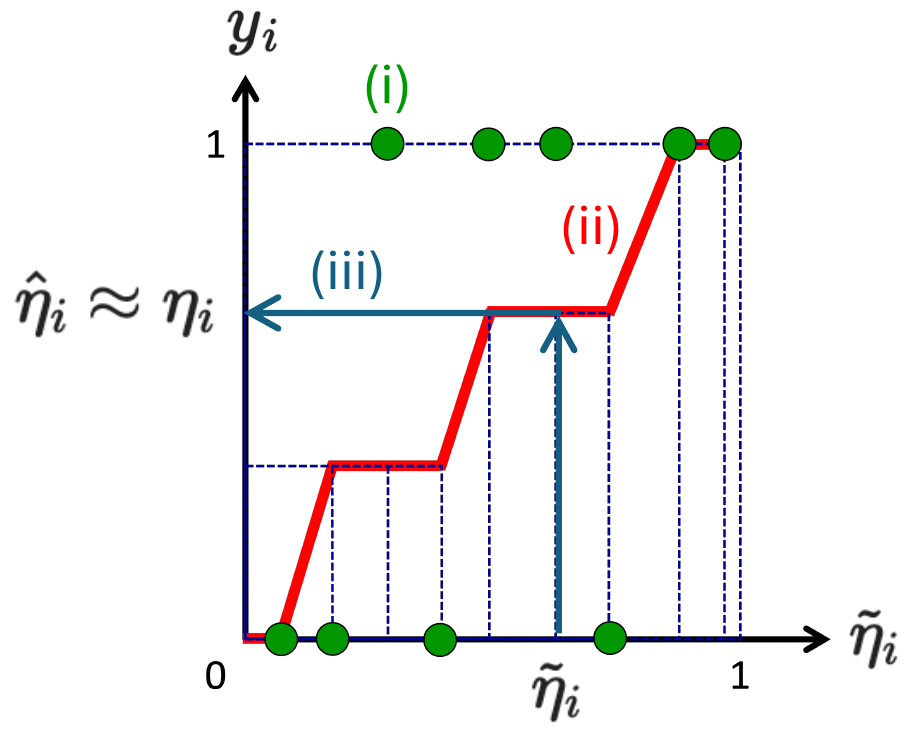}
  \caption{
    Illustration of the use of isotonic regression
    to recover clean soft labels
    from pairs 
    $(\xi_{1}, y_{1}), \dots, (\xi_{n}, y_{n})$
    of corrupted soft labels
    and class labels.
    (i) Plot the points
    $(\xi_{i}, y_{i})$.
    (ii) Find the non-decreasing function
    that best fits the points.
    (iii) Use the fitted function values
    as estimates of the clean soft labels.
    Roughly speaking, 
    the fitted function is expected to 
    approximate the inverse $f^{-1}$ 
    of the unknown skew function $f$.
  }
  \label{fig:corrupted}
\end{wrapfigure}

To approximately recover the clean soft labels $\eta_{1}, \dots, \eta_{n}$,
we adopt an approach based on 
isotonic regression,
which was originally proposed by
\citet{ushio2025practical}
in the context of Bayes error estimation.
Isotonic regression is one of the most widely used methods 
for calibrating classifier outputs 
to produce reliable probability estimates \citep{zadrozny2002transforming},
but here we use it for calibrating the corrupted soft labels.
Then the calibrated soft labels can be used as estimates of the clean soft labels.

A technical description of the algorithm is as follows.
We first
    sort the soft labels $\xi_{1}, \dots, \xi_{n}$
    in ascending order,
    and
    let $(j)$ denote the index of the $j$-th smallest
    of them,
    i.e., $\xi_{(1)} \leq \dots \leq \xi_{(n)}$.
    Then we use the pool adjacent violator algorithm
    \citep[PAVA;][]{ayer1955empirical}
    to find the monotonic sequence
    $\widehat{\eta}_{(1)} \leq \dots \leq \widehat{\eta}_{(n)}$
    that best fits the corresponding class labels
    $y_{(1)}, \dots, y_{(n)}$
    in the least squares sense:
    $
      \min_{
        \widehat{\eta}_{(1)} 
        \leq \dots \leq 
        \widehat{\eta}_{(n)}
      }
      \frac{1}{n}
      \sum_{i=1}^{n}
      \paren{
        y_{(i)} - \widehat{\eta}_{(i)}
      }^{2}
    $.
Each resulting $\widehat{\eta}_{i}$
is expected to approximate 
the true clean soft label $\eta_{i}$.
See also \cref{fig:corrupted}.
Assuming standard sorting algorithms
with worst-case time complexity of $O(n \log n)$,
this algorithm runs in
$O(n \log n)$ time in total,
since
the PAVA part takes only $O(n)$ time
\citep{best1990active,busing2022monotone}.

\paragraph{Estimation of class prior}

As 
the BER and AUC estimators
introduced in
\cref{sec:clean}
involve
the unknown class prior $\theta$,
it needs to be estimated from data.
Note that
\citet{ushio2025practical}
did not discuss this problem
since their Bayes error estimator did not
depend on
the class prior.

Since $\theta = \E{y} = \E{\eta(x)}$,
two natural candidate estimators
of $\theta$
are
$\frac{1}{n} \sum_{i} y_{i}$
and
$\frac{1}{n} \sum_{i} \widehat{\eta}_{i}$.
Interestingly,
however,
it turns out that
they are actually no different from each other.
Indeed,
we have $\widehat{\bm{\eta}} + t \bm{1} \in \mathcal{M}_{n}$
for any $t \in \R$.\footnote{
  $\bm{1}$ denotes the all-one vector.
}
So by the optimality of
$\widehat{\bm{\eta}}$,
$
  t \mapsto
  \frac{1}{2}
  \enorm{
    \bm{y} - (\widehat{\bm{\eta}} + t \bm{1})
  }^{2}
$
achieves its minimum at $t = 0$.
Therefore, we have
$
  \frac{1}{2}
  \left.
  \frac{d}{dt}
  \enorm{
    \bm{y} - (\widehat{\bm{\eta}} + t \bm{1})
  }^{2}
  \right|_{t=0}
  =
  \widehat{\bm{\eta}}^\top \bm{1}
  -
  \bm{y}^\top \bm{1}
  =
  0
$,
which implies 
$\frac{1}{n} \sum_{i} y_{i} = \frac{1}{n} \sum_{i} \widehat{\eta}_{i}$.
Similarly to \cref{sec:unknown-class-prior},
we use a clipped version of this estimator:
$
  \thetatilde 
  \coloneqq 
  \clip_{\tau_{n}}^{1 - \tau_{n}}(\frac{1}{n} \sum_{i} y_{i})
  =
  \clip_{\tau_{n}}^{1 - \tau_{n}}(\frac{1}{n} \sum_{i} \widehat{\eta}_{i})
$,
where the clipping threshold $\tau_{n}$
is set according to \cref{ass:unknown-class-prior:epsilon}.
This estimator $\thetatilde$ 
converges to $\theta$
at rate $\frac{1}{\sqrt{n}}$
(see \cref{lem:corrupted:class-prior:bernstein}).

\paragraph{Estimation error bounds}

Now we can estimate $\bestber$ and $\bestauc$
by plugging in $\widehat{\bm{\eta}}$ and $\thetatilde$
in place of $\bm{\eta}$ and $\theta$,
respectively,
in the definitions of $\widehat{\bestber}$ and $\widetilde{\bestauc}$.
Let us denote these plug-in estimators by
$
  \widehat{\bestber}(
    \widehat{\bm{\eta}}, 
    \thetatilde
  )
$
and
$
  \widetilde{\bestauc}(
    \widehat{\bm{\eta}}, 
    \thetatilde
  )
$.
They enjoy the following finite-sample error bounds,
where
$\lesssim$
represents 
an inequality
up to some constant
not depending on $n$, $\theta$ or $\delta$.

\begin{theorem}[store=thm:corrupted:noiseless]
  \label{thm:corrupted:noiseless}
  Under
  \cref{ass:unknown-class-prior:epsilon},
  for any
  $\delta \in (0, 1)$,
  each of the following holds with probability at least $1 - \delta$:
  \IfRestatingTF{
    \begin{equation}
      \nonumber
      \abs{
        \widehat{\bestber}(
          \widehat{\bm{\eta}},
          \thetatilde
        )
        -
        \bestber
      }
      \lesssim
      \frac{1}{\min \set{\theta, 1 - \theta}}
      \paren{
        \frac{1}{n^{1 / 3}}
        +
        \sqrt{
          \frac{\log (5 / \delta)}{n}
        }
      }
    \end{equation}
    and
    \begin{equation}
      \nonumber
      \abs{
        \widetilde{\bestauc}(
          \widehat{\bm{\eta}},
          \thetatilde
        )
        -
        \bestauc
      }
      \lesssim
      \frac{1}{\theta (1 - \theta)}
      \paren{
        \frac{1}{n^{1 / 3}}
        +
        \sqrt{
          \frac{\log (5 / \delta)}{n}
        }
      }
      \text{.}
    \end{equation}
  }{
    $
      \abs{
        \widehat{\bestber}(
          \widehat{\bm{\eta}},
          \thetatilde
        )
        -
        \bestber
      }
      \lesssim
      \frac{1}{\min \set{\theta, 1 - \theta}}
      \paren{
        \frac{1}{n^{1 / 3}}
        +
        \sqrt{
          \frac{\log (5 / \delta)}{n}
        }
      }
    $
    and
    $
      \abs{
        \widetilde{\bestauc}(
          \widehat{\bm{\eta}},
          \thetatilde
        )
        -
        \bestauc
      }
      \lesssim
      \frac{1}{\theta (1 - \theta)}
      \paren{
        \frac{1}{n^{1 / 3}}
        +
        \sqrt{
          \frac{\log (5 / \delta)}{n}
        }
      }
    $.
  }
\end{theorem}

\paragraph{Noisy corruption model}

\cref{thm:corrupted:noiseless}
can be extended to the case where
the corruption involves additive noise, i.e.,
$\xi_i = f(\eta_i) + \epsilon_i$.
Formally, we consider the following corruption model.

\begin{assumption}[Noisy corruption model]
  \label{ass:corrupted:noisy}
  Corrupted soft labels $\xi_{i} \in [0, 1]$
  are generated as
  $\xi_{i} = f(\eta_{i}) + \epsilon_{i}$.
  Here, the skew function $f: [0, 1] \to [0, 1]$ is differentiable with
  $f' \geq \gamma$ for some constant $\gamma > 0$
  and satisfies $f(0) = 0$ and $f(1) = 1$.
  Also, given $\eta_{1}, \dots, \eta_{n}$,
  $\epsilon_{i}$ is a zero-mean random variable
  with variance at most $\sigma^2$.
  We further assume that
  $\epsilon_{1}, \dots, \epsilon_{n}, y_{1}, \dots, y_{n}$ are
  conditionally independent given $\eta_{1}, \dots, \eta_{n}$.
\end{assumption}

The following counterpart of
\cref{thm:corrupted:noiseless}
holds in this noisy setting.
This result is similar
to Theorem~3 of \citet{ushio2025practical}.

\begin{theorem}
  \label{thm:corrupted:noisy}
  Under
  \cref{ass:unknown-class-prior:epsilon}
  and
  \cref{ass:corrupted:noisy},
  for any
  $\delta \in (0, 1)$,
  each of the following holds
  with probability at least $1 - \delta$:
  \begin{equation}
    \begin{aligned}
      \abs{
        \widehat{\bestber}(
          \widehat{\bm{\eta}},
          \thetatilde
        )
        -
        \bestber
      }
      &\lesssim
      \frac{1}{\min \set{\theta, 1 - \theta}}
      \paren{
        \frac{\sigma}{\gamma}
        +
        \frac{1}{n^{1 / 3}}
        +
        \sqrt{
          \frac{\log (5 / \delta)}{n}
        }
      }
      \text{,}
      \\
      \abs{
        \widetilde{\bestauc}(
          \widehat{\bm{\eta}},
          \thetatilde
        )
        -
        \bestauc
      }
      &\lesssim
      \frac{1}{\theta (1 - \theta)}
      \paren{
        \frac{\sigma}{\gamma}
        +
        \frac{1}{n^{1 / 3}}
        +
        \sqrt{
          \frac{\log (5 / \delta)}{n}
        }
      }
      \text{.}
    \end{aligned}
  \end{equation}

\end{theorem}

A simple example of this model is when
the corrupted soft labels are obtained by averaging multiple hard labels
drawn from a corrupted Bernoulli distribution.
That is, for each $i = 1, \dots, n$,
we independently draw $m$ hard labels
$y_i^{(1)}, \dots, y_i^{(m)}$
from a Bernoulli distribution whose mean is $f(\eta_i)$
rather than $\eta_i$.
We then obtain the noisy corrupted soft label
$
  \xi_i
  =
  \frac{1}{m}
  \sum_{j = 1}^{m} y_i^{(j)}
$
by averaging these hard labels.
In this case,
$\xi_i = f(\eta_i) + \epsilon_i$,
where $\epsilon_i$ has mean zero and standard deviation at most
$\sigma = 1 / (2 \sqrt{m})$.
Therefore, the preceding bounds give estimation errors of order
$m^{-1/2} + n^{-1/3}$.
The additional $m^{-1/2}$ term diminishes as
the number $m$ of hard labels increases.

\section{Real-world evaluation of optimal performance estimators}
\label{sec:evaluation}

So far, we have discussed
estimators of the optimal BER and AUC values
and their theoretical guarantees.
However, how can we empirically evaluate 
the performance of these estimators
\emph{on real-world datasets}?
This is a non-trivial question because 
the true values of $\bestber$ and $\bestauc$
are unobservable for real-world data distributions.
In this section, 
we first briefly review \emph{FeeBee} \citep{renggli2021evaluating},
an existing evaluation framework
for estimators of 
the Bayes error
$\besterr$.
Then, we draw inspiration from FeeBee to propose
a method to evaluate estimators of the optimal BER and AUC.

\paragraph{Review of FeeBee}

The simplest way to evaluate
Bayes error estimators on a real-world dataset
would be to 
compare the estimates with the error rate $u$
of the best available classifier on the dataset.
Since we have a trivial bound
$0 \leq \besterr \leq u$
on the Bayes error,
we can at least rule out estimators 
returning values greater than $u$.
However, such a naive method is not very helpful
because
we cannot tell which of two estimators
is better 
if both of them return estimates
within the range $[0, u]$.
Moreover, 
this approach
cannot rule out
an obviously useless estimator
that always returns some constant
between $0$ and $u$.

The key idea of FeeBee is
to inject label noise of a controlled level
$\nu \in [0, 1)$
into the dataset.
The injected noise increases 
the Bayes error $\besterr_{\nu}$
for the noise-injected distribution,
which gives us another set of bounds
$L(\nu) \leq \besterr_{\nu} \leq U(\nu)$.
We can then 
estimate 
$\besterr_{\nu}$
for the noise-injected dataset
and test whether the estimate
$\widehat{\besterr_{\nu}}$
satisfies the bounds.
Repeating this process for many different 
noise levels $\nu$
and aggregating the results,
we can obtain a single score for the estimator.

\paragraph{Bounds on the optimal BER and AUC for noise-injected distributions}

In our proposed evaluation framework for the optimal BER and AUC,
we inject label noise into the original distribution
as follows:
for a fixed noise level $\nu \in [0, 1)$,
with probability $\nu$,
each original label $y$ is replaced
with an independent random label $w \in \set{0, 1}$ 
with a fixed mean $\beta \in (0, 1)$.
The resulting noisy label is given by
$y_{\nu} \coloneqq (1 - z) \cdot y + z \cdot w$,
where $z \in \set{0, 1}$ is an independent random variable with mean $\nu$.
The class posterior for the noisy distribution is then
related to the original posterior $\eta(x)$ by
$\mathbb{P}(y_{\nu} = 1 \mid x) = \lambda_{\nu}(\eta(x))$,
where
$\lambda_{\nu}(t) \coloneqq (1 - \nu) t + \nu \beta$.
Our noise model coincides with 
that of the original FeeBee
when $\beta = \frac{1}{2}$,
but we allow $\beta$ to be chosen arbitrarily in $(0, 1)$.

\begin{theorem}[store=thm:corrupted:noiseless]
  \label{thm:corrupted:noiseless}
  Under
  \cref{ass:unknown-class-prior:epsilon},
  for any
  $\delta \in (0, 1)$,
  each of the following holds with probability at least $1 - \delta$:
  \IfRestatingTF{
    \begin{equation}
      \nonumber
      \abs{
        \widehat{\bestber}(
          \widehat{\bm{\eta}},
          \thetatilde
        )
        -
        \bestber
      }
      \lesssim
      \frac{1}{\min \set{\theta, 1 - \theta}}
      \paren{
        \frac{1}{n^{1 / 3}}
        +
        \sqrt{
          \frac{\log (5 / \delta)}{n}
        }
      }
    \end{equation}
    and
    \begin{equation}
      \nonumber
      \abs{
        \widetilde{\bestauc}(
          \widehat{\bm{\eta}},
          \thetatilde
        )
        -
        \bestauc
      }
      \lesssim
      \frac{1}{\theta (1 - \theta)}
      \paren{
        \frac{1}{n^{1 / 3}}
        +
        \sqrt{
          \frac{\log (5 / \delta)}{n}
        }
      }
      \text{.}
    \end{equation}
  }{
    $
      \abs{
        \widehat{\bestber}(
          \widehat{\bm{\eta}},
          \thetatilde
        )
        -
        \bestber
      }
      \lesssim
      \frac{1}{\min \set{\theta, 1 - \theta}}
      \paren{
        \frac{1}{n^{1 / 3}}
        +
        \sqrt{
          \frac{\log (5 / \delta)}{n}
        }
      }
    $
    and
    $
      \abs{
        \widetilde{\bestauc}(
          \widehat{\bm{\eta}},
          \thetatilde
        )
        -
        \bestauc
      }
      \lesssim
      \frac{1}{\theta (1 - \theta)}
      \paren{
        \frac{1}{n^{1 / 3}}
        +
        \sqrt{
          \frac{\log (5 / \delta)}{n}
        }
      }
    $.
  }
\end{theorem}

\paragraph{Noisy corruption model}

\cref{thm:corrupted:noiseless}
can be extended to the case where
the corruption involves additive noise, i.e.,
$\xi_i = f(\eta_i) + \epsilon_i$.
Formally, we consider the following corruption model.

\begin{assumption}[Noisy corruption model]
  \label{ass:corrupted:noisy}
  Corrupted soft labels $\xi_{i} \in [0, 1]$
  are generated as
  $\xi_{i} = f(\eta_{i}) + \epsilon_{i}$.
  Here, the skew function $f: [0, 1] \to [0, 1]$ is differentiable with
  $f' \geq \gamma$ for some constant $\gamma > 0$
  and satisfies $f(0) = 0$ and $f(1) = 1$.
  Also, given $\eta_{1}, \dots, \eta_{n}$,
  $\epsilon_{i}$ is a zero-mean random variable
  with variance at most $\sigma^2$.
  We further assume that
  $\epsilon_{1}, \dots, \epsilon_{n}, y_{1}, \dots, y_{n}$ are
  conditionally independent given $\eta_{1}, \dots, \eta_{n}$.
\end{assumption}

The following counterpart of
\cref{thm:corrupted:noiseless}
holds in this noisy setting.
This result is similar
to Theorem~3 of \citet{ushio2025practical}.

\begin{theorem}
  \label{thm:corrupted:noisy}
  Under
  \cref{ass:unknown-class-prior:epsilon}
  and
  \cref{ass:corrupted:noisy},
  for any
  $\delta \in (0, 1)$,
  each of the following holds
  with probability at least $1 - \delta$:
  \begin{equation}
    \begin{aligned}
      \abs{
        \widehat{\bestber}(
          \widehat{\bm{\eta}},
          \thetatilde
        )
        -
        \bestber
      }
      &\lesssim
      \frac{1}{\min \set{\theta, 1 - \theta}}
      \paren{
        \frac{\sigma}{\gamma}
        +
        \frac{1}{n^{1 / 3}}
        +
        \sqrt{
          \frac{\log (5 / \delta)}{n}
        }
      }
      \text{,}
      \\
      \abs{
        \widetilde{\bestauc}(
          \widehat{\bm{\eta}},
          \thetatilde
        )
        -
        \bestauc
      }
      &\lesssim
      \frac{1}{\theta (1 - \theta)}
      \paren{
        \frac{\sigma}{\gamma}
        +
        \frac{1}{n^{1 / 3}}
        +
        \sqrt{
          \frac{\log (5 / \delta)}{n}
        }
      }
      \text{.}
    \end{aligned}
  \end{equation}

\end{theorem}

A simple example of this model is when
the corrupted soft labels are obtained by averaging multiple hard labels
drawn from a corrupted Bernoulli distribution.
That is, for each $i = 1, \dots, n$,
we independently draw $m$ hard labels
$y_i^{(1)}, \dots, y_i^{(m)}$
from a Bernoulli distribution whose mean is $f(\eta_i)$
rather than $\eta_i$.
We then obtain the noisy corrupted soft label
$
  \xi_i
  =
  \frac{1}{m}
  \sum_{j = 1}^{m} y_i^{(j)}
$
by averaging these hard labels.
In this case,
$\xi_i = f(\eta_i) + \epsilon_i$,
where $\epsilon_i$ has mean zero and standard deviation at most
$\sigma = 1 / (2 \sqrt{m})$.
Therefore, the preceding bounds give estimation errors of order
$m^{-1/2} + n^{-1/3}$.
The additional $m^{-1/2}$ term diminishes as
the number $m$ of hard labels increases.

\cref{thm:feebee}
expresses
the optimal BER value $\bestber_{\nu}$ for the noisy distribution 
as an increasing function of $\bestber$,
and thus induces bounds on $\bestber_{\nu}$ as follows.
Let $u_\ber$ be any upper bound on 
the optimal BER value $\bestber$ for the original distribution
(e.g., the lowest BER achieved by existing classifiers).
Then, for each noise level $\nu \in [0, 1)$,
$\bestber_{\nu}$ must fall between
$L_{\ber}(\nu) \coloneqq F_{\nu}(0)$
and
$U_{\ber}(\nu) \coloneqq F_{\nu}(u_\ber)$.
Similarly, 
the optimal AUC value $\bestauc_{\nu}$ for the noisy distribution
must satisfy
$L_{\auc}(\nu) \coloneqq 1 - F_{\nu}(1 - l_{\auc})
\leq
\bestauc_{\nu}
\leq
U_{\auc}(\nu) \coloneqq 1 - F_{\nu}(0)$,
where 
$l_{\auc}$ is any lower bound on 
the optimal AUC value $\bestauc$ for the original distribution.

\paragraph{Computing scores for BER and AUC}

Based on these bounds, 
we propose to evaluate estimators of the optimal BER
on real-world datasets as follows.
We first generate $N$ noise-injected datasets
with different noise levels 
$\nu_i \coloneqq \frac{i - 1}{N} \ (i = 1, \dots, N)$
and evaluate a given estimator on each of them
to obtain estimates $\widehat{\bestber_{\nu_i}}$.
Then, 
the estimator's score $s_{\beta}$ 
for noise mean $\beta$
is computed as
$
  s_{\beta} \coloneqq
  \frac{1}{N} \sum_{i=1}^{N}
  \left[
  (
  \widehat{\bestber_{\nu_i}} - U_{\ber}(\nu_i)
  )_+
  +
  (
  L_{\ber}(\nu_i) - \widehat{\bestber_{\nu_i}}
  )_+
  \right]
$,
where $(x)_+ \coloneqq \max\{x, 0\}$.
The idea is that
if the estimates fall outside their corresponding bounds
$[L_{\ber}(\nu_i), U_{\ber}(\nu_i)]$,
we penalize the estimator
by the amount of violation;
then the score is obtained by aggregating 
the penalties over all noise levels.
The lower the score is, the better the estimator is.
Note that the score $s_{\beta}$ depends on the choice of $\beta$.
Thus, how to choose $\beta$ is an interesting question;
possible choices include $\beta = \frac{1}{2}$ (the same as the original FeeBee) 
and $\beta = \theta$, in which case $F_{\nu}$ 
reduces simply to
$F_{\nu}(t) = (1 - \nu) t + \frac{\nu}{2}$.
Another option is averaging $s_{\beta}$ over many different $\beta$ values.
In the experiments in \cref{sec:experiments:max-min},
we will see that the last option
aligns best with variance estimates,
which suggests that it is a reasonable choice.\footnote{
  This might be analogous to AUC,
  which aggregates performance over different classification thresholds.
}

The score for AUC can be computed in the same way by replacing $L_{\ber}$ and $U_{\ber}$ with $L_{\auc}$ and $U_{\auc}$, respectively.
This provides a practical way to evaluate estimators
of the optimal BER and AUC
on real-world datasets without requiring knowledge of the true optima.
Note that the class prior $\theta$
has to be estimated from the data in order to compute the bounds.
In many scenarios where the dataset contains
hard labels $y_{i} \in \set{0, 1}$,
this can be done by simply using their sample mean.

\subsection{Bias sensitivity}
\label{sec:evaluation:bias}

Here we show that our score can reflect certain types of estimator bias,
unlike bootstrap variance.

\paragraph{Theoretical analysis}

The following proposition roughly states that
any bias larger than the interval width
$U_{\ber}(\nu) - L_{\ber}(\nu)$ produces a positive expected
penalty.
This is in contrast to the bootstrap variance,
which does not change if a constant is added to the estimator.
Here we discuss the BER case, but the same argument also applies to AUC
by replacing $u_{\ber}$ with $1 - l_{\auc}$.

\begin{proposition}
  \label{prop:evaluation:bias}
  Let
  $\mathrm{bias}_{\nu} =
    \E{\widehat{\bestber_{\nu}}} - \bestber_{\nu}$
  denote the bias of the estimator
  $\widehat{\bestber_{\nu}}$ at noise level $\nu$.
  Then
  \begin{equation}
    \E{
      \paren{\widehat{\bestber_{\nu}} - U_{\ber}(\nu)}_+
      +
      \paren{L_{\ber}(\nu) - \widehat{\bestber_{\nu}}}_+
    }
    \geq
    \left[
      \abs{\mathrm{bias}_{\nu}}
      -
      \frac{u_{\ber}}{K} (1 - \nu)
    \right]_+
    \text{,}
  \end{equation}
  where
  $
    K
    \coloneqq
    \min \set{
    1,
    \frac{\beta (1 - \beta)}{\theta (1 - \theta)}
    }
    \in (0, 1]
  $.
\end{proposition}

\begin{proof}
  Let $d(z_1, z_2) = |z_1 - z_2|$
  and $I_{\nu} = [L_{\ber}(\nu), U_{\ber}(\nu)]$.
  Then, the pointwise penalty
  $
    (
    \widehat{\bestber_{\nu}} - U_{\ber}(\nu)
    )_+
    +
    (
    L_{\ber}(\nu) - \widehat{\bestber_{\nu}}
    )_+
  $
  at noise level $\nu$ is precisely
  $
    d(\widehat{\bestber_{\nu}}, I_{\nu}) \coloneqq \inf_{a \in I_{\nu}} d(\widehat{\bestber_{\nu}}, a)
  $.
  By the convexity of $d(\cdot, I_{\nu})$ and Jensen's inequality,
  the expected pointwise penalty is bounded as
  \begin{equation}
    \label{eq:evaluation:bias:jensen}
    \E{d \paren{\widehat{\bestber_\nu}, I_{\nu}}}
    \geq
    d \paren{\E{\widehat{\bestber_\nu}}, I_{\nu}}
    \text{.}
  \end{equation}
  On the other hand,
  since $\bestber_{\nu} \in I_{\nu}$,
  the bias of the estimator $\widehat{\bestber_{\nu}}$ is bounded as follows:\footnote{
    This follows from the following fact:
    let $(\mathcal{Z},d)$ be a metric space
    and let $A \subset \mathcal{Z}$.
    For any $z \in \mathcal{Z}$ and $a \in A$, we have
    $d(z, a) \leq d(z, A) + \mathrm{diam}(A)$, where
    $d(z, A) = \inf_{a' \in A} d(z, a')$ and
    $\mathrm{diam}(A) = \sup_{a_1, a_2 \in A} d(a_1, a_2)$.
  }
  \begin{equation}
    \label{eq:evaluation:bias:bias-bound}
    \abs{\mathrm{bias}_{\nu}}
    =
    d\paren{\E{\widehat{\bestber_\nu}}, \bestber_{\nu}}
    \leq
    d\paren{\E{\widehat{\bestber_\nu}}, I_{\nu}}
    +
    U_{\ber}(\nu) - L_{\ber}(\nu)
    \text{.}
  \end{equation}
  The interval width $U_{\ber}(\nu) - L_{\ber}(\nu)$ is bounded as
  \begin{equation}
    \begin{aligned}[b]
      U_{\ber}(\nu) - L_{\ber}(\nu)
      &=
      F_{\nu}(u_{\ber}) - F_{\nu}(0)
      \\&=
      \frac{
        (1 - \nu)
        \theta (1 - \theta) u_{\ber}
      }{
        \lambda_{\nu}(\theta) (1 - \lambda_{\nu}(\theta))
      }
      \\&\leq
      \frac{
        (1 - \nu)
        \theta (1 - \theta) u_{\ber}
      }{
        \min \set{
          \theta (1 - \theta),
          \beta (1 - \beta)
        }
      }
      \\&=
      \label{eq:evaluation:bias:interval-width-bound}
      \frac{ u_{\ber} }{ K }
      (1 - \nu)
      \text{.}
    \end{aligned}
  \end{equation}
  Combining
  \eqref{eq:evaluation:bias:jensen},
  \eqref{eq:evaluation:bias:bias-bound}, and
  \eqref{eq:evaluation:bias:interval-width-bound},
  we conclude that
  \begin{equation}
    \E{d \paren{\widehat{\bestber_{\nu}}, I_{\nu}}}
    \geq
    \left[
      \abs{\mathrm{bias}_{\nu}}
      -
      \frac{u_{\ber}}{K} (1 - \nu)
    \right]_+
    \text{.}
  \end{equation}
\end{proof}

Thus, for each noise level $\nu$, 
a bias larger than $\frac{u_{\ber}}{K} (1 - \nu)$
produces a positive expected penalty.
Moreover, since $\frac{u_{\ber}}{K} (1 - \nu) \to 0$ as $\nu \to 1$,
any fixed nonzero additive bias that persists across noise levels
is expected to be detected at a sufficiently large noise level.
It is worth noting that this is not a universal guarantee:
a bias that shrinks sufficiently fast as $\nu \to 1$
may remain undetected.

\cref{prop:evaluation:bias} quantifies 
the \emph{pointwise} penalty for each noise level,
and we can use it to discuss the overall score aggregated over different noise levels.
As a simple example, consider an idealized version of $s_{\beta}$,
obtained by replacing the finite average over $N$ noise levels
with an integral (informally, $N \to \infty$):
\begin{equation}
  s_{\beta}
  =
  \int_0^1 d \paren{\widehat{\bestber_{\nu}}, I_{\nu}} \, d\nu
  =
  \int_{0}^{1}
  \left[
    \paren{\widehat{\bestber_{\nu}} - U_{\ber}(\nu)}_+
    +
    \paren{L_{\ber}(\nu) - \widehat{\bestber_{\nu}}}_+
  \right]
  \, d\nu
  \text{.}
\end{equation}
The result above implies that the expected score is bounded as
\begin{equation}
  \E{s_{\beta}}
  \geq
  \int_0^1
  \left[
    \abs{\mathrm{bias}_{\nu}}
    -
    \frac{u_{\ber}}{K} (1 - \nu)
  \right]_+
  \, d\nu
  \text{.}
\end{equation}
For example, 
if $\abs{\mathrm{bias}_{\nu}} \geq B \geq 0$ for all $\nu \in [0, 1)$,
the expected score is lower-bounded
by an increasing function of $B$ as
\begin{equation}
  \label{eq:evaluation:bias:score-lower-bound}
  \E{s_{\beta}}
  \geq
  \begin{cases}
    \frac{K}{2 u_{\ber}} B^2
    & \text{if } B \leq \frac{u_{\ber}}{K} \text{,}
    \\
    B - \frac{u_{\ber}}{2K}
    & \text{otherwise.}
  \end{cases}
\end{equation}

\paragraph{Empirical analysis}

We also present an experiment using synthetic data with
controlled bias.
We drew $n=10{,}000$ points from a one-dimensional
two-component Gaussian mixture with means $-0.5$ and $0.5$,
unit variances, and class prior $\theta=0.5$.
For each point, we computed the clean soft label $\eta_i$
using the closed form obtained from the Gaussian PDF.
By \cref{prop:gaussian}, the true optimal values are
$\bestber = \Phi(-1/2) \approx 0.308538$ and
$\bestauc = \Phi(1/\sqrt{2}) \approx 0.760250$.

We repeated the experiment 100 times.
In each trial, we computed the min formula estimators
at $N=100$ noise levels.
To change the bias without changing the variance,
we defined shifted estimators at each noise level $\nu$ as
$\widehat{\bestber_{\nu}}^{(b)}
\coloneqq \widehat{\bestber_{\nu}}^{(0)} - b$
and
$\widehat{\bestauc_{\nu}}^{(b)}
\coloneqq \widehat{\bestauc_{\nu}}^{(0)} + b$.
Because the same constant is added to or subtracted
from the estimate in every trial,
these changes do not change the variance of either estimator.
We used
$b \in \{0, 0.01, 0.02, 0.05, 0.10, 0.15\}$
and the same samples for every value of $b$.
The parameters $u_{\ber}$ and $l_{\auc}$ were both set to $0.5$.
We computed the scores $s_\beta$
for $\beta = 0.5$
and the average scores over 
$\beta \in \{0.1, \dots, 0.9\}$.

\cref{fig:evaluation:bias} reports averages over the 100 trials
(blue and orange solid curves)
along with the theoretical lower bounds in
\eqref{eq:evaluation:bias:score-lower-bound}
(blue and orange dotted curves).
All standard errors of the scores were below $2.8\times10^{-5}$.
The standard deviations (SDs) of the BER estimates and
AUC estimates were 0.001241 and 0.001370, respectively,
for every value of $b$
(as expected, since the variance was kept constant).
In contrast, the score averaged over $\beta$ increased monotonically with $b$,
from 0 to 0.029953 for BER and from 0 to 0.040368 for AUC.
These results show that, unlike the standard deviations,
our evaluation score successfully reflects
the increase in bias even when
the estimator variance is unchanged.

\begin{figure}[t]
  \centering
  \begin{subfigure}[t]{0.49\textwidth}
    \centering
    \includegraphics[width=\linewidth]{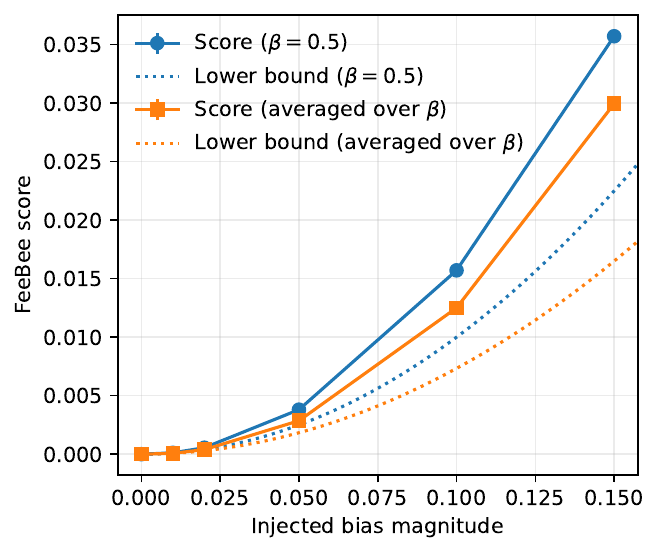}
    \caption{BER}
    \label{fig:evaluation:bias:ber}
  \end{subfigure}
  \hfill
  \begin{subfigure}[t]{0.49\textwidth}
    \centering
    \includegraphics[width=\linewidth]{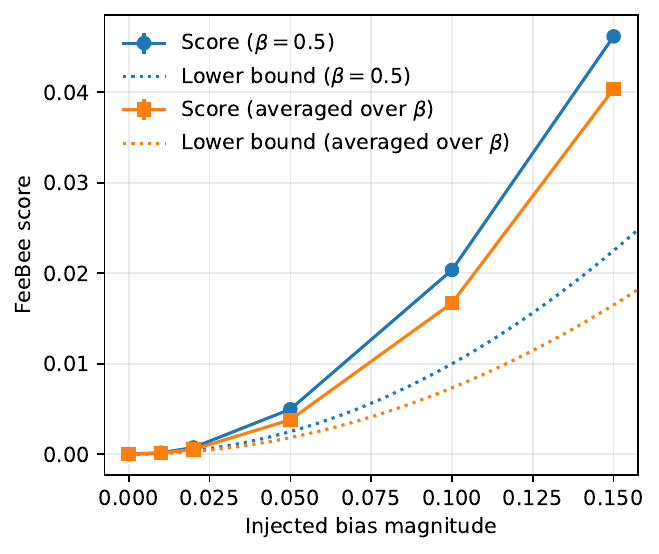}
    \caption{AUC}
    \label{fig:evaluation:bias:auc}
  \end{subfigure}
  \caption{
    Our evaluation scores versus injected bias magnitude
    for BER and AUC.
    Points show averages over 100 trials.
    Their 95\% Monte Carlo confidence intervals
    are smaller than the markers.
    The blue curves use $\beta=0.5$,
    whereas the orange curves average the scores
    over $\beta\in\{0.1,\dots,0.9\}$.
    The blue and orange dotted curves show the theoretical lower bounds in
    \eqref{eq:evaluation:bias:score-lower-bound} 
    for the corresponding $\beta$ settings,
    evaluated at $B = b$.
    Both scores increase monotonically with the injected bias
    (although the estimator standard deviations remain unchanged).
  }
  \label{fig:evaluation:bias}
\end{figure}

\section{Experiments}
\label{sec:experiments}

Here, we present experimental results
to support our theoretical results.
Further experimental details and additional results
can be found in \cref{sec:app:experiments}.

\subsection{Estimation on synthetic data}
\label{sec:app:experiments:estimation}

This appendix provides the full setup
for the synthetic-data experiment of \cref{sec:experiments:estimation}.

\paragraph{Data generation}
We follow the synthetic-data protocol of \citet{ushio2025practical}.
Inputs $x_{i} \in \mathbb{R}^{2}$ are drawn from a two-component Gaussian mixture
with means $(0, 0)$ and $(2, 2)$, identity covariances,
and mixture weights $(0.8, 0.2)$, 
so that the class prior is $\theta = 0.2$.
For each $x_{i}$, the clean soft label
$\eta_{i} = \P{y = 1 \mid x = x_{i}}$
is computed in closed form from the mixture density,
and a binary label $y_{i}$ 
is sampled from $\mathrm{Bernoulli}(\eta_{i})$.
We use $n = 10{,}000$ samples.
In this model, the optimal values can be computed analytically
using \cref{prop:gaussian}.
The Mahalanobis distance $J$ is
$\sqrt{(2,2)^{\top} I_{2}^{-1} (2,2)} = 2 \sqrt{2}$,
where $I_{2}$ is the $2 \times 2$ identity matrix.
Thus,
$\bestber = \Phi(-\sqrt{2}) \approx 0.07865$
and
$\bestauc = \Phi(2) \approx 0.97725$.
The dashed lines in \cref{fig:estimation} mark these values as references.

The corruption map $f \colon [0, 1] \to [0, 1]$ 
used here is the inverse beta-calibration
\citep{kull2017beta} map 
also used in the experiments by \citet{ushio2025practical},
\begin{equation}
  \label{eq:inverse-beta}
  f(\eta; a, b)
  = \frac{1}{1 + \left( \tfrac{1 - \eta}{\eta} \right)^{1/a} \left( \tfrac{1 - b}{b} \right)},
\end{equation}
parameterised by $(a, b)$. 
The three regimes used in \cref{fig:estimation}
correspond to $(a, b) = (1, 0.5), (1.5, 0.5)$,
and $(0.5, 0.5)$, respectively.
Setting $a > 1$
produces under-confident soft labels
(pulled toward $0.5$), 
and $a < 1$ produces over-confident soft labels
(pushed toward $0$ or $1$).
The choice $b = 0.5$ keeps the corruption symmetric about $0.5$.

\paragraph{Compared estimators}

For each metric, we compare the following estimates.
For the clean setting,
we simply
apply the estimator
$\bestberhat$ or $\widetilde{\bestauc}$ 
directly to the soft labels $\eta_{i}$
as discussed in \cref{sec:clean} (\textbf{Clean}).
For the corrupted settings (under- or over-confidence), 
we evaluate two estimators:
\begin{itemize}
  \item \textbf{Corrupted, calibrated} (\cref{sec:corrupted}):
    fit isotonic regression
    \citep{ayer1955empirical}
    of $y_{i}$ on $\xi_{i}$ to obtain
    $\widehat{\eta}_{i}$,
    then apply $\bestberhat$ or $\bestauchat$ 
    to $\widehat{\eta}_{i}$.
  \item \textbf{Corrupted, raw} (naive baseline):
    apply $\bestberhat$ or $\bestauchat$
    directly to the corrupted soft labels $\xi_{i}$
    without recalibration.
\end{itemize}
The class prior $\theta$
is treated as unknown and
estimated from the given data,
namely,
by the sample mean of 
the soft labels $\eta_{i}$ in the clean setting
and that of the hard labels $y_{i}$ in the corrupted settings.
We use the rule of \cref{sec:min-vs-max}
to pick between the min and max formulae
on each resample.

\paragraph{Bootstrap confidence intervals}
We compute 95\% confidence intervals
using the bias-corrected accelerated (BCa) bootstrap
with $1{,}000$ resamples,
as implemented by \texttt{scipy.stats.bootstrap}
in SciPy 1.15.3 \citep{2020SciPy-NMeth}.

\subsection{Real-data experiment: Discriminants and evaluation scores}
\label{sec:app:experiments:max-min}

Here we describe the detailed setup
for the experiment on real-world datasets
presented in \cref{sec:experiments:max-min}.

\subsubsection{Datasets}
\label{sec:datasets}

In our experiments, we used the following datasets
from various domains, 
including computer vision, 
natural language processing and academic peer-review.

\paragraph{CIFAR-10}
We used 
the test set of CIFAR-10 \citep{cifar10} with soft labels
from CIFAR-10H \citep{cifar10h}.\footnote{
  We were unable to find an explicit license stated
  for CIFAR-10, 
  but it is widely used in the research community and 
  is generally considered to be in the public domain.
  CIFAR-10H is licensed under BY-NC-SA 4.0.
}
Following \citet{ishida2023is,ushio2025practical},
the original 10 classes
were reorganized into two classes
(animal vs. non-animal)
by treating 
\emph{bird}, \emph{cat}, \emph{deer}, \emph{dog}, \emph{frog} and \emph{horse}
as positive 
and the rest as negative.

\paragraph{Fashion-MNIST}
We also used 
the Fashion-MNIST dataset \citep{xiao2017online} 
and its soft-labeled counterpart, Fashion-MNIST-H \citep{ishida2023is}.\footnote{
  Fashion-MNIST is MIT-licensed.
  Fashion-MNIST-H is distributed as a part of their
  GPL-3.0 repository \url{https://github.com/ishida-lab/irreducible},
  but no separate data-specific license is stated.
}
As in CIFAR-10,
we followed \citet{ishida2023is,ushio2025practical}
and relabeled the original 10 classes into two classes:
positive 
(\emph{T-shirt/top}, \emph{pullover}, \emph{dress}, \emph{coat} and \emph{shirt})
and negative (the rest).

\paragraph{ChaosNLI}
We used a natural language processing dataset
called
ChaosNLI \citep{ynie2020chaosnli, xzhou2022distnli}.\footnote{
  Available at \url{https://github.com/easonnie/ChaosNLI}
  and distributed under the CC BY-NC 4.0 license.
}
It contains 100 hard labels per data point,
which can be used to construct soft labels
by taking their average.
It consists of three sub-datasets:
SNLI ($n = 1,514$), MNLI ($n = 1,599$) and AbductiveNLI ($n = 1,532$).

\paragraph{ICLR 2017-2026 peer-review datasets}
\citet{ushio2025practical} 
collected 
$n = 32,829$ instances 
of reviewer scores and confidences 
from the 2017-2025 ICLR conferences
using the OpenReview API.
Then they used them to
construct a dataset consisting of
pairs of a soft and hard label.
Specifically, for each submitted paper $x_i$,
they first calculated the averaged score 
weighted by the confidences,
and then obtained a corrupted soft label $\tilde{\eta}_i$ 
by normalizing the averaged score $s_i$ into $[0, 1]$.
The hard label is simply the final decision
for the paper, i.e., accept ($y_i = 1$) or reject ($y_i = 0$).
In this paper, we 
followed the same procedure to
collect the 2017-2026 ICLR peer-review data
(note that the 2026 data was not available at the time of \citet{ushio2025practical}).

\subsubsection{Experiment setup}
\label{sec:app:experiments:max-min:setup}

For each metric (BER, AUC) and each dataset of \cref{sec:datasets},
we evaluate three quantities:
the min formula
($\widehat{\bestber_{1}}$ or $\widehat{\bestauc_{1}}$),
the max formula
($\widehat{\bestber_{2}}$ or $\widehat{\bestauc_{2}}$),
and the empirical discriminant
($\widehat{\Delta}_{\ber}$ or $\widehat{\Delta}_{\auc}$).

\paragraph{Bootstrap point estimate, CI, and SE.}
For each of the three estimators we report
the point estimate,
the bootstrap standard error,
and a 95\% bias-corrected accelerated (BCa)
bootstrap confidence interval,
all computed with $1{,}000$ resamples
using \texttt{scipy.stats.bootstrap}
in SciPy 1.15.3 \citep{2020SciPy-NMeth}.

\paragraph{Evaluation score $s_{\beta}$.}
For the min and max formula estimators,
we additionally compute the score $s_{\beta}$ of \cref{sec:evaluation}.
$s_{\beta}$ is computed with $N = 100$ different noise levels
$\nu_{i} = (i - 1) / N \in [0, 1)$ for $i = 1, \dots, N$.
The upper bound $u_{\ber}$ on $\bestber$
(resp.\ the lower bound $l_{\auc}$ on $\bestauc$)
is set to its theoretical limit $0.5$ on every dataset.
We compare three choices for the noise mean $\beta$:
(i) $\beta = 0.5$, recovering the original FeeBee setting
of \citet{renggli2021evaluating};
(ii) $\beta = \theta$, the estimated class prior; and
(iii) the average of $s_{\beta}$ over the grid
$\beta \in \{0.1, \ldots, 0.9\}$.

\section{Conclusion}
\label{sec:conclusion}

We studied the estimation of the Bayes-optimal BER and AUC,
motivated by their practical importance.
We proposed soft-label-based estimators
which can be computed efficiently,
and analyzed their theoretical properties.
Our estimators handle not only clean soft labels but also corrupted soft labels,
which makes them more practical.
We also proposed a practical evaluation procedure
for assessing the quality of estimators of the optimal BER and AUC
without knowledge of the true optima.

Here, we mention some limitations of our work:
our corruption model (\cref{sec:corrupted}) 
is quite simple and may not capture the full complexity of
all possible data generation processes in practice.
For example,
it does not cover arbitrary instance-dependent corruption,
which would require additional structure
to recover the individual clean posterior values.
Future work includes extending our results to 
other metrics popular under class imbalance or label noise,
such as the F-measure
\citep{gu2009evaluation,parambath2014optimizing,natarajan2015optimal,narasimhan2015optimizing}.
Most of such metrics are non-decomposable,
i.e., their dataset-level values cannot be written as averages over individual data points,
making their Bayes-optimal values harder to analyze and estimate.
Another natural direction is multi-class BER:
especially,
the corrupted case requires an appropriate multi-class corruption model,
which is not trivial to design.

\bibliography{ref}
\bibliographystyle{icml2026}

\newpage
\appendix
\onecolumn

\crefalias{section}{appsec}
\crefalias{subsection}{appsec}
\crefalias{subsubsection}{appsec}

\section{Supplementary for \cref{sec:clean}}
\label{sec:proof-clean}

\subsection{Deriving the Optimal BER and AUC}
\label{sec:proof-clean:bestber-bestauc}

Here, we prove 
\cref{lem:bestber,lem:bestauc}
that gives the expressions for the optimal
BER and AUC.

\getkeytheorem{lem:bestber}

\begin{proof}
  For any classifier $h: \mathcal{X} \to \set{0, 1}$, 
  its false negative rate and false positive rate can be expressed as 
  $
  \E{\frac{\eta(x)}{\theta} \I{h(x) = 0}}
  $ and $
  \E{\frac{1 - \eta(x)}{1 - \theta} \I{h(x) = 1}}
  $, respectively, by Bayes' theorem. 
  Therefore, we have 
  \begin{equation}
  \ber(h)
  = 
  \frac{1}{2} 
  \E{\frac{\eta(x)}{\theta} \I{h(x) = 0}
  + \frac{1 - \eta(x)}{1 - \theta} \I{h(x) = 1}}
  \text{,}
  \end{equation}
  which is minimized by taking $h(x) = \I{\eta(x) \geq \theta}$.
  This gives the optimal BER as
  \begin{equation}
      \bestber
      = 
      \frac{1}{2} 
      \E{\min \set{\frac{\eta(x)}{\theta},  \frac{1 - \eta(x)}{1 - \theta}}}
      =
      \E{ \phimin_{\ber}(\eta(x)) }
    \text{.}
  \end{equation}
  Next, since $\E{\eta(x)} = \theta$, we have
  \begin{equation}
    \E{
      \min \set{  \frac{\eta(x)}{\theta},  \frac{1 - \eta(x)}{1 - \theta}}
      +
      \max \set{  \frac{\eta(x)}{\theta},  \frac{1 - \eta(x)}{1 - \theta}}
    }
    = \E{
      \frac{\eta(x)}{\theta} + \frac{1 - \eta(x)}{1 - \theta}
    } 
    = 2
    \text{.}
  \end{equation}
  This implies $\bestber = \E{ \phimax_{\ber}(\eta(x)) }$.
\end{proof}

\getkeytheorem{lem:bestauc}

\begin{proof}
  Fix any scoring function $f: \mathcal{X} \to \R$.
  One can replace the expectation with respect to 
  the class-conditional distributions with the expectation 
  with respect to the marginal distribution by Bayes' theorem:
  \begin{align}
  \auc(f)
  & = 
  \E[\substack{x_+ \sim \mathbb{P}_1 \\ x_- \sim \mathbb{P}_0}]{
  \I{f(x_+) > f(x_-)} + \frac{1}{2} \I{f(x_+) = f(x_-)}
  }
  \nonumber
  \\ & =
  \frac{1}{\theta (1 - \theta)}
  \E[x, x' \sim \mathbb{P}_\mathcal{X}]{
  \eta(x) (1 - \eta(x'))
  \paren{\I{f(x) > f(x')} + \frac{1}{2} \I{f(x) = f(x')}}
  }
  \text{.}
  \label{eq:auc-1}
  \end{align}
  Noting $x$ and $x'$ are independent and identically distributed, we may exchange them to get
  \begin{equation}
  \label{eq:auc-2}
  \auc(f)
  =
  \frac{1}{\theta (1 - \theta)}
  \E{
  \eta(x') (1 - \eta(x))
  \paren{\I{f(x) < f(x')} + \frac{1}{2} \I{f(x) = f(x')}}
  }
  \text{.}
  \end{equation}
  Adding \eqref{eq:auc-1} and \eqref{eq:auc-2}, we obtain
  \begin{align*}
  \auc(f)
  &=
  \frac{1}{2 \theta (1 - \theta)}
  \expectation \bigg[
  \eta(x) (1 - \eta(x')) \I{f(x) > f(x')} 
  \\ & \quad\quad\quad\quad\quad\quad\quad\quad +
  \frac{\eta(x) (1 - \eta(x')) + \eta(x') (1 - \eta(x))}{2} \I{f(x) = f(x')}
  \\ & \quad\quad\quad\quad\quad\quad\quad\quad\quad\quad\quad\quad\quad\quad\quad\quad +
  \eta(x') (1 - \eta(x))
  \I{f(x) < f(x')}
  \bigg]
  \text{.}
  \end{align*}
  Since $\eta(x) (1 - \eta(x')) < \eta(x') (1 - \eta(x)) \iff \eta(x) < \eta(x')$
  (and similarly for the other direction), 
  the expression inside the expectation is maximized by taking $f = \eta$
  (or any strictly increasing transformation of $\eta$).
  and thus the maximum AUC is given by
  \begin{equation}
  \bestauc
  = \frac{1}{2 \theta (1 - \theta)}
  \E{\max\set{\eta(x) (1 - \eta(x')), \eta(x') (1 - \eta(x))}}
  = \E{\phimax_{\auc}(\eta(x), \eta(x'))}
  \text{.}
  \end{equation}
  To prove $\bestauc = \E{\phimin_{\auc}(\eta(x), \eta(x'))}$, observe that
  \begin{equation}
    \label{eq:auc-min-max-sum}
    \begin{aligned}
    & \E{
      \min \set{ \eta(x) (1 - \eta(x')), \eta(x') (1 - \eta(x)) }
      + \max \set{ \eta(x) (1 - \eta(x')), \eta(x') (1 - \eta(x)) }
    }
    \\ ={} &
    \E{ \eta(x) (1 - \eta(x')) + \eta(x') (1 - \eta(x)) }
    \\ ={} &
    2 \theta (1 - \theta)
    \text{.}
    \end{aligned}
  \end{equation}
  Here, we used $\E{\eta(x)} = \theta$ and the independence of $x$ and $x'$.
\end{proof}

\section{Supplementary for \cref{sec:corrupted}}
\label{sec:corrupted:appendix}

\begin{lemma}[store=lem:corrupted:class-prior:bernstein]
  \label{lem:corrupted:class-prior:bernstein}
  Under \cref{ass:unknown-class-prior:epsilon},
  for any 
  $\delta \in (0, 1)$
  and
  $
    n \geq \frac
      {\paren{2 \log (2 / \delta) + 3c}^2}
      {18 \theta (1 - \theta) \log (2 / \delta)}
  $,
  with probability at least $1 - \delta$,
  we have
  $
    \lvert 
      \thetatilde - \theta 
    \rvert
    \leq
    \sqrt{\frac{8 \theta (1 - \theta) \log (2 / \delta)}{n}}
  $.
\end{lemma}

\begin{proof}
  The proof is almost identical to that of 
  \cref{lem:unknown-class-prior:bernstein}.
\end{proof}

\begin{lemma}
  \label{lem:corrupted:lipschitz}
  For any
  $
    \bm{z} = (z_{1}, \dots, z_{n})
    \in [0, 1]^{n}
  $
  and
  $
    \bm{w} = (w_{1}, \dots, w_{n}) 
    \in [0, 1]^{n}
  $,
  we have
  \begin{align}
    \label{eq:corrupted:lipschitz:ber}
    \abs{
      \bestberhat(\bm{w}) 
      - 
      \bestberhat(\bm{z})
    }
    &\leq
    \frac{1}{2 \min \set{\theta, 1 - \theta}} 
    \sqrt{
      \frac{1}{n} \sum_{i=1}^n 
      (w_i - z_i)^{2}
    }
    \text{,}
    \\
    \label{eq:corrupted:lipschitz:auc}
    \abs{
      \bestauchat(\bm{w}) 
      - 
      \bestauchat(\bm{z})
    }
    &\leq
    \frac{1}{\theta (1 - \theta)}
    \sqrt{
      \frac{1}{n} \sum_{i=1}^n
      (w_i - z_i)^{2}
    }
    \text{.}
  \end{align}
\end{lemma}

\begin{proof}
  To prove \eqref{eq:corrupted:lipschitz:ber},
  observe that the function
  $
    z \in [0, 1] 
    \mapsto 
    \min \set{
      \frac{z}{\theta}, 
      \frac{1-z}{1-\theta}
    }
  $
  is 
  $(\min \set{\theta, 1 - \theta})^{-1}$-Lipschitz.
  Hence, we have
  \begin{equation}
    \begin{aligned}
      \abs{
        \widehat{\bestber}(\bm{w}) 
        - 
        \widehat{\bestber}(\bm{z})
      }
      &\leq
      \frac{1}{2 n} \sum_{i=1}^{n} 
      \abs{
        \min \set{
          \frac{w_i}{\theta}, 
          \frac{1 - w_i}{1 - \theta}
        } 
        - 
        \min \set{
          \frac{z_i}{\theta}, 
          \frac{1 - z_i}{1 - \theta}
        }
      }
      \\ &\leq
      \frac{1}{2 \min \set{\theta, 1 - \theta}} 
      \frac{1}{n} \sum_{i=1}^{n} 
      \abs{w_i - z_i}
      \\ &\leq
      \frac{1}{2 \min \set{\theta, 1 - \theta}} 
      \sqrt{
        \frac{1}{n} \sum_{i=1}^n 
        (w_i - z_i)^{2}
      }
      \text{,}
    \end{aligned}
  \end{equation}
  where the last inequality follows 
  from Jensen's inequality.

  Next, we prove \eqref{eq:corrupted:lipschitz:auc}.
  Observe that 
  $
    (z_1, z_2) \in [0, 1]^2 
    \mapsto 
    z_1 (1 - z_2)
  $
  is $1$-Lipschitz with respect to $1$-norm.
  From this, we see that 
  $
    (z_1, z_2) \in [0, 1]^2 
    \mapsto 
    \min \set{z_1 (1 - z_2), z_2 (1 - z_1)}
  $ 
  is also $1$-Lipschitz. 
  Therefore,
  \begin{equation}
    \begin{aligned}
      & 
      \abs{
        \bestauchat(\bm{w}) 
        - 
        \bestauchat(\bm{z})
      }
      \\ \leq{} &
      \frac{1}{\theta(1 - \theta) n (n - 1)} 
      \sum_{i < j}
      \abs{
        \min \set{
          w_i (1 - w_j), 
          w_j (1 - w_i)
        }
        -
        \min \set{
          z_i (1 - z_j), 
          z_j (1 - z_i)
        }
      }
      \\ \leq{} & 
      \frac{1}{\theta(1-\theta)n(n - 1)} 
      \sum_{i < j} 
      \paren{
        \abs{w_i - z_i} 
        + 
        \abs{w_j - z_j}
      }
      \text{.}
    \end{aligned}
  \end{equation}
  For ease of notation,
  let 
  $\Delta_i = \abs{w_i - z_i}$ 
  for each $i = 1, \dots, n$.
  Then, we have
  \begin{equation}
    \begin{aligned}
      \sum_{i<j}
      \paren{
        \abs{w_i - z_i} 
        + 
        \abs{w_j - z_j}
      }
      &= 
      \frac{1}{2} 
      \paren{
        \sum_{i=1}^n \sum_{j=1}^n (\Delta_i + \Delta_j)
        - \sum_{i=1}^n (2 \Delta_i)
      }
      \\ &= 
      \frac{1}{2} \cdot 2 n \sum_{i=1}^n \Delta_i 
      - \sum_{i=1}^n \Delta_i
      \\ &=
      (n - 1) \sum_{i=1}^n \Delta_i
      \text{.}
    \end{aligned}
  \end{equation}
  Therefore, we get
  \begin{equation}
  \abs{
    \bestauchat(\bm{w}) 
    - 
    \bestauchat(\bm{z})
  }
  \leq 
  \frac{1}{\theta (1 - \theta) n} 
  \sum_{i=1}^n 
  \Delta_{i}
  \leq
  \frac{1}{\theta (1 - \theta)} 
  \sqrt{
    \frac{1}{n}
    \sum_{i=1}^n 
    \Delta_{i}^{2}
  }
  \text{.}
  \end{equation}
  The second inequality is again due to Jensen's inequality.
  Now the proof is complete. 
\end{proof}

\begin{lemma}[\citet{ushio2025practical}]
  \label{lem:isotonic-regression-bound}
  Let
  $
    \mathcal{M}_n 
    = 
    \set{
      \bm{u} \in \R^n 
      \mid 
      u_1 \leq \dots \leq u_n
    }
  $
  be
  the set of non-decreasing sequences of length $n$,
  and for constants $a < b$,
  let
  $
    \mathcal{M}_n(a, b)
    = 
    \set{
      \bm{u} \in \R^n 
      \mid 
      a \leq u_1 \leq \dots \leq u_n \leq b
    }
  $.

  Let $\bm{\mu} \in [0, 1]^n$
  be an unknown mean vector,
  which might not belong to $\mathcal{M}_{n}(0, 1)$.
  Given a random vector
  $\bm{y} = (y_1, \dots, y_n) \in \set{0, 1}^{n}$
  of $n$ independent binary observations
  with mean $\E{\bm{y}} = \bm{\mu}$,
  let $\widehat{\bm{\mu}}$ be
  the isotonic regression estimator
  of $\bm{\mu}$, i.e.,
  \begin{equation}
    \widehat{\bm{\mu}}
    \coloneqq
    \argmin_{\bm{u} \in \mathcal{M}_n}
    \enorm{\bm{y} - \bm{u}}^2
    \text{.}
  \end{equation}
  Then, for any
  $\delta \in (0, 1)$,
  with probability at least $1 - \delta$, 
  we have
  \begin{equation}
    \frac{1}{n} 
    \enorm{\widehat{\bm{\mu}} - \bm{\mu}}^2 
    \leq 
    \min_{\bm{u} \in \mathcal{M}_n(0, 1)} 
    \frac{1}{n} 
    \enorm{\bm{u} - \bm{\mu}}^2
    + 
    \paren{
      \frac{C}{n^{1 / 3}}
      + 
      \sqrt{\frac{2 \log (1 / \delta)}{n}}
    }^2
    \text{,}
  \end{equation}
  where $C$ is an absolute constant.
\end{lemma}

\begin{lemma}
  \label{lem:isotonic-root-mean-square:noiseless}
  Assume the skew function $f$ is increasing.
  Then, 
  for any $\delta \in (0, 1)$,
  with probability at least $1 - \delta$,
  we have
  \begin{equation}
    \sqrt{\frac{1}{n} \sum_{i=1}^n 
    \paren{
      \widehat{\eta}_{i} - \eta_{i}
    }^2}
    \leq
    \frac{C}{n^{1/3}}
    + 
    \sqrt{\frac{2 \log(1/\delta)}{n}}
    \text{.}
  \end{equation}
\end{lemma}

\begin{proof}
  By applying \cref{lem:isotonic-regression-bound} 
  for 
  $\bm{\mu} = (\eta_{(1)}, \dots, \eta_{(n)})$,
  with conditional probability at least $1 - \delta$
  given $(\eta_{1}, \dots, \eta_{n})$, 
  we have
  \begin{equation}
    \frac{1}{n} \sum_{i=1}^n 
    \paren{\widehat{\eta}_{i} - \eta_i}^2
    \leq 
    \min_{\bm{u} \in \mathcal{M}_n(0, 1)} 
    \frac{1}{n} \sum_{i=1}^n 
    \paren{u_i - \eta_{(i)}}^2
    + 
    \paren{
      \frac{C}{n^{1/3}}
      + 
      \sqrt{\frac{2 \log(1/\delta)}{n}}
    }^2
    \text{.}
  \end{equation}
  Since $f$ is increasing and 
  $\xi_{(1)} = f(\eta_{(1)}) \leq \dots \leq \xi_{(n)} = f(\eta_{(n)})$, 
  we have
  $\eta_{(1)} \leq \dots \leq \eta_{(n)}$, i.e., $(\eta_{(1)}, \dots, \eta_{(n)}) \in \mathcal{M}_n$.
  Therefore, the first term on the right-hand side is equal to zero.
  This completes the proof.
\end{proof}

\getkeytheorem{thm:corrupted:noiseless}

\begin{proof}
\textbf{BER case.} \quad
  By the same argument as in the proof of
  \cref{lem:unknown-class-prior:ber}
  and
  \cref{thm:unknown-class-prior:ber},
  with probability at least $1 - \frac{4}{5} \delta$,
  we have
  \begin{equation}
    \begin{aligned}
      \abs{
        \widehat{\bestber}(
          \widehat{\bm{\eta}}, 
          \thetatilde
        )
        -
        \widehat{\bestber}(
          \widehat{\bm{\eta}}, 
          \theta
        )
      }
      +
      \abs{
        \widehat{\bestber}(
          \bm{\eta},
          \theta
        )
        -
        \bestber
      }
      &\leq
      \frac{17}{8}
      \sqrt{
        \frac{2 \log (5/\delta)}{\theta (1 - \theta) n}
      }
      \\ &\leq
      \frac{17}{8 \min \set{\theta, 1 - \theta}}
      \sqrt{
        \frac{2 \log (5/\delta)}{n}
      }
    \text{.}
    \end{aligned}
  \end{equation}
  Also,
  by 
  \cref{lem:corrupted:lipschitz,lem:isotonic-root-mean-square:noiseless},
  with probability at least $1 - \frac{1}{5} \delta$,
  \begin{equation}
    \abs{
      \widehat{\bestber}(
        \widehat{\bm{\eta}}, 
        \theta
      )
      -
      \widehat{\bestber}(
        \bm{\eta},
        \theta
      )
    }
    \leq
    \frac{1}{2 \min \set{\theta, 1 - \theta}}
    \paren{
      \frac{C}{n^{1 / 3}}
      +
      \sqrt{
        \frac{2 \log (5 / \delta)}{n}
      }
    }
  \end{equation}
  Therefore, we have the following with
  probability at least $1 - \delta$:
  \begin{equation}
    \begin{aligned}
      &
      \abs{
        \widehat{\bestber}(
          \widehat{\bm{\eta}}, 
          \thetatilde
        )
        -
        \bestber
      }
      \\ \leq{} &
      \abs{
        \widehat{\bestber}(
          \widehat{\bm{\eta}}, 
          \thetatilde
        )
        -
        \widehat{\bestber}(
          \widehat{\bm{\eta}}, 
          \theta
        )
      }
      +
      \abs{
        \widehat{\bestber}(
          \widehat{\bm{\eta}}, 
          \theta
        )
        -
        \widehat{\bestber}(
          \bm{\eta},
          \theta
        )
      }
      +
      \abs{
        \widehat{\bestber}(
          \bm{\eta},
          \theta
        )
        -
        \bestber
      }
      \\ \leq{} &
      \frac{17}{8 \min \set{\theta, 1 - \theta}}
      \sqrt{
        \frac{2 \log (5/\delta)}{n}
      }
      +
      \frac{1}{2 \min \set{\theta, 1 - \theta}}
      \paren{
        \frac{C}{n^{1 / 3}}
        +
        \sqrt{
          \frac{2 \log (5 / \delta)}{n}
        }
      }
      \\ ={} &
      \frac{1}{2 \min \set{\theta, 1 - \theta}}
      \paren{
        \frac{C}{n^{1 / 3}}
        +
        \frac{21}{4}
        \sqrt{
          \frac{2 \log (5 / \delta)}{n}
        }
      }
      \text{.}
    \end{aligned}
  \end{equation}

\textbf{AUC case.} \quad
  First, suppose that
  $
    2
    \leq
    n
    <
    \frac{
      18 \log(5 / \delta)
    }{
      \theta(1 - \theta)
    }
  $.
  Since both
  $
    \widetilde{\bestauc}(
      \widehat{\bm{\eta}},
      \thetatilde
    )
  $
  and $\bestauc$ belong to $[\frac{1}{2}, 1]$, we have
  \begin{equation}
    \abs{
      \widetilde{\bestauc}(
        \widehat{\bm{\eta}},
        \thetatilde
      )
      -
      \bestauc
    }
    \leq
    \frac{1}{2}
    \text{.}
  \end{equation}
  On the other hand, 
  since $n < \frac{18 \log(5 / \delta)}{\theta(1 - \theta)}$
  and
  $\theta(1-\theta) \leq \frac{1}{4}$,
  we have
  \begin{equation}
    \begin{aligned}
      \frac{1}{\theta(1 - \theta)}
      \paren{
        \frac{1}{n^{1/3}}
        +
        \sqrt{
          \frac{\log(5 / \delta)}{n}
        }
      }
      &>
      \frac{1}{\theta(1 - \theta)}
      \sqrt{
        \frac{\theta(1 - \theta)}{18}
      }
      \\&=
      \frac{1}{
        \sqrt{18\theta(1 - \theta)}
      }
      \\&\geq
      \frac{\sqrt{2}}{3}
      \text{.}
    \end{aligned}
  \end{equation}
  Therefore, the claimed bound holds in this case
  with the hidden constant at most $\frac{3}{2 \sqrt{2}}$.

  It remains to consider the large-$n$ case:
  $
    n
    \geq
    \frac{
      18 \log(5 / \delta)
    }{
      \theta(1 - \theta)
    }
  $.
  Let
  \begin{equation}
    A(\theta)
    \coloneqq
    \frac{1}{\theta (1 - \theta)}
    \text{,}
    \quad
    B(\bm{\eta})
    \coloneqq
    \frac{1}{n (n - 1)}
    \sum_{i < j}
    \min \set{
      \eta_{i} (1 - \eta_{j}), 
      \eta_{j}(1 - \eta_{i})
    }
    \text{.}
  \end{equation}
  Then,
  since $n \geq \frac{18 \log (5 / \delta)}{\theta (1 - \theta)}$,
  by the same argument as the proof of 
  \cref{lem:unknown-class-prior:auc2},
  with probability at least $1 - \frac{2}{5} \delta$,
  both
  \eqref{eq:corrupted:auc:noiseless:a-diff-bound}
  and
  \eqref{eq:corrupted:auc:noiseless:a-bound}
  hold:
  \begin{align}
    \frac{
      \abs{
        A(\thetatilde) - A(\theta)
      }
    }{
      A(\theta)
    }
    &=
    \frac{
      \abs{
        \thetatilde (1 - \thetatilde)
        -
        \theta (1 - \theta)
      }
    }{
      \thetatilde (1 - \thetatilde)
    }
    \nonumber
    \\ &\leq
    \frac{
      \abs{
        \thetatilde
        -
        \theta
      }
    }{
      \thetatilde (1 - \thetatilde)
    }
    \nonumber
    \\ &\leq
    \label{eq:corrupted:auc:noiseless:a-diff-bound}
    6 \sqrt{\frac{2 \log (5 / \delta)}{\theta (1 - \theta) n}}
    \text{,}
    \\
    \label{eq:corrupted:auc:noiseless:a-bound}
    A(\thetatilde)
    &=
    \frac{1}{\thetatilde (1 - \thetatilde)}
    \leq
    \frac{3}{\theta (1 - \theta)}
    \text{.}
  \end{align}

  Also, by
  \cref{lem:corrupted:lipschitz,lem:isotonic-root-mean-square:noiseless},
  with probability at least $1 - \frac{1}{5} \delta$,
  \begin{equation}
    \abs{
      B(\widehat{\bm{\eta}})
      -
      B(\bm{\eta})
    }
    \leq
    \sqrt{
      \frac{1}{n}
      \sum_{i=1}^{n}
      \paren{\widehat{\eta}_{i} - \eta_{i}}^{2}
    }
    \leq
    \frac{C}{n^{1/3}}
    + 
    \sqrt{\frac{2 \log(5/\delta)}{n}}
    \text{.}
  \end{equation}
  Therefore,
  we have
  \begin{equation}
    \begin{aligned}
      &
      \abs{
        \widehat{\bestauc}(
          \widehat{\bm{\eta}},
          \thetatilde
        )
        -
        \widehat{\bestauc}(
          \bm{\eta},
          \theta
        )
      }
      \\ ={} &
      \abs{
        A(\thetatilde) B(\widehat{\bm{\eta}})
        -
        A(\theta) B(\bm{\eta})
      }
      \\ \leq{} &
      A(\thetatilde)
      \abs{B(\widehat{\bm{\eta}}) - B(\bm{\eta})}
      + 
      A(\theta) B(\bm{\eta})
      \cdot
      \frac{\abs{A(\thetatilde) - A(\theta)}}{A(\theta)}
      \\ \leq{} &
      \frac{3}{\theta (1 - \theta)}
      \paren{
        \frac{C}{n^{1 / 3}}
        +
        \sqrt{
          \frac{2 \log (5 / \delta)}{n}
        }
      }
      +
      (
        1
        -
        \widehat{\bestauc}(
          \bm{\eta},
          \theta
        )
      )
      \cdot
      6 \sqrt{
        \frac{2 \log (5 / \delta)}{\theta (1 - \theta) n}
      }
      \text{.}
    \end{aligned}
  \end{equation}
  Also, by the assumption that
  $n \geq \frac{18 \log (5 / \delta)}{\theta (1 - \theta)} > \frac{\log (5 / \delta)}{9 \theta (1 - \theta)}$,
  we can apply \cref{lem:ber-auc:bound:auc:bernstein}
  to obtain that
  with probability at least $1 - \frac{2}{5} \delta$,
  \begin{equation}
    \abs{
      \widehat{\bestauc}(\bm{\eta}, \theta)
      -
      \bestauc
    }
    \leq
    \sqrt{
      \frac{\log (5/\delta)}{\theta (1 - \theta) n}
    }
    \text{.}
  \end{equation}
  Combining these results,
  with probability at least $1 - \delta$,
  we have
  \begin{equation}
    \begin{aligned}
      &
      \abs{
        \widetilde{\bestauc}(
          \widehat{\bm{\eta}}, 
          \thetatilde
        )
        -
        \bestauc
      }
      \\ \leq{} &
      \abs{
        \widehat{\bestauc}(
          \widehat{\bm{\eta}}, 
          \thetatilde
        )
        -
        \bestauc
      }
      \\ \leq{} &
      \abs{
        \widehat{\bestauc}(
          \widehat{\bm{\eta}},
          \thetatilde
        )
        -
        \widehat{\bestauc}(
          \bm{\eta},
          \theta
        )
      }
      +
      \abs{
        \widehat{\bestauc}(\bm{\eta}, \theta)
        -
        \bestauc
      }
      \\ \leq{} &
      \frac{3}{\theta (1 - \theta)}
      \paren{
        \frac{C}{n^{1 / 3}}
        +
        \sqrt{
          \frac{2 \log (5 / \delta)}{n}
        }
      }
      +
      \paren{
        1 - \bestauc
        +
        \sqrt{
          \frac{\log (5 / \delta)}{\theta (1 - \theta) n}
        }
      }
      \cdot
      6 \sqrt{
        \frac{2 \log (5 / \delta)}{\theta (1 - \theta) n}
      }
      +
      \sqrt{
        \frac{\log (5/\delta)}{\theta (1 - \theta) n}
      }
      \\ \leq{} &
      \frac{3}{\theta (1 - \theta)}
      \paren{
        \frac{C}{n^{1 / 3}}
        +
        \sqrt{
          \frac{2 \log (5 / \delta)}{n}
        }
      }
      +
      \frac{
        3 \sqrt{2} 
        +
        3
      }{2 \theta (1 - \theta)}
      \sqrt{
        \frac{\log (5 / \delta)}{n}
      }
      \text{,}
    \end{aligned}
  \end{equation}
  where the last inequality follows from
  $\bestauc \geq \frac{1}{2}$,
  $
    \frac{\log (5 / \delta)}{\theta (1 - \theta) n}
    \leq
    \frac{1}{3 \sqrt{2}}
    \sqrt{\frac{\log (5 / \delta)}{\theta (1 - \theta) n}}
  $
  (by
  $n \geq \frac{18 \log (5 / \delta)}{\theta (1 - \theta)}$),
  and
  $\sqrt{\theta (1 - \theta)} \geq 2 \theta (1 - \theta)$.
  Now the claim has been proved.
\end{proof}

\paragraph{Noisy corruption model}

\cref{thm:corrupted:noisy} is proved by
using the following lemma instead of
\cref{lem:isotonic-root-mean-square:noiseless}
in the proof of
\cref{thm:corrupted:noiseless}.
This lemma is the counterpart of
\cref{lem:isotonic-root-mean-square:noiseless}
in the noisy setting.

\begin{lemma}
  \label{lem:isotonic-root-mean-square:noisy}
  Under \cref{ass:corrupted:noisy},
  for any $\delta \in (0, 1)$, with probability 
  at least $1 - \delta$, we have 
  \begin{equation}
    \sqrt{\frac{1}{n} \sum_{i=1}^n 
    \paren{
      \widehat{\eta}_{i} - \eta_{i}
    }^2}
    \leq
      \frac{\sigma}{\gamma}
      +
      \frac{C}{n^{1/3}} 
      +
      (1 + \sqrt{2})
      \sqrt{\frac{\log(2/\delta)}{n}}
    \text{.}
  \end{equation}
\end{lemma}

\begin{proof}
  By applying \cref{lem:isotonic-regression-bound} 
  for 
  $\bm{\mu} = (\eta_{(1)}, \dots, \eta_{(n)})$,
  with conditional probability at least $1 - \delta/2$
  given $\eta_{1}, \dots, \eta_{n}, \epsilon_{1}, \dots, \epsilon_{n}$, 
  we have
  \begin{equation}
    \label{eq:isotonic-root-mean-square:noisy:1}
    \frac{1}{n} \sum_{i=1}^n 
    \paren{\widehat{\eta}_{i} - \eta_i}^2
    \leq 
    \min_{\bm{u} \in \mathcal{M}_n(0, 1)} 
    \frac{1}{n} \sum_{i=1}^n 
    \paren{u_i - \eta_{(i)}}^2
    + 
    \paren{
      \frac{C}{n^{1/3}}
      + 
      \sqrt{\frac{2 \log(2/\delta)}{n}}
    }^2
    \text{.}
  \end{equation}

  We then proceed to bound the first term on the right-hand side of 
  \eqref{eq:isotonic-root-mean-square:noisy:1}.
  Under the assumption, 
  $f$ is invertible, and 
  $f^{-1}$ is also differentiable and increasing. 
  Since $f^{-1}$ is increasing and 
  $\xi_{(1)} \leq \dots \leq \xi_{(n)}$,
  we have 
  $\paren{f^{-1}(\xi_{(1)}), \dots, f^{-1}(\xi_{(n)})} \in \mathcal{M}_{n}(0, 1)$.
  Therefore, we have
  \begin{equation}
    \min_{\bm{u} \in \mathcal{M}_{n}(0, 1)} 
    \frac{1}{n} \sum_{i=1}^{n} (u_{i} - \eta_{(i)})^{2}
    \leq
    \frac{1}{n} \sum_{i=1}^{n} 
    (f^{-1}(\xi_{(i)}) - \eta_{(i)})^{2}
    =
    \frac{1}{n} \sum_{i=1}^{n} 
    (f^{-1}(\xi_{i}) - \eta_{i})^{2}
    \text{.}
  \end{equation}
  Now, let
  $D_{i} \coloneqq (f^{-1}(\xi_{i}) - \eta_{i})^{2}$.
  The mean value theorem and the assumption that $f' \geq \gamma$ imply that
  \begin{equation}
    D_{i}
    =
    \paren{
      f^{-1}(f(\eta_{i}) + \epsilon_{i}) - f^{-1}(f(\eta_{i}))
    }^{2}
    \leq
    \frac{\epsilon_{i}^{2}}{\gamma^{2}}
  \end{equation}
  and hence 
  $\E{D_{i} \mid \eta_1,\dots,\eta_n} \leq \frac{\sigma^{2}}{\gamma^{2}}$.
  Also, since $D_{i} \in [0, 1]$,
  $\Var{D_{i} \mid \eta_1,\dots,\eta_n} \leq \E{D_{i}^{2} \mid \eta_{1}, \dots, \eta_{n}} \leq \E{D_{i} \mid \eta_{1}, \dots, \eta_{n}} \leq \frac{\sigma^{2}}{\gamma^{2}}$.
  Moreover, $D_1,\dots,D_n$ are conditionally independent
  given $\eta_1,\dots,\eta_n$.
  Therefore, by Bernstein's inequality, we have
  \begin{equation}
    \frac{1}{n} \sum_{i=1}^{n} D_{i}
    \leq 
    \frac{\sigma^2}{\gamma^{2}}
    + 
    \sqrt{\frac{2 (\sigma^{2} / \gamma^{2}) \log(2/\delta)}{n}}
    +
    \frac{2 \log(2/\delta)}{3 n}
    \leq 
    \paren{
      \frac{\sigma}{\gamma}
      + 
      \sqrt{\frac{\log(2/\delta)}{n}}
    }^{2}
  \end{equation}
  with conditional probability at least $1 - \delta/2$
  given $\eta_{1}, \dots, \eta_{n}$.
  Therefore, with probability at least $1 - \delta$, we have
  \begin{equation}
    \begin{aligned}
      \sqrt{\frac{1}{n} \sum_{i=1}^n 
      \paren{
        \widehat{\eta}_{i} - \eta_{i}
      }^2}
      &\leq
      \sqrt{
        \paren{
          \frac{\sigma}{\gamma}
          + 
          \sqrt{\frac{\log(2/\delta)}{n}}
        }^{2}
        +
        \paren{
          \frac{C}{n^{1/3}}
          + 
          \sqrt{\frac{2 \log(2/\delta)}{n}}
        }^2
      }
      \\ &\leq
        \frac{\sigma}{\gamma}
        + 
        \sqrt{\frac{\log(2/\delta)}{n}}
        +
        \frac{C}{n^{1/3}}
        + 
        \sqrt{\frac{2 \log(2/\delta)}{n}}
      \\ &=
      \frac{\sigma}{\gamma}
      +
      \frac{C}{n^{1/3}} 
      +
      (1 + \sqrt{2})
      \sqrt{\frac{\log(2/\delta)}{n}}
      \text{.}
    \end{aligned}
  \end{equation}
\end{proof}

\section{Supplementary for \cref{sec:evaluation}}
\label{sec:proof-evaluation}

\subsection{Proof of \cref{thm:feebee}}
\label{sec:proof-feebee}

Here we prove \cref{thm:feebee}, which gives the exact expressions of the optimal BER and AUC values for the noise-injected distribution.

\getkeytheorem{thm:feebee}

The BER part is proved in \cref{thm:feebee-ber} and 
the AUC part is proved in \cref{thm:feebee-auc}.

\begin{lemma}
  \label{lem:feebee-ber-rewrite}
  The optimal BER value can be written as
  \begin{equation}
    \bestber = 
    \frac{1}{2 \theta (1 - \theta)}
    \paren{
      \E{
        \min \set{
          \eta(x), \theta
        }
      }
      - \theta^{2}
    }
    \text{.}
  \end{equation}
\end{lemma}

\begin{proof}
  \begin{equation}
    \begin{aligned}
      \bestber
      &= 
      \frac{1}{2} \E{
        \min \set{
          \frac{\eta(x)}{\theta},
          \frac{1 - \eta(x)}{1 - \theta}
        }
      }
      \quad \text{(\cref{lem:bestber})}
      \\ &=
      \frac{1}{2 \theta (1 - \theta)} \E{
        \min \set{
          \eta(x) (1 - \theta),
          \theta (1 - \eta(x))
        }
      }
      \\ &=
      \frac{1}{2 \theta (1 - \theta)} \E{
        \min \set{
          \eta(x), \theta
        }
        - \theta \eta(x)
      }
      \\ &=
      \frac{1}{2 \theta (1 - \theta)}
      \paren{
        \E{
          \min \set{
            \eta(x), \theta
          }
        }
        - \theta^{2}
      }
      \text{,}
    \end{aligned}
  \end{equation}
  where we used $\E{\eta(x)} = \theta$  in the last step.
\end{proof}

\begin{theorem}[store=thm:feebee-ber]
  \label{thm:feebee-ber}
  For any noise level $\nu \in [0, 1)$, 
  the optimal BER for
  the noise-injected distribution is given by
  \begin{equation}
  \bestber_\nu = F_{\nu}(\bestber)
  \text{.}
  \end{equation}
\end{theorem}

\begin{proof}
  Applying \cref{lem:feebee-ber-rewrite} to the noisy distribution
  (and recalling that we assumed $\lambda_{\nu}(\theta) \in (0, 1)$), we have
  \begin{equation}
    \bestber_{\nu} =
    \frac{
      \E{\min \set{\lambda_{\nu}(\eta(x)), \lambda_{\nu}(\theta)}}
      - \lambda_{\nu}(\theta)^{2}
    }{
      2 \lambda_{\nu}(\theta) (1 - \lambda_{\nu}(\theta))
    }
    \text{.}
  \end{equation}
  Since $\lambda_{\nu}$ is non-decreasing and affine, we can write
  \begin{equation}
    \E{\min \set{\lambda_{\nu}(\eta(x)), \lambda_{\nu}(\theta)}}
    =
    \E{\lambda_{\nu}( \min \set{\eta(x), \theta} )}
    = 
    \lambda_{\nu}(\E{ \min \set{\eta(x), \theta} })
    \text{.}
  \end{equation}
  Now, using \cref{lem:feebee-ber-rewrite} again for the clean distribution,
  we have
  \begin{equation}
    \E{ \min \set{\eta(x), \theta} }
    =
    2 \theta (1 - \theta) \bestber + \theta^{2}
  \end{equation}
  which gives
  \begin{equation}
    \bestber_{\nu}
    =
    \frac{
      \lambda_{\nu} \paren{
        2 \theta (1 - \theta) \bestber + \theta^{2}
      }
      - \lambda_{\nu}(\theta)^{2}
    }{
      2 \lambda_{\nu}(\theta) (1 - \lambda_{\nu}(\theta))
    }
    \text{,}
  \end{equation}
  which proves the theorem.
\end{proof}

\begin{lemma}
  \label{lem:feebee-auc-rewrite}
  The optimal AUC value can be written as
  \begin{equation}
    \bestauc = 
    1 -
    \frac{1}{2 \theta (1 - \theta)}
    \paren{
      \E{
        \min \set{
          \eta(x), \eta(x')
        }
      }
      - \theta^{2}
    }
    \text{,}
  \end{equation}
  where $x'$ is an i.i.d.\ copy of $x$.
\end{lemma}

\begin{proof}
  \begin{equation}
    \begin{aligned}
      1 - \bestauc
      &= 
      \frac{1}{2 \theta (1 - \theta)} 
      \E{
        \min \set{
          \eta(x) (1 - \eta(x')), \eta(x') (1 - \eta(x))
        }
      }
      \quad \text{(\cref{lem:bestauc})}
      \\ &=
      \frac{1}{2 \theta (1 - \theta)} \E{
        \min \set{
          \eta(x), \eta(x')
        }
        - \eta(x) \eta(x')
      }
      \\ &=
      \frac{1}{2 \theta (1 - \theta)}
      \paren{
        \E{
          \min \set{
            \eta(x), \eta(x')
          }
        }
        - \theta^{2}
      }
      \text{,}
    \end{aligned}
  \end{equation}
  where we used $\E{\eta(x)} = \E{\eta(x')}= \theta$
  and the independence of $x$ and $x'$ in the last step.
\end{proof}

\begin{theorem}[store=thm:feebee-auc]
  \label{thm:feebee-auc}
  For any noise level $\nu \in [0, 1)$, 
  the optimal AUC for
  the noise-injected distribution is given by
  \begin{equation}
  \bestauc_\nu = 1 - F_{\nu}(1 - \bestauc)
  \text{.}
  \end{equation}
\end{theorem}

\begin{proof}
  The proof is analogous to that of \cref{thm:feebee-ber}.
  Applying \cref{lem:feebee-auc-rewrite} to the noisy distribution, we have
  \begin{equation}
    1 - \bestauc_{\nu} =
    \frac{
      \E{\min \set{\lambda_{\nu}(\eta(x)), \lambda_{\nu}(\eta(x'))}}
      - \lambda_{\nu}(\theta)^{2}
    }{
      2 \lambda_{\nu}(\theta) (1 - \lambda_{\nu}(\theta))
    }
    \text{.}
  \end{equation}
  Since $\lambda_{\nu}$ is non-decreasing and affine, we can write
  \begin{equation}
    \E{\min \set{\lambda_{\nu}(\eta(x)), \lambda_{\nu}(\eta(x'))}}
    =
    \E{\lambda_{\nu}( \min \set{\eta(x), \eta(x')} )}
    = 
    \lambda_{\nu}(\E{ \min \set{\eta(x), \eta(x')} })
    \text{.}
  \end{equation}
  Now, using \cref{lem:feebee-auc-rewrite} again for the clean distribution,
  we have
  \begin{equation}
    \E{ \min \set{\eta(x), \eta(x')} }
    =
    2 \theta (1 - \theta) (1 - \bestauc) + \theta^{2}
  \end{equation}
  which gives
  \begin{equation}
    1 - \bestauc_{\nu}
    =
    \frac{
      \lambda_{\nu} \paren{
        2 \theta (1 - \theta) (1 - \bestauc) + \theta^{2}
      }
      - \lambda_{\nu}(\theta)^{2}
    }{
      2 \lambda_{\nu}(\theta) (1 - \lambda_{\nu}(\theta))
    }
    \text{,}
  \end{equation}
  which proves the theorem.
\end{proof}

\subsection{Bias sensitivity}
\label{sec:evaluation:bias}

Here we show that our score can reflect certain types of estimator bias,
unlike bootstrap variance.

\paragraph{Theoretical analysis}

The following proposition roughly states that
any bias larger than the interval width
$U_{\ber}(\nu) - L_{\ber}(\nu)$ produces a positive expected
penalty.
This is in contrast to the bootstrap variance,
which does not change if a constant is added to the estimator.
Here we discuss the BER case, but the same argument also applies to AUC
by replacing $u_{\ber}$ with $1 - l_{\auc}$.

\begin{proposition}
  \label{prop:evaluation:bias}
  Let
  $\mathrm{bias}_{\nu} =
    \E{\widehat{\bestber_{\nu}}} - \bestber_{\nu}$
  denote the bias of the estimator
  $\widehat{\bestber_{\nu}}$ at noise level $\nu$.
  Then
  \begin{equation}
    \E{
      \paren{\widehat{\bestber_{\nu}} - U_{\ber}(\nu)}_+
      +
      \paren{L_{\ber}(\nu) - \widehat{\bestber_{\nu}}}_+
    }
    \geq
    \left[
      \abs{\mathrm{bias}_{\nu}}
      -
      \frac{u_{\ber}}{K} (1 - \nu)
    \right]_+
    \text{,}
  \end{equation}
  where
  $
    K
    \coloneqq
    \min \set{
    1,
    \frac{\beta (1 - \beta)}{\theta (1 - \theta)}
    }
    \in (0, 1]
  $.
\end{proposition}

\begin{proof}
  Let $d(z_1, z_2) = |z_1 - z_2|$
  and $I_{\nu} = [L_{\ber}(\nu), U_{\ber}(\nu)]$.
  Then, the pointwise penalty
  $
    (
    \widehat{\bestber_{\nu}} - U_{\ber}(\nu)
    )_+
    +
    (
    L_{\ber}(\nu) - \widehat{\bestber_{\nu}}
    )_+
  $
  at noise level $\nu$ is precisely
  $
    d(\widehat{\bestber_{\nu}}, I_{\nu}) \coloneqq \inf_{a \in I_{\nu}} d(\widehat{\bestber_{\nu}}, a)
  $.
  By the convexity of $d(\cdot, I_{\nu})$ and Jensen's inequality,
  the expected pointwise penalty is bounded as
  \begin{equation}
    \label{eq:evaluation:bias:jensen}
    \E{d \paren{\widehat{\bestber_\nu}, I_{\nu}}}
    \geq
    d \paren{\E{\widehat{\bestber_\nu}}, I_{\nu}}
    \text{.}
  \end{equation}
  On the other hand,
  since $\bestber_{\nu} \in I_{\nu}$,
  the bias of the estimator $\widehat{\bestber_{\nu}}$ is bounded as follows:\footnote{
    This follows from the following fact:
    let $(\mathcal{Z},d)$ be a metric space
    and let $A \subset \mathcal{Z}$.
    For any $z \in \mathcal{Z}$ and $a \in A$, we have
    $d(z, a) \leq d(z, A) + \mathrm{diam}(A)$, where
    $d(z, A) = \inf_{a' \in A} d(z, a')$ and
    $\mathrm{diam}(A) = \sup_{a_1, a_2 \in A} d(a_1, a_2)$.
  }
  \begin{equation}
    \label{eq:evaluation:bias:bias-bound}
    \abs{\mathrm{bias}_{\nu}}
    =
    d\paren{\E{\widehat{\bestber_\nu}}, \bestber_{\nu}}
    \leq
    d\paren{\E{\widehat{\bestber_\nu}}, I_{\nu}}
    +
    U_{\ber}(\nu) - L_{\ber}(\nu)
    \text{.}
  \end{equation}
  The interval width $U_{\ber}(\nu) - L_{\ber}(\nu)$ is bounded as
  \begin{equation}
    \begin{aligned}[b]
      U_{\ber}(\nu) - L_{\ber}(\nu)
      &=
      F_{\nu}(u_{\ber}) - F_{\nu}(0)
      \\&=
      \frac{
        (1 - \nu)
        \theta (1 - \theta) u_{\ber}
      }{
        \lambda_{\nu}(\theta) (1 - \lambda_{\nu}(\theta))
      }
      \\&\leq
      \frac{
        (1 - \nu)
        \theta (1 - \theta) u_{\ber}
      }{
        \min \set{
          \theta (1 - \theta),
          \beta (1 - \beta)
        }
      }
      \\&=
      \label{eq:evaluation:bias:interval-width-bound}
      \frac{ u_{\ber} }{ K }
      (1 - \nu)
      \text{.}
    \end{aligned}
  \end{equation}
  Combining
  \eqref{eq:evaluation:bias:jensen},
  \eqref{eq:evaluation:bias:bias-bound}, and
  \eqref{eq:evaluation:bias:interval-width-bound},
  we conclude that
  \begin{equation}
    \E{d \paren{\widehat{\bestber_{\nu}}, I_{\nu}}}
    \geq
    \left[
      \abs{\mathrm{bias}_{\nu}}
      -
      \frac{u_{\ber}}{K} (1 - \nu)
    \right]_+
    \text{.}
  \end{equation}
\end{proof}

Thus, for each noise level $\nu$, 
a bias larger than $\frac{u_{\ber}}{K} (1 - \nu)$
produces a positive expected penalty.
Moreover, since $\frac{u_{\ber}}{K} (1 - \nu) \to 0$ as $\nu \to 1$,
any fixed nonzero additive bias that persists across noise levels
is expected to be detected at a sufficiently large noise level.
It is worth noting that this is not a universal guarantee:
a bias that shrinks sufficiently fast as $\nu \to 1$
may remain undetected.

\cref{prop:evaluation:bias} quantifies 
the \emph{pointwise} penalty for each noise level,
and we can use it to discuss the overall score aggregated over different noise levels.
As a simple example, consider an idealized version of $s_{\beta}$,
obtained by replacing the finite average over $N$ noise levels
with an integral (informally, $N \to \infty$):
\begin{equation}
  s_{\beta}
  =
  \int_0^1 d \paren{\widehat{\bestber_{\nu}}, I_{\nu}} \, d\nu
  =
  \int_{0}^{1}
  \left[
    \paren{\widehat{\bestber_{\nu}} - U_{\ber}(\nu)}_+
    +
    \paren{L_{\ber}(\nu) - \widehat{\bestber_{\nu}}}_+
  \right]
  \, d\nu
  \text{.}
\end{equation}
The result above implies that the expected score is bounded as
\begin{equation}
  \E{s_{\beta}}
  \geq
  \int_0^1
  \left[
    \abs{\mathrm{bias}_{\nu}}
    -
    \frac{u_{\ber}}{K} (1 - \nu)
  \right]_+
  \, d\nu
  \text{.}
\end{equation}
For example, 
if $\abs{\mathrm{bias}_{\nu}} \geq B \geq 0$ for all $\nu \in [0, 1)$,
the expected score is lower-bounded
by an increasing function of $B$ as
\begin{equation}
  \label{eq:evaluation:bias:score-lower-bound}
  \E{s_{\beta}}
  \geq
  \begin{cases}
    \frac{K}{2 u_{\ber}} B^2
    & \text{if } B \leq \frac{u_{\ber}}{K} \text{,}
    \\
    B - \frac{u_{\ber}}{2K}
    & \text{otherwise.}
  \end{cases}
\end{equation}

\paragraph{Empirical analysis}

We also present an experiment using synthetic data with
controlled bias.
We drew $n=10{,}000$ points from a one-dimensional
two-component Gaussian mixture with means $-0.5$ and $0.5$,
unit variances, and class prior $\theta=0.5$.
For each point, we computed the clean soft label $\eta_i$
using the closed form obtained from the Gaussian PDF.
By \cref{prop:gaussian}, the true optimal values are
$\bestber = \Phi(-1/2) \approx 0.308538$ and
$\bestauc = \Phi(1/\sqrt{2}) \approx 0.760250$.

We repeated the experiment 100 times.
In each trial, we computed the min formula estimators
at $N=100$ noise levels.
To change the bias without changing the variance,
we defined shifted estimators at each noise level $\nu$ as
$\widehat{\bestber_{\nu}}^{(b)}
\coloneqq \widehat{\bestber_{\nu}}^{(0)} - b$
and
$\widehat{\bestauc_{\nu}}^{(b)}
\coloneqq \widehat{\bestauc_{\nu}}^{(0)} + b$.
Because the same constant is added to or subtracted
from the estimate in every trial,
these changes do not change the variance of either estimator.
We used
$b \in \{0, 0.01, 0.02, 0.05, 0.10, 0.15\}$
and the same samples for every value of $b$.
The parameters $u_{\ber}$ and $l_{\auc}$ were both set to $0.5$.
We computed the scores $s_\beta$
for $\beta = 0.5$
and the average scores over 
$\beta \in \{0.1, \dots, 0.9\}$.

\cref{fig:evaluation:bias} reports averages over the 100 trials
(blue and orange solid curves)
along with the theoretical lower bounds in
\eqref{eq:evaluation:bias:score-lower-bound}
(blue and orange dotted curves).
All standard errors of the scores were below $2.8\times10^{-5}$.
The standard deviations (SDs) of the BER estimates and
AUC estimates were 0.001241 and 0.001370, respectively,
for every value of $b$
(as expected, since the variance was kept constant).
In contrast, the score averaged over $\beta$ increased monotonically with $b$,
from 0 to 0.029953 for BER and from 0 to 0.040368 for AUC.
These results show that, unlike the standard deviations,
our evaluation score successfully reflects
the increase in bias even when
the estimator variance is unchanged.

\begin{figure}[t]
  \centering
  \begin{subfigure}[t]{0.49\textwidth}
    \centering
    \includegraphics[width=\linewidth]{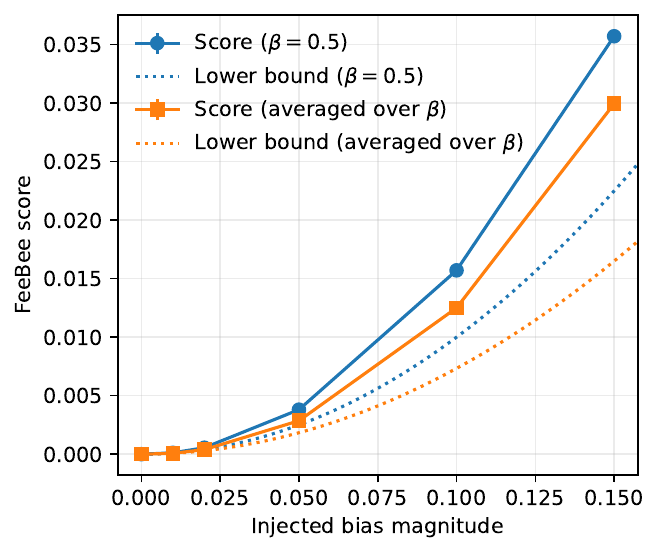}
    \caption{BER}
    \label{fig:evaluation:bias:ber}
  \end{subfigure}
  \hfill
  \begin{subfigure}[t]{0.49\textwidth}
    \centering
    \includegraphics[width=\linewidth]{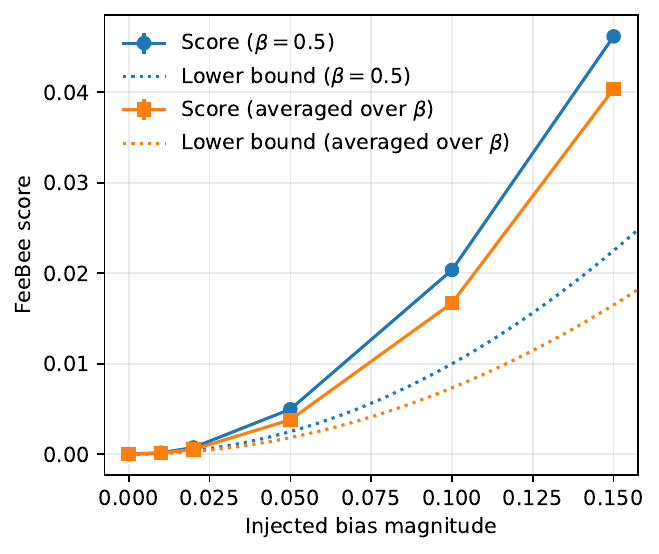}
    \caption{AUC}
    \label{fig:evaluation:bias:auc}
  \end{subfigure}
  \caption{
    Our evaluation scores versus injected bias magnitude
    for BER and AUC.
    Points show averages over 100 trials.
    Their 95\% Monte Carlo confidence intervals
    are smaller than the markers.
    The blue curves use $\beta=0.5$,
    whereas the orange curves average the scores
    over $\beta\in\{0.1,\dots,0.9\}$.
    The blue and orange dotted curves show the theoretical lower bounds in
    \eqref{eq:evaluation:bias:score-lower-bound} 
    for the corresponding $\beta$ settings,
    evaluated at $B = b$.
    Both scores increase monotonically with the injected bias
    (although the estimator standard deviations remain unchanged).
  }
  \label{fig:evaluation:bias}
\end{figure}

\section{
  Closed-form expressions for the optimal BER and AUC 
  under Gaussian mixture models with a shared covariance matrix
}
\label{sec:gaussian}

Here, we derive the closed-form expressions for
the optimal BER and AUC
used as reference values in the
Gaussian-mixture experiments of \cref{sec:app:experiments}.

Suppose $\mathcal{X} = \R^{d}$ and that
the class-conditional distributions are given by
Gaussian distributions with a shared covariance matrix:
\begin{equation}
  x \mid \{y=0\} \sim \mathcal{N}(\mu_0, \Sigma)
  \text{,}\quad
  x \mid \{y=1\} \sim \mathcal{N}(\mu_1, \Sigma)
  \text{,}
\end{equation}
where $\mu_0, \mu_1 \in \R^d$ and 
$\Sigma \in \R^{d \times d}$ is a positive definite matrix.
For $k \in \{0, 1\}$, let $p_k(x)$ denote the class-conditional 
density function of $x$ given $y=k$:
\begin{equation}
  p_k(x) = \frac{1}{(2 \pi)^{d/2} \det(\Sigma)^{1/2}}
  \exp\paren{
    -\frac{1}{2} (x - \mu_k)^\top \Sigma^{-1} (x - \mu_k)
  }
  \text{.}
\end{equation}

The goal of this section is to prove the following proposition, 
which gives closed-form expressions for the optimal BER and AUC 
under this setting.

\begin{proposition}
  \label{prop:gaussian}
  Under the shared-covariance Gaussian mixture model described above,
  the optimal BER and AUC are given by
  \begin{equation}
    \bestber = \Phi\paren{- \frac{J}{2}}
    \text{,} \quad
    \bestauc = \Phi\paren{\frac{J}{\sqrt{2}}}
    \text{,}
  \end{equation} 
  where
  $\Phi$ is the cumulative distribution function of 
  the standard normal distribution and
  $J \coloneqq \sqrt{(\mu_1 - \mu_0)^\top \Sigma^{-1} (\mu_1 - \mu_0)}$
  is the Mahalanobis distance between $\mu_{0}$ and $\mu_{1}$.
\end{proposition}

To prove \cref{prop:gaussian},
define the score function $f^{*}: \mathcal{X} \to \R$ as
\begin{equation}
  \label{eq:gaussian:llr}
  f^{*}(x)
  \coloneqq
  \log \frac{p_1(x)}{p_0(x)} 
  =
  (\mu_1 - \mu_0)^\top \Sigma^{-1} \paren{x - \frac{\mu_0 + \mu_1}{2}}
  \text{.}
\end{equation}

\begin{proof}
  We first consider the nondegenerate case $J>0$.
  We begin by proving the formula for $\bestber$.
  We saw in the proof of \cref{lem:bestber} that the BER-optimal
  classifier is given by $h_{\ber}^{*}(x) = \I{\eta(x) \geq \theta}$.
  Since the posterior probability $\eta(x)$ can be expressed as
  \begin{equation}
    \eta(x)
    =
    \frac{\theta p_1(x)}{\theta p_1(x) + (1 - \theta) p_0(x)}
    \text{,}
  \end{equation}
  the condition $\eta(x) \geq \theta$ is equivalent to
  $p_1(x) \geq p_0(x)$, which is in turn equivalent to
  $f^{*}(x) \geq 0$.
  The class-conditional distributions of the score $f^{*}(x)$ are given by
  \begin{equation}
    \label{eq:gaussian:score-dist}
    f^{*}(x) \mid \{y=0\} \sim \mathcal{N}\paren{-\frac{J^2}{2}, J^2}
    \text{,}\quad
    f^{*}(x) \mid \{y=1\} \sim \mathcal{N}\paren{\frac{J^2}{2}, J^2}
    \text{.}
  \end{equation}
  Therefore, the within-class error rates of $h_{\ber}^{*}$ are given by
  \begin{equation}
    \begin{aligned}
      \fpr(h_{\ber}^{*})
      &=
      \P{f^{*}(x) \geq 0 \mid y=0}
      =
      \Phi\paren{- \frac{J}{2}}
      \text{,}
      \\
      \fnr(h_{\ber}^{*})
      &=
      \P{f^{*}(x) < 0 \mid y=1}
      =
      \Phi\paren{- \frac{J}{2}}
      \text{,}
    \end{aligned}
  \end{equation}
  and hence the optimal BER is given by
  \begin{equation}
    \bestber
    = \frac{1}{2} \paren{\fpr(h_{\ber}^{*}) + \fnr(h_{\ber}^{*})}
    = \Phi\paren{- \frac{J}{2}}
    \text{.}
  \end{equation}

  Next, we prove the AUC case.
  In the proof of \cref{lem:bestauc}, we showed that
  the maximum AUC is attained by choosing a scoring function to be
  any strictly increasing transformation of $\eta(x)$.
  The scoring function $f^{*}$ defined in \eqref{eq:gaussian:llr}
  is an instance of such scores; to see this, note that
  \begin{equation}
    f^*(x) = \log \frac{\theta^{-1} - 1}{\eta(x)^{-1} - 1}
    \text{,}
  \end{equation}
  which is strictly increasing in $\eta(x)$.
  Thus, $f^*$ achieves the maximum AUC.

  Now, let
  $x_+ \sim \P{x \mid y=1}$ and
  $x_- \sim \P{x \mid y=0}$
  be independent.
  Using the class-conditional distributions of $f^*(x)$
  given in \eqref{eq:gaussian:score-dist}, we obtain
  \begin{equation}
    f^*(x_+) - f^*(x_-)
    \sim
    \mathcal{N}\paren{J^2, 2J^2}
  \end{equation}
  and hence the maximum AUC is given by
  \begin{equation}
    \begin{aligned}
      \bestauc
      &=
      \P{f^*(x_+) > f^*(x_-)}
      +
      \frac{1}{2} \P{f^*(x_+) = f^*(x_-)}
      \\ &=
      \P{f^*(x_+) - f^*(x_-) > 0} + 0
      \\ &=
      \Phi\paren{\frac{J}{\sqrt{2}}}
      \text{.}
    \end{aligned}
  \end{equation}

  It remains to consider the degenerate case $J=0$,
  in which $f^*$ remains optimal.
  Since $\Sigma$ is positive definite,
  $J=0$ implies $\mu_0=\mu_1$
  and that the optimal classifier always predicts
  $h^*(x) = \I{f^*(x) \geq 0} = \I{0 \geq 0} = 1$.
  Hence, $\bestber = (1 + 0) / 2 = \Phi(-J/2)$.
  Similarly, we have $\bestauc=1/2=\Phi(J/\sqrt{2})$
  when $J=0$ because $f^*(x_+) - f^*(x_-) = 0 - 0 = 0$ almost surely.
  This completes the proof.
\end{proof}

\section{Experiments}
\label{sec:app:experiments}

Here, we provide
details on the experiments we conducted
in \cref{sec:experiments}.
Our experiment code is based on
that of \citet{ushio2025practical},
available at \url{https://github.com/RyotaUshio/bayes-error-estimation}
and licensed under the MIT License.

\subsection{Approximately monotone corruption}
\label{sec:experiments:order-breakage}

Here, we present an additional experiment with synthetic data
to investigate how
the performance of our isotonic-regression-based estimator
degrades when the corruption is only approximately monotone.

\subsubsection{Experimental setup}

Following the experimental protocol in 
Appendix~D.5 of \citet{ushio2025practical}, 
we isolated order breakage by starting from 
the under-confident inverse-beta-calibration map $f$
of \eqref{eq:inverse-beta}
with parameters $(a,b)=(1.5,0.5)$ 
and adding Gaussian noise on the logit scale:
\begin{equation}
  \xi_i
  =
  \operatorname{sigmoid}
  \left(
  \operatorname{logit}(f(\eta_i)) + z_i
  \right)
  \text{,} \quad
  z_i \sim \mathcal{N}(0, \sigma^2)
\end{equation}
Consequently, 
the corruption used here is nearly monotone but not exactly monotone. 
The parameter $\sigma$ controls how far 
the corruption departs from monotonicity.

We used the same two-dimensional Gaussian mixture 
as in \cref{sec:experiments:estimation}.
We used a grid $\sigma\in\{0,0.05,0.1,0.2,0.4,0.8,1.6\}$.
For each $\sigma$, we used the proposed isotonic-regression-based estimator 
to compute the $\bestber$ and $\bestauc$ estimates. 
We also computed
the Kendall rank correlation $\tau$
and
$q=(1-\tau)/2$.
The quantity $q$ corresponds to
the fraction of reversed pairs.
We repeated each setting on 30 independently sampled datasets. 
By \cref{prop:gaussian},
the ground-truth optimal values are 
$\bestber = \Phi(-\sqrt{2})=0.078650$
and 
$\bestauc = \Phi(2)=0.977250$.

\subsubsection{Results}

\begin{figure}[t]
  \centering
  \begin{subfigure}[t]{0.49\textwidth}
    \centering
    \includegraphics[width=\linewidth]{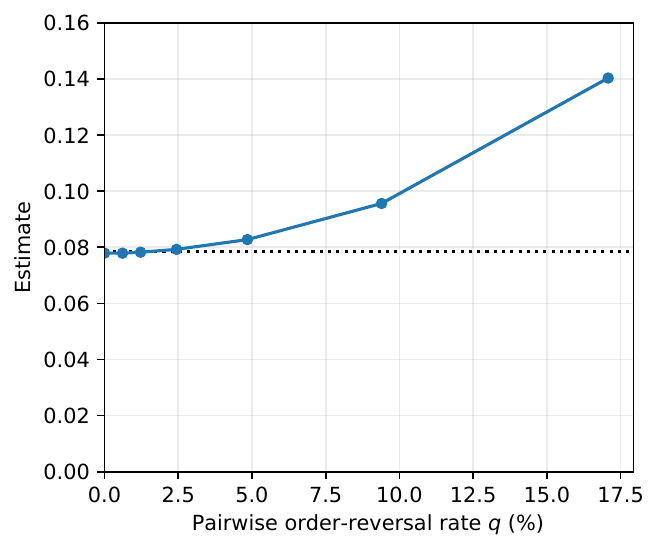}
    \caption{BER}
    \label{fig:order-breakage:ber}
  \end{subfigure}
  \hfill
  \begin{subfigure}[t]{0.49\textwidth}
    \centering
    \includegraphics[width=\linewidth]{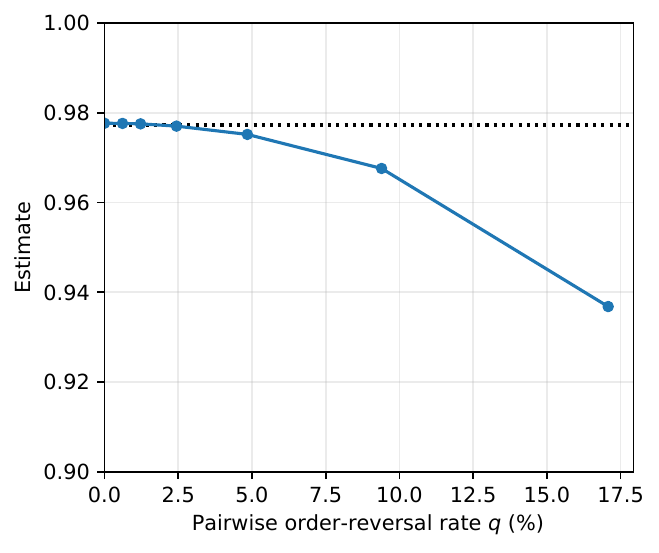}
    \caption{AUC}
    \label{fig:order-breakage:auc}
  \end{subfigure}
  \caption{
    Estimates of $\bestber$ and $\bestauc$ 
    under increasing violations of order preservation.
    The horizontal axis represents the mean pairwise order-reversal rate
    $q=(1-\tau)/2$.
    The points and error bars indicate means and 95\% Monte Carlo CIs over
    30 independent datasets.
    Note that the error bars are smaller than the marker size for most points.
    The dotted lines indicate the ground-truth optimal values.
  }
  \label{fig:order-breakage}
\end{figure}

\Cref{fig:order-breakage} shows the results.
The estimates remain close to the ground-truth values
as the pairwise order-reversal rate increases to $2.44\%$.
Over this range,
the 95\% Monte Carlo confidence intervals (CIs) contain the ground-truth values.
At $q=4.85\%$ or higher,
no CI contains the corresponding population value,
and the deviations become substantially larger as $q$ increases.
Overall,
introducing a small amount of rank violation
translated to a mild increase in estimation error
although
the error increased progressively as the order-reversal rate grew.

\subsection{Compute resources}
\label{sec:compute-resources}

All of the experiments presented in this paper
were conducted on the CPU of a single Apple MacBook Pro 
(M1 chip, 16GB RAM).

\makeatletter
\if@preprint\else
\clearpage
\input{checklist.tex}
\fi
\makeatother

\end{document}